\pdfoutput=1

\documentclass[11pt]{article}

\usepackage{amsmath}
\usepackage{amsthm}
\usepackage{amssymb}
\usepackage{amsfonts}
\usepackage[top = 1 in, bottom = 1.2 in, left = 1 in, right = 1 in]{geometry}
\usepackage{color}
\usepackage{mathrsfs}
\usepackage[normalem]{ulem}
\usepackage{longtable}
\usepackage{hyperref}
\usepackage{cleveref}
\usepackage{dsfont}
\usepackage{graphics,graphicx}
\usepackage{wrapfig}
\usepackage{caption}
\usepackage{subcaption}
\usepackage{float}
\usepackage{mathtools}
\usepackage{thmtools}
\usepackage[textsize=tiny]{todonotes}
\usepackage[mathscr]{euscript}

\usepackage{upgreek}

\usepackage{multirow}

\newtheorem{theorem}{Theorem}
\newtheorem{lemma}{Lemma}

\newtheorem{corollary}{Corollary}

\declaretheoremstyle[headfont=\bf,bodyfont=\normalfont]{ex}

\declaretheoremstyle[bodyfont=\normalfont]{rm}

\DeclareMathOperator*{\argmin}{arg\,min}

\newcommand{\defn}{\ensuremath{:\,=}}
\newcommand{\nfed}{\ensuremath{=\,:}}
\DeclareMathOperator{\proj}{\Pi}
\DeclareMathOperator{\diag}{{\rm diag}}

\newcommand{\dx}{\ensuremath{\diff \state}}
\newcommand{\diff}{\ensuremath{d}}

\makeatletter
\long\def\@makecaption#1#2{
	\vskip 0.8ex
	\setbox\@tempboxa\hbox{\small {\bf #1:} #2}
	\parindent 1.5em  %% How can we use the global value of this???
	\dimen0=\hsize
	\advance\dimen0 by -3em
	\ifdim \wd\@tempboxa >\dimen0
	\hbox to \hsize{
		\parindent 0em
		\hfil 
		\parbox{\dimen0}{\def\baselinestretch{0.96}\small
			{\bf #1.} {#2}
		} 
		\hfil}
	\else \hbox to \hsize{\hfil \box\@tempboxa \hfil}
	\fi
}
\makeatother

\newcommand{\Prob}{\ensuremath{\mathbb{P}}}
\newcommand{\Exp}{\ensuremath{\mathbb{E}}}
\newcommand{\Real}{\ensuremath{\mathds{R}}}
\newcommand{\Natural}{\ensuremath{\mathds{N}}}
\newcommand{\Int}{\ensuremath{\mathds{Z}}}

\newcommand{\thetastar}{\ensuremath{\theta^*}}
\newcommand{\thetahat}{\ensuremath{\widehat{\theta}}}

\newcommand{\Deltahat}{\ensuremath{\widehat{\Delta}}}

\newcommand{\alphabold}{\ensuremath{\boldsymbol{\alpha}}}

\newcommand{\PackNum}{\ensuremath{M}}
\newcommand{\RKHS}{\ensuremath{\mathds{H}}}
\newcommand{\RKHShat}{\ensuremath{\widehat{\mathds{H}}}}
\newcommand{\Lmu}{\ensuremath{L^2(\distr)}}

\newcommand{\MRP}{\ensuremath{\mathscr{I}}}

\newcommand{\Ker}{\ensuremath{\mathcal{K}}}
\newcommand{\KerFun}{\Ker}
\newcommand{\MRPclass}{\ensuremath{\mathfrak{M}}}

\newcommand{\reward}{\ensuremath{r}}

\newcommand{\discount}{\ensuremath{\gamma}}

\newcommand{\TransOp}{\ensuremath{\mathcal{P}}}

\newcommand{\IdOp}{\ensuremath{\mathcal{I}}}
\newcommand{\StateSp}{\ensuremath{\mathcal{X}}}

\newcommand{\CovOp}{\ensuremath{\Sigma_{\rm cov}}}
\newcommand{\CovOphat}{\ensuremath{\widehat{\Sigma}_{\rm cov}}}
\newcommand{\CovOptilde}{\ensuremath{\widetilde{\Sigma}_{\rm cov}}}
\newcommand{\CrOp}{\ensuremath{\Sigma_{\rm cr}}}
\newcommand{\CrossOp}{\CrOp}
\newcommand{\CrOphat}{\ensuremath{\widehat{\Sigma}_{\rm cr}}}

\newcommand{\CovMt}{\ensuremath{{\bf K}_{\rm cov}}}
\newcommand{\CrMt}{\ensuremath{{\bf K}_{\rm cr}}}
\newcommand{\IdMt}{\ensuremath{{\bf I}_{\numobs}}}
\newcommand{\yvec}{\ensuremath{\boldsymbol{y}}}

\newcommand{\widgraph}[2]{\includegraphics[keepaspectratio,width=#1]{#2}}

\newcommand{\mydefn}{\ensuremath{: \, =}}

\newcommand{\numobs}{\ensuremath{n}}
\newcommand{\delcrit}{\ensuremath{\delta_\numobs}}
\newcommand{\delcritsq}{\ensuremath{\delta^2_\numobs}}

\newcommand{\statdim}{\ensuremath{d_\numobs}}

\newcommand{\eig}{\ensuremath{\mu}}

\newcommand{\distr}{\ensuremath{\upmu}}
\newcommand{\distrm}[1]{\ensuremath{\distr_{#1}}}

\newcommand{\ridge}{\ensuremath{\lambda_\numobs}}

\newcommand{\Event}{\ensuremath{\mathcal{E}}}

\newcommand{\Term}{\ensuremath{T}}

\DeclarePairedDelimiterX{\inprod}[2]{\langle}{\rangle}{#1, \, #2}

\DeclarePairedDelimiterX{\kulldiv}[2]{(}{)}{#1\;\delimsize\|\;#2}
\makeatletter
\newcommand{\@kullstar}[2]{D_{\text{KL}}\kulldiv*{#1}{#2}}
\newcommand{\@kullnostar}[3][]{D_{\text{KL}}\kulldiv[#1]{#2}{#3}}
\newcommand{\kull}{\@ifstar\@kullstar\@kullnostar}
\makeatother

\makeatletter
\newcommand{\@hilinstar}[2]{\inprod*{#1}{#2}_{\RKHS}}
\newcommand{\@hilinnostar}[3][]{\inprod[#1]{#2}{#3}_{\RKHS}}
\newcommand{\hilin}{\@ifstar\@hilinstar\@hilinnostar}
\makeatother

\DeclarePairedDelimiterX{\norm}[1]{\|}{\|}{#1}
\makeatletter
\newcommand{\@normstar}[1]{\norm*{#1}_{\RKHS}}
\newcommand{\@normnostar}[2][]{\norm[#1]{#2}_{\RKHS}}
\newcommand{\hilnorm}{\@ifstar\@normstar\@normnostar}
\makeatother

\makeatletter
\newcommand{\@supnormstar}[1]{\norm*{#1}_{\infty}}
\newcommand{\@supnormnostar}[2][]{\norm[#1]{#2}_{\infty}}
\newcommand{\supnorm}{\@ifstar\@supnormstar\@supnormnostar}
\makeatother

\makeatletter
\newcommand{\@munormstar}[1]{\norm*{#1}_{\distr}}
\newcommand{\@munormnostar}[2][]{\norm[#1]{#2}_{\distr}}
\newcommand{\munorm}{\@ifstar\@munormstar\@munormnostar}
\makeatother

\makeatletter
\newcommand{\@mumnormstar}[2]{\norm*{#2}_{\distrm{#1}}}
\newcommand{\@mumnormnostar}[3][]{\norm[#1]{#3}_{\distrm{#2}}}
\newcommand{\mumnorm}{\@ifstar\@mumnormstar\@mumnormnostar}
\makeatother

\newcommand{\Rep}[1]{\ensuremath{\Phi_{#1}}}

\newcommand{\xplus}{\ensuremath{x_+}}
\newcommand{\xneg}{\ensuremath{x_-}}

\newcommand{\mprime}{\ensuremath{{m'}}}
\newcommand{\mdagger}{\ensuremath{{m^\dagger}}}

\long\def\comment#1{}

\newcommand{\plaincon}{\ensuremath{c}}

\newcommand{\Aevent}{\ensuremath{\mathcal{A}}}
\newcommand{\svar}{\ensuremath{\nu}}

\newcommand{\GamOp}{\ensuremath{\Gamma}}
\newcommand{\GamOpHat}{\ensuremath{\widehat{\GamOp}}}
\newcommand{\specfun}{\ensuremath{\rho}}
\newcommand{\real}{\ensuremath{\mathbb{R}}}
\newcommand{\usedim}{\ensuremath{d}}
\newcommand{\Bevent}{\ensuremath{\mathcal{B}}}
\newcommand{\plaincontwo}{\ensuremath{\plaincon^\prime}}
\newcommand{\Ztil}{\ensuremath{\widetilde{Z}}}

\newcommand{\Exs}{\Exp}

\newcommand{\newrad}{\ensuremath{R}}
\newcommand{\newradsq}{\ensuremath{\newrad^2}}

\newcommand{\stdfun}{\ensuremath{\sigma}}

\newcommand{\cbar}{\ensuremath{\bar{c}}}

\newcommand{\tcrit}{\ensuremath{t_\numobs}}
\newcommand{\tcritsq}{\ensuremath{t^2_\numobs}}

\newcommand{\ucrit}{\ensuremath{u_\numobs}}
\newcommand{\ucritsq}{\ensuremath{u^2_\numobs}}

\newcommand{\SqDis}{\ensuremath{(1-\discount)}}

\newcommand{\Ellipse}{\ensuremath{\mathcal{E}}}
\newcommand{\Elltil}{\ensuremath{\tilde{\Ellipse}}}

\newcommand{\SpecialConstant}{\ensuremath{\zeta}}

\newcommand{\unibou}{\ensuremath{\kappa}}
\newcommand{\bou}{\ensuremath{b}}

\newcommand{\CI}{\ensuremath{\operatorname{CI}}}

\newcommand{\parap}{\ensuremath{p}}

\newcommand{\Deltap}{\ensuremath{\Delta p}}

\newcommand{\interval}{\ensuremath{\Delta}}

\makeatletter
\newcommand{\@twonormstar}[1]{\norm*{#1}_{L^2}}
\newcommand{\@twonormnostar}[2][]{\norm[#1]{#2}_{L^2}}
\newcommand{\twonorm}{\@ifstar\@twonormstar\@twonormnostar}
\makeatother

\newcommand{\base}{\ensuremath{\phi}}

\newcommand{\newradbar}{\ensuremath{\bar{\newrad}}}

\newcommand{\newradbarreward}{\ensuremath{\newradbar}}

\newcommand{\stdbarreward}{\ensuremath{\bar{\stdfun}}}
\newcommand{\eigtil}{\ensuremath{\tilde{\eig}}}
\newcommand{\basetil}{\ensuremath{\tilde{\base}}}

\newcommand{\stdbar}{\ensuremath{\bar{\stdfun}}}
\newcommand{\bv}{\ensuremath{\boldsymbol{v}}}
\newcommand{\bu}{\ensuremath{\boldsymbol{u}}}
\newcommand{\util}{\ensuremath{\tilde{u}}}
\newcommand{\bx}{\ensuremath{\boldsymbol{\state}}}
\newcommand{\bD}{\ensuremath{{\bf D}}}
\newcommand{\bI}{\ensuremath{{\bf I}}}
\newcommand{\bSigma}{\ensuremath{{\bf \Sigma}}}
\newcommand{\bP}{\ensuremath{{\bf P}}}
\newcommand{\bPbase}{\ensuremath{\bP_0}}
\newcommand{\bPr}{\ensuremath{\bPA}}
\newcommand{\bPt}{\ensuremath{\bPB}}
\newcommand{\bmu}{\ensuremath{\boldsymbol{\distr}}}
\newcommand{\bF}{\ensuremath{{\bf F}}}
\newcommand{\br}{\ensuremath{\boldsymbol{\reward}}}
\newcommand{\btheta}{\ensuremath{\boldsymbol{\theta}}}
\newcommand{\bthetabase}{\ensuremath{\btheta_0}}
\newcommand{\bthetar}{\ensuremath{\btheta_A}}
\newcommand{\bthetat}{\ensuremath{\btheta_B}}
\newcommand{\cha}{\ensuremath{\chi}}
\newcommand{\be}{\ensuremath{{\bf e}}}
\newcommand{\TermA}{\ensuremath{A}}
\newcommand{\TermAtil}{\ensuremath{\widetilde{A}}}

\newcommand{\state}{\ensuremath{x}}
\newcommand{\statetwo}{\ensuremath{\state'}}
\newcommand{\statenew}{\ensuremath{y}}
\newcommand{\State}{\ensuremath{X}}
\newcommand{\Statetwo}{\ensuremath{\State'}}
\newcommand{\Statenew}{\ensuremath{Y}}

\newcommand{\LB}{\operatorname{LB}}
\newenvironment{carlist}
 {\begin{list}{$\bullet$}
 {\setlength{\topsep}{0in} \setlength{\partopsep}{0in}
  \setlength{\parsep}{0in} \setlength{\itemsep}{\parskip}
  \setlength{\leftmargin}{0.15in} \setlength{\rightmargin}{0.08in}
  \setlength{\listparindent}{0in} \setlength{\labelwidth}{0.08in}
  \setlength{\labelsep}{0.1in} \setlength{\itemindent}{0in}}}
 {\end{list}}

\newcommand{\bcar}{\begin{carlist}}
\newcommand{\ecar}{\end{carlist}}

\newcommand{\simpleub}{\varepsilon^2}
\newcommand{\RKHSA}{\RKHS_A}
\newcommand{\RKHSB}{\RKHS_B}
\newcommand{\MRPclassA}{\MRPclass_A}
\newcommand{\MRPclassB}{\MRPclass_B}
\newcommand{\bPA}{\bP_A}
\newcommand{\bPB}{\bP_B}
\newcommand{\KerA}{\Ker_A}
\newcommand{\KerB}{\Ker_B}
\newcommand{\rewardA}{\reward_A}
\newcommand{\rewardB}{\reward_B}
\newcommand{\baseA}[1]{\ensuremath{\base_{A,{#1}}}}
\newcommand{\baseB}[1]{\ensuremath{\base_{B,{#1}}}}
\newcommand{\numint}{\ensuremath{K}}

\long\def\comment#1{}

\newcommand{\EffHorizon}{\ensuremath{H}}

\newcommand{\specexp}{\ensuremath{\nu}}

\newcommand{\KerComplex}{\ensuremath{\mathcal{C}}}

\newcommand{\Vstar}{\ensuremath{V^*}}

\newcommand{\Bellman}{\ensuremath{\mathcal{T}}}

\newcommand{\Gclass}{\ensuremath{\mathds{G}}}

\begin{document}
	
\begin{center}
  {\bf \LARGE Optimal policy evaluation using kernel-based
    temporal difference methods} \\
  \vspace{.5em} 
\end{center}

\begin{center}
  \begin{tabular}{ccc}
    Yaqi Duan && Mengdi Wang \\
    Department of ORFE & & Department of ECE \\
    Princeton University & & Princeton University
  \end{tabular}

\vspace*{0.05in}  
\begin{tabular}{c}  
  Martin J. Wainwright \\
  Departments of Statistics and EECS \\
  UC Berkeley
\end{tabular}

\vspace*{0.3in}
\today

\end{center}

\begin{center}
  \begin{abstract}
    We study methods based on reproducing kernel Hilbert spaces for
    estimating the value function of an infinite-horizon discounted
    Markov reward process (MRP).  We study a regularized form of the
    kernel least-squares temporal difference (LSTD) estimate; in the
    population limit of infinite data, it corresponds to the fixed
    point of a projected Bellman operator defined by the associated
    reproducing kernel Hilbert space.  The estimator itself is
    obtained by computing the projected fixed point induced by a
    regularized version of the empirical operator; due to the
    underlying kernel structure, this reduces to solving a linear
    system involving kernel matrices.  We analyze the error of this
    estimate in the $L^2(\upmu)$-norm, where $\upmu$ denotes the
    stationary distribution of the underlying Markov chain.  Our
    analysis imposes no assumptions on the transition operator of the
    Markov chain, but rather only conditions on the reward function
    and population-level kernel LSTD solutions.  We use empirical
    process theory techniques to derive a non-asymptotic upper bound
    on the error with explicit dependence on the eigenvalues of the
    associated kernel operator, as well as the instance-dependent
    variance of the Bellman residual error.  In addition, we prove
    minimax lower bounds over sub-classes of MRPs, which shows that
    our rate is optimal in terms of the sample size $n$ and the
    effective horizon $H = (1 - \gamma)^{-1}$.  Whereas existing
    worst-case theory predicts cubic scaling ($H^3$) in the effective
    horizon, our theory reveals that there is in fact a much wider
    range of scalings, depending on the kernel, the stationary
    distribution, and the variance of the Bellman residual error.
    Notably, it is only parametric and near-parametric problems that
    can ever achieve the worst-case cubic scaling.
\end{abstract}

\end{center}

%%%%%%%%%%%%%%%%%%%%%%%%%%%%%%%%%%%%%%%%%%%%%%%%%%%%%%%%%%%%%%%%%%%%%%%%%%%
	
\section{Introduction}
\label{SecIntroduction}

Markov decision processes provide a formalism for studying optimal
decision-making in dynamic
settings~\cite{Puterman05,bertsekas1995dynamic}, and are used in a
wide variety of applications
(e.g.,~\cite{BouDij17,sutton2018reinforcement}).  Reinforcement
learning (RL) refers to methods that operate in settings where the
model structure and/or parameters are unknown.  In this context, a
central problem is to use samples to evaluate the quality of a given
policy, as assessed via its value function.  Indeed, the estimation of
value functions serves as a fundamental building block for many RL
algorithms~\cite{BerTsi96,sutton2018reinforcement}.

When a given policy is fixed, a Markov decision process reduces to a
Markov reward process (MRP).  The value of any given initial state in
an MRP corresponds to the expected cumulative reward along a
trajectory when starting from the given state; the collection of all
such state values defines the value function.  The problem of
estimating this function is known as \emph{policy evaluation}, or
\emph{value function estimation}, and we use these terms
interchangeably.  In practice, policy evaluation is challenging
because the state space might be continuous, or even when discrete, it
might involve a huge number of possible states.  For this reason,
practical methods for policy evaluation typically involve some form of
function approximation.

The simplest and most well-studied approach is based on linear
function approximation, in which the value function is approximated as
a weighted combination of a fixed set of features.  This particular
choice leads to the least-squares policy evaluation estimator, also
known as the least-squares temporal difference (LSTD) estimate, along
with its online temporal difference variants
(e.g.,~\cite{bradtke1996linear,tsitsiklis1997analysis,sutton2018reinforcement,MouLiWaiBarJor20}).
The choice of linear functions is attractive in that the LSTD estimate
is easy to compute, based on solving a linear system of equations.
However, the expressivity of linear functions is limited, and so that
it is natural to seek approximations in richer function classes.

In many types of statistical problems, including regression, density
estimation, and clustering, methods based on reproducing kernel
Hilbert spaces (RKHSs) have proven
useful~\cite{Gu02,BerTho04,shawe2004kernel,wainwright2019high}.  As we
discuss in Section~\ref{SecRelated}, kernel methods have also proven
useful in the specific context of reinforcement learning.  Kernel
methods allow for much richer representations of functions, by
working---in an implicit way---over a possibly infinite set of
features, as defined by the eigenfunctions of the associated kernel
integral operator.  However, at the same time, due to the classical
representer theorem~\cite{KimWah71, wainwright2019high}, a broad class
of kernel-based estimators can be computed relatively easily by
working directly with $\numobs$-dimensional kernel matrices, where
$\numobs$ is the sample size.

The main goal of this paper is to provide a sharp characterization of
the statistical properties of a family of kernel-based procedures for
policy evaluation. So as to bring our specific contributions into
sharp focus, we study the case of infinite-horizon
$\discount$-discounted Markov reward processes (MRPs), but much of our
analysis and associated techniques also has consequences for kernel
methods in the finite-horizon setting.  In our analysis, we assume
that we have access to the reward function and i.i.d. transition pairs
drawn from the stationary distribution.  We analyze a kernel-based
temporal difference estimator, whose population limit corresponds to
the fixed point of a projected Bellman operator.  We measure the
difference between the empirical and population estimators in $\Lmu$
norm, with $\distr$ denoting the stationary distribution. We refer to
this $\Lmu$ error as the \emph{estimation error}.  At a high level,
the main contribution of this paper is to provide a sharp and
partially instance-independent analysis of this estimation error.

%%%%%%%%%%%%%%%%%%%%%%%%%%%%%%%%%%%%%%%%%%%%%%%%%%%%%%%%%%%%%%%%%%%%%%%%%%%%%%%%%

\subsection{Related work and our contributions}
\label{SecRelated}

We begin by discussing related work and then, with this context in
place, provide a high-level overview of our contributions.

\paragraph{Related work:}

Here we provide a partial overview of past work, with an emphasis on
those papers providing estimation error bounds that are most relevant
for putting our results in context.  The utility of kernel methods in
reinforcement learning is by now well-established, as attested to by
the lengthy line of previous papers on the topic
(e.g.,~\cite{BagSch03,grunewalder2012modelling,
  taylor2009kernelized,BarPrePin16,koppel2020policy,
  dai2017learning,feng2019kernel,feng2020accountable}).  In the
special case of a linear kernel function, the kernel-based method
studied in this paper reduces to the classical least-squares temporal
difference (LSTD) method~\cite{Sut88,
  sutton2018reinforcement,bradtke1996linear}.

In terms of papers that provide guarantees on statistical estimation
in non-parametric settings, early work by Ormoneit and
Sen~\cite{ormoneit2002kernel} studied the use of local-averaging
kernel methods for approximating value functions; they proved various
types of asymptotic consistency results.  Munos and
Szepesvari~\cite{MunSze08} studied methods for fitted value iteration
(FVI) under various types of $\ell_p$-norms; under metric entropy
conditions on the function space, they proved various types of
consistency results, but without providing sharp or minimax-optimal
guarantees.  In later work, Farahmand et
al.~\cite{farahmand2016regularized} studied a class of regularized
procedures for both policy evaluation and policy optimization.  Their
analysis is attractive in allowing for quite general function classes,
with reproducing kernel Hilbert spaces being an important special
case.  They provided guarantees under bounds on the sup-norm metric
entropy of the function classes at hand, and for certain function
classes, they argued that their bounds achieved the optimal scaling in
sample size $\numobs$.  Farahmand et al.  also conjectured that it
should be possible to prove similar guarantees using metric entropy
conditions in the $\distr$-norm, and indeed, in the special case of
RKHS classes, one consequence of our results is to confirm this
conjecture.  A more recent line of work has studied variants of fitted
Q-iteration (FQI) using neural network approximation, and provided
statistical guarantees under different notions of smoothness.  For
example, Fan et al.~\cite{fan2020theoretical} exploited the H\"older
smoothness of the range of Bellman operator to derive bounds on
estimation error; Nguyen-Tang et al.~\cite{nguyen2021sample}
approximated deep ReLU networks using Besov classes; and Long et
al.~\cite{long20212} analyzed two-layer neural networks based on neural
tangent kernels or Barron spaces. All these works contribute to the
understanding of empirical success of deep reinforcement learning.

A notable feature of much past work is while it provides bounds on
statistical error, it does not carefully track the dependence on the
(effective) horizon and model dynamics, and the variance of the
Bellman residual.  As we argue in this paper, understanding how
non-parametric procedures depend on the latter quantities is
essential, as they are the ingredients that actually distinguish value
function estimation from a typical (static) prediction problem, with
ordinary non-parametric regression being the archetypal example.  In
order to reveal this dependence, the analysis of this paper makes use
of empirical process techniques~\cite{vandeGeer,wainwright2019high}
that have proven successful for analyzing kernel ridge regression and
related estimators
(e.g.,~\cite{Zhang2005b,yang2017randomized,RasWaiYu12}).  Essential
for obtaining sharp rates is the local Rademacher complexity, which
has an explicit expression in terms of the eigenvalues of the kernel
integral operator~\cite{mendelson2002geometric}; see Chapters 12 and
13 in the book~\cite{wainwright2019high} for more details.

It is also worth noting that recent years have witnessed considerable
progress in understanding policy evaluation in off-policy settings,
and/or providing guarantees that have optimal instance-dependent
rates.  This work can be separated into work that is either
asymptotic~\cite{jiang2016doubly,kallus2019efficiently,kallus2020double}
and
non-asymptotic~\cite{pananjady2019value,khamaru2020temporal,xie2019towards,yin2020asymptotically}
in nature.  In this non-asymptotic setting, much of this work is
focused on either the tabular case, or the simpler setting of linear
function approximation, as opposed to the non-parametric cases of
interest here.  We note that our results do depend on the problem
instance, but this instance-dependence is not (yet) as sharp as that
established in the simpler setting of tabular
problems~\cite{pananjady2019value,khamaru2020temporal}.

This paper also makes connections to the large body of work on
instrumental variable (IV) methods
(e.g.,~\cite{Whi82,NewPow03,Woo10}).  It is known that the
least-squares temporal difference (LSTD) estimate can be derived as a
classical linear IV estimate~\cite{bradtke1996linear}.  More
generally, the kernel-based procedures in this paper correspond to a
non-parametric form of an instrumental variable method.  While
portions of our analysis are specific to reinforcement learning, we
suspect that our techniques can be adapted so as to provide guarantees
for other non-parametric IV estimates.

%%%%%%%%%%%%%%%%%%%%%%%%%%%%%%%%%%%%%%%%%%%%%%%%%%%%%%%%%%%%%%%%%%%%%%%%%%%%%

\paragraph{Our contributions:}

Consistency of any statistical estimator is certainly a desirable
requirement.  A more ambitious goal, and a centerpiece of
high-dimensional statistics, is to give a more refined non-asymptotic
characterization, one which tracks not only sample size but also other
structural properties of the problem.  In the context of policy
evaluation for Markov reward processes with discount factor $\discount
\in (0,1)$, such structural properties include: (a) the complexity of
the population-level value function $\thetastar$ that is being
estimated; (b) the ``richness'' of the function class used for
approximation relative to the stationary measure of the Markov chain;
(c) the effective horizon \mbox{$\EffHorizon \defn (1 -
  \discount)^{-1}$,} which measures the typical scale over which the
discounted reward process evolves; and (d) the underlying noise
function, given by the variance of the Bellman residual.  The latter
two properties are of particular interest, since they distinguish the
dynamic nature of value function estimation from a standard problem of
static non-parametric estimation.

The main contribution of this paper is to give a precise
characterization, including both matching upper and lower bounds, on
how a well-tuned version of the kernel-based LSTD estimate depends on
all of these structural parameters.  Notably, our characterization is
instance-dependent, in that the bounds vary considerably depending on
the structure of $\thetastar$ and the associated variance of the
Bellman residual error, along with the eigenvalues of the kernel
integral operator, which vary as a function of both the kernel
function class, and the stationary measure of the Markov chain.  En
route to doing so, we provide specific guidance on how the
regularization parameter, essential for non-parametric methods such as
those based on RKHSs, should be chosen.

\Cref{thm:ub} provides two types of non-asymptotic bounds on the
estimation error of a regularized kernel LSTD estimate: a ``slow''
rate and a ``fast'' rate.  These two guarantees differ in the way that
the inherent noise of the problem is measured.  While the ``slow''
guarantee holds for any sample size $\numobs$, the guarantee is based
on a crude measure of the noise level, based on bounds on the sup-norm
and Hilbert norm of the population-level value function.  The second
``fast'' guarantee holds only once the sample size exceeds a certain
threshold, but depends on the variance of the Bellman residual error,
which is a fundamental quantity for the problem.  Indeed, in our
second main result, stated as \Cref{thm:lb}, we study the best
performance of any procedure of two particular sub-classes of MRPs,
and prove lower bounds that match the ``fast'' rates from
\Cref{thm:ub} in terms of all relevant problem-dependent quantities.
These matching upper and lower bounds establish the optimality of our
procedure.

Our theory applies to a fairly general class of kernel functions in
arbitrary dimension, with the rates depending on the eigenvalues
$\{\mu_j\}_{j=1}^\infty$ of the induced kernel operator.  It is
important to note that these eigenvalues depend not just on the
kernel, but also on the stationary distribution of the Markov chain.
Let us briefly highlight some interesting predictions made by our
theory regarding how the optimal $L^2(\distr)$-error should scale with
the effective horizon $\EffHorizon = (1 - \discount)^{-1}$.  One
special case, of interest in its own right, are kernels and stationary
distributions for which these eigenvalues decay at a polynomial rate,
say $\mu_j \asymp (1/j)^{2 \alpha}$ for some $\alpha > 1/2$.  In
Section~\ref{SecSimulations}, we construct a ``hard'' ensemble of MRPs
for which our theory---both upper and lower bounds---guarantees that
for a fixed sample size, the squared $L^2(\distr)$-error should grow as
$\EffHorizon^{\frac{6 \alpha +2}{2 \alpha + 1}}$.  In the limit as
$\alpha \rightarrow +\infty$, the kernel class becomes a parametric
function class, and the horizon dependence becomes the familiar cubic
one $\EffHorizon^3$.  However, for genuinely non-parametric classes
where $\alpha$ is relatively small, the dependence on the effective
horizon is much milder---e.g., it scales as $\EffHorizon^{8/3}$ for a
kernel with $\alpha = 1$.  This reveals the interesting phenomenon
that non-parametric forms of value estimation exhibit milder horizon
dependence.  Moreover, since our theory is instance-dependent via the
variance of Bellman residual, we can show that global minimax
predictions are often conservative.  In particular, we also construct
an ``easy'' ensemble for which the scaling in horizon is much milder,
given by $\EffHorizon^{\frac{4 \alpha}{2 \alpha + 1}}$.

%%%%%%%%%%%%%%%%%%%%%%%%%%%%%%%%%%%%%%%%%%%%%%%%%%%%%%%%%%%%%%%%%%%%%%%%%%%%%%%%%%%%

\subsection{Paper organization and notation}
	
The remainder of the paper is structured as follows.  We begin in
\Cref{SecBackground} by introducing background on Markov reward
processes and policy estimation, along with reproducing kernel Hilbert
spaces and the kernel LSTD estimate analyzed in this paper.
\Cref{SecMainResults} is devoted to the statement of our main results,
along with discussion of some of their consequences.

\Cref{thm:ub} provides two finite-sample upper bounds and ranges of
regularization to achieve them. \Cref{thm:lb} establishes matching
minimax lower bounds over two MRP
sub-classes. \Cref{SecSimulations} exhibits numerical
experiments with synthetic data as an illustration of our theoretical
predictions. \Cref{SecProof} contains the proofs of
\Cref{thm:ub,thm:lb}. We conclude with a discussion in
\Cref{SecConclusion}.

\paragraph{Notation:}  For any event $\Event$, we use
$\mathds{1}\{ \Event \}$ to denote the $0-1$-valued indicator
function.  We use $C$, $c$, $c_0$ etc. to denote universal constants
whose numerical values may vary from line to line. For any $D
\in \Int_+$, denote $[D] \defn \{ 1,2,\ldots,D \}$. Given a
distribution $\distr$, let $\munorm{\cdot}$ denote the $\Lmu$-norm,
which is defined as $\munorm{f}^2 \defn \int f^2 \distr(\dx)$ for $f
\in \Lmu$.  Notation $\supnorm{\cdot}$ represents the uniform bound
given by $\supnorm{f} \defn \sup_{\state \in \StateSp} |f(\state)|$.
For two measures $p, q$ with $p$ absolutely continuous with respect to
$q$, we take their Kullback–Leibler (KL) divergence $\kull{p}{q} \defn
\Exp_p \big[ \log \big( \frac{\diff p}{\diff q} \big) \big]$, along
with the $\chi^2$-divergence $\chi^2\kulldiv{p}{q} \defn \Exp_q \big[
  \big( \frac{\diff p}{\diff q} - 1 \big)^2 \big]$.

%%%%%%%%%%%%%%%%%%%%%%%%%%%%%%%%%%%%%%%%%%%%%%%%%%%%%%%%%%%%%%%%%%%%%%%%%%%%%%%%

\section{Background and problem set-up}
\label{SecBackground}

In this section, we provide background prior to formulating the kernel
estimator to be analyzed.  We begin by formulating the value function
estimation problem more precisely in \Cref{sec:problem_formulation}.
\Cref{sec:LSTD} is devoted to background on reproducing kernel Hilbert
spaces (RKHSs), along with a description of the kernel \emph{least
squares temporal difference} (LSTD) estimator.

%%%%%%%%%%%%%%%%%%%%%%%%%%%%%%%%%%%%%%%%%%%%%%%%%%%%%%%%%%

\subsection{Problem formulation}
\label{sec:problem_formulation}

A discounted Markov reward process, denoted by $\MRP(\TransOp,
\reward, \discount)$, consists of the combination of a Markov chain, a
discount factor $\discount \in (0,1)$, along with a reward function
$\reward$. In the infinite-horizon discounted setting studied here,
the Markov chain is homogeneous, defined on a state space $\StateSp$
with a transition kernel \mbox{$\TransOp: \StateSp \times \StateSp
  \rightarrow \Real$}.  The reward function $\reward: \StateSp
\rightarrow \Real$ models the reward associated with each given state,
and for some specified discount factor $\discount \in (0,1)$, our goal
is to maximize the expected discount sum of all future rewards.  More
precisely, we define the \emph{value function} $\Vstar: \StateSp
\rightarrow \real$ via
\begin{align}
\Vstar(\state) & \defn \Exp\big[ \sum_{h=0}^{\infty} \discount^h \,
  \reward(\State_h) \mid \State_0 = \state \big],
\end{align}
where the expectation is taken over a trajectory $(\state, \State_1,
\State_2, \ldots)$ from the Markov chain governed by the transition
kernel $\TransOp$. The existence of the value function $\Vstar$ is
guaranteed by mild assumptions such as the boundedness of reward
$\reward$.  For future reference, we note that the value function
$\Vstar$ is the solution to the Bellman fixed point equation
\begin{align}
  \label{eq:Bellman_star}
  \Vstar(\state) =  \reward(\state) +
  \discount \; \Exp_{\Statetwo \mid \state} \, \Vstar(\Statetwo) \,
  \qquad \text{for any $\state \in \StateSp$}.
\end{align}

In this paper, we study the problem of estimating the value function
$\Vstar$ on the basis of samples from the Markov chain, when the
reward function $\reward$ and discount factor $\discount \in (0,1)$
are given.\footnote{As we discuss, our results can be easily extended
to the setting with an unknown reward function $\reward$; so as to bring
the essential challenges into clear focus, we take it as known for the
bulk of our development.}  Throughout our discussion, we consider the
i.i.d. observation model, where the dataset consists of $\numobs$
i.i.d. sample pairs \mbox{$\{ (\state_i, \statetwo_i)
  \}_{i=1}^{\numobs} \subset \StateSp \times \StateSp$}.  We let
$\distr$ be any stationary distribution of the Markov chain
$\TransOp$. The sample pair $(\state_i, \statetwo_i)$ is generated by
\begin{align}
\state_i \sim \distr, \qquad \text{and} \qquad \statetwo_i \sim
\TransOp(\cdot \mid \state_i).
\end{align}
The joint distribution induced by the pair $(\distr, \TransOp)$
corresponds to the stationary joint distribution over consecutive
state pairs in the Markov chain.

Given an estimate $\thetahat$ of the value function, we measure its
error in the squared-$\Lmu$-norm
\begin{align}
\munorm{\thetahat - \Vstar}^2 \defn \Exp_{\State \sim \distr} \big[
  \big( \thetahat(\State) - \Vstar(\State) \big)^2 \big],
\end{align}
where $\distr$ is the population distribution of samples $\{ \state_i
\}_{i=1}^{\numobs}$.  In simple cases---such as the tabular setting,
in which the state space $\StateSp$ is a finite set---policy
evaluation is a parametric problem, since the value function can be
encoded as a vector with one entry per state.

Of interest to us in this paper are problems with ``richer'' state
spaces, for which estimating the value function is more challenging,
and often non-parametric in nature.  In such settings, it is standard
to seek approximate solutions of the Bellman operator, via the notion
of a \emph{projected fixed point}
(e.g.,~\cite{bertsekas2011dynamic,tsitsiklis1997analysis,yu2010error,mou2020optimal}).
Given a convex class of functions $\Gclass$ closed in $\Lmu$, the
projection operator $\proj: \Lmu \rightarrow \Gclass$ is given by
\begin{align}
\proj (f) \defn \argmin_{g \in \Gclass} \, \munorm{g - f} \qquad
\text{for any function $f \in \Lmu$}.
\end{align}
We then seek a solution to the projected fixed point equation
\begin{align}
\label{eq:Bellman_project}
\thetastar = \proj \big( \Bellman(\thetastar) \big)
\end{align}
where $\Bellman(\thetastar)(\state) \defn \reward(\state) + \discount
\; \Exp_{\Statetwo \mid \state} \, \thetastar(\Statetwo)$ is the
Bellman operator.  Since the Bellman operator is
contractive\footnote{This fact is a consequence of the choice
$\discount \in (0,1)$ and the non-expansiveness of the transition
  operator on $\Lmu$, due to the stationarity of $\distr$.} in the
$\Lmu$-norm and $\Pi$ is non-expansive, this fixed point equation has
a unique solution.

When the approximating function class $\Gclass$ is chosen to be the
linear span of fixed features, then this approach leads to the
least-squares temporal difference (LSTD) method.  In this paper, our
primary focus is more flexible function classes, as defined by
reproducing kernel Hilbert spaces.  Let us now describe this approach.

%%%%%%%%%%%%%%%%%%%%%%%%%%%%%%%%%%%%%%%%%%%%%%%%%%%%%%%%%%%%%%%%%
	
\subsection{Kernel least-squares temporal differences}
\label{sec:LSTD}

Reproducing kernel Hilbert spaces (RKHSs) provide a fertile ground for
developing non-parametric estimators.  In this paper, we analyze a
standard RKHS-based estimate in reinforcement learning, known as the
kernel least-squares estimate, which we now introduce.  We begin with
some basic background on reproducing kernel Hilbert spaces; see the
books~\cite{Gu02,BerTho04,wainwright2019high} for more details.

An RKHS is a particular type of Hilbert space of real-value functions
$f$ with domain $\StateSp$.  As a Hilbert space, the RKHS has an inner
product $\hilin{f}{g}$ along with the associated norm $\hilnorm{f}$.
The distinguishing property of an RKHS is the existence of a symmetric
kernel function $\KerFun: \StateSp \times \StateSp \rightarrow \real$
that acts as the representer of evaluation.  In particular, for each
$\state \in \StateSp$, the function $z \mapsto \Ker(z, x)$ belongs to
the Hilbert space, and moreover we have
\begin{align}
\hilin{\Ker(\cdot, \state) }{f} = f(\state) \quad \text{for all $f \in
  \RKHS$}.
\end{align}
In order to simplify notation, in much of our development, we adopt
the shorthand $\Rep{\state} = \Ker(\cdot, \state)$ for this
\emph{representer of evaluation}.

The population-level kernel LSTD estimate $\thetastar$ is, by
definition, equal to the projected fixed
point~\eqref{eq:Bellman_project} with the choice $\Gclass = \RKHS$.
Since $\RKHS$ is a reproducing kernel Hilbert space, this fixed point
has a more explicit expression in terms of certain operators defined
on the Hilbert space.  In particular, the covariance and
cross-covariance operators are defined as
\begin{align}
\label{eq:def_CovCr}
\CovOp \defn \Exp_{\State \sim \distr}[\Rep{\State} \otimes
  \Rep{\State}] \quad \text{and} \quad \CrOp \defn \Exp_{(\State,
 \Statetwo) \sim \distr \times \TransOp} [\Rep{\State} \otimes
  \Rep{\Statetwo}] \, .
\end{align}
By construction, the covariance operator $\CovOp(f)$, when applied to
some $f \in \RKHS$, has the property that $\hilin{g}{\CovOp(f)} =
\Exs_{\State \sim \distr}[g(\State) f(\State)]$, with a similar
property for the cross-covariance operator.  In terms of these
operators, the population-level kernel LSTD fixed point must
satisfy\footnote{In writing this equation, we have assumed that the
reward function $\reward$ belongs to the Hilbert space; if not, it
should be replaced by the projection $\Pi(\reward)$.}  the fixed point
relation
\begin{align}
\label{eq:Bellman_project_op}
\CovOp \, \thetastar = \CovOp \, \reward + \discount \; \CrOp \,
\thetastar \, .
\end{align}
When $\RKHS$ is generated by a linear kernel, then the associated
Hilbert space is simply the span of a finite set of features, and
equation~\eqref{eq:Bellman_project_op} defines the population version
of the least-squares temporal difference (LSTD) estimate.  Of more
interest to us in this paper is the estimate defined by richer classes
of kernel functions.

The population-level estimate $\thetastar$ depends on the unknown
operators $\CovOp$ and $\CrossOp$.  In order to obtain an estimator,
we need to replace these unknown quantities with data-dependent
versions.  In this paper, we analyze the \emph{regularized kernel LSTD
estimate} $\thetahat$ given by the solution to the equation
\begin{align}
  \label{eq:def_thetahat}
  \big( \CovOphat + \ridge \IdOp \big) \, \thetahat = \big(\CovOphat +
  \ridge \IdOp\big) \, \reward + \discount \; \CrOphat \, \thetahat.
\end{align}
where $\ridge > 0$ is a user-defined regularization parameter, $\IdOp$
is the identity operator on the Hilbert space, and we have defined the
empirical operators
\begin{align*}
  \CovOphat \mydefn \frac{1}{\numobs} \sum_{i=1}^\numobs
  \Phi_{\state_i} \otimes \Phi_{\state_i}, \quad \text{and} \quad
  \CrOphat \mydefn \frac{1}{\numobs} \sum_{i=1}^\numobs
  \Phi_{\state_i} \otimes \Phi_{\statetwo_i}.
\end{align*}
Note that equation~\eqref{eq:def_thetahat} is a fixed point equation
in the (possibly infinite-dimensional) Hilbert space.  However, as a
consequence of the representer theorem~\cite{KimWah71}, this fixed
point relation can be formulated as an $\numobs$-dimensional linear
system involving kernel matrices.  See \Cref{lemma:matrix} in
\Cref{sec:matrix} for this computationally efficient representation,
which we use in our experiments.

Consider the empirical estimate $\thetahat$ as an estimate of the
unknown value function $\Vstar$.  The error $\munorm{\thetahat -
  \Vstar}$ can be decomposed as
\begin{align}
  \munorm{\thetahat - \Vstar} \leq \underbrace{\munorm{\thetahat -
      \thetastar}}_{\text{Estimation error}} +
  \underbrace{\munorm{\thetastar - \Vstar}}_{\text{Approximation
      error}}.
\end{align}
The approximation error in this decomposition has been studied in past
work, and there are various ways to bound it
(e.g.,~\cite{bertsekas2011dynamic,tsitsiklis1997analysis}); see the
papers~\cite{yu2010error,mou2020optimal} for some refined and optimal
results.

In this paper, our main interest is to study the statistical
estimation error $\munorm{\thetahat - \thetastar}$, and to
characterize its behavior as a function of sample size and structural
properties of the MRP and RKHS.  The eigenvalues of the kernel
integral operator play an important role here; in particular, under
relatively mild conditions (required to satisfy Mercer's theorem, and
assumed here), the kernel function admits a decomposition of the form
\begin{align}
\Ker(x, z) & = \sum_{j=1}^\infty \mu_j \phi_j(x) \phi_j(z),
\end{align}
where $\{\mu_j\}_{j=1}^\infty$ are a non-negative sequence of
eigenvalues, and $\{\phi_j \}_{j=1}^\infty$ are the kernel
eigenfunctions, orthonormal in $\Lmu$.  As we show, the statistical
estimation error is controlled by a kernel complexity function that
depends on the rate at which the eigenvalues decay.

%%%%%%%%%%%%%%%%%%%%%%%%%%%%%%%%%%%%%%%%%%%%%%%%%%%%%%%%%%%%%%%%%

\section{Main results}
\label{SecMainResults}
	
We now turn to the statement of our main results, along with some
discussion of their consequences.  \Cref{SecUpper} is devoted to upper
bounds on the $\Lmu$-error of kernel LSTD estimator, whereas
\Cref{SecLower} provides minimax lower bounds, applicable to any
estimator.

%%%%%%%%%%%%%%%%%%%%%%%%%%%%%%%%%%%%%%%%%%%%%%%%%%%%%%%%%%%%%%%%%%%%%%%%%%%%%

\subsection{Non-asymptotic upper bounds on kernel LSTD}
\label{SecUpper}

Our first main result provides a non-asymptotic upper bound on the
$\Lmu$-error of the kernel LSTD estimator.  We begin by stating the
assumptions under which this upper bound holds.  First, we assume that
the kernel function is uniformly bounded, in the sense that
\begin{align}
  \label{as:b}      
  \sup_{\state \in \StateSp} \sqrt{\Ker(\state,\state)} & \leq \bou
\end{align}
for some finite constant $\bou$.  Note that any continuous kernel
function over a compact domain $\StateSp$ satisfies this condition;
moreover, even on unbounded domains, various standard kernels (e.g.,
Gaussian, Laplacian etc.) satisfy this condition.  

In addition, one of our results---namely, a so-called ``fast
rate''---requires a bound on the sup-norm of the kernel eigenfunctions
$\{\base_j\}_{j=1}^\infty$: that is, we assume that
\begin{align}
  \label{EqnKerEigBound}
\max_{j \geq 1} \|\base_j\|_\infty \leq \unibou \qquad \mbox{for some
  finite quantity $\unibou$.}
\end{align}
For example, any convolutional kernel has eigenfunctions given by the
Fourier basis, and so satisfies this condition.  In the examples that
follow the theorem, we provide additional examples of kernels that
have bounded eigenfunctions.\\

Central to our analysis is a certain inequality, one that arises from
a localized analysis of the empirical process defined by the kernel
class.  The idea of localization is needed in order to obtain optimal
results for standard (non-dynamic) prediction problems; see Chapters
13 and 14 in the book~\cite{wainwright2019high} for background,
including specifics on kernel ridge regression (\S 13.4.2).  Our use
of localization here identifies very clearly how the structural
properties of the Markov reward process determine the statistical
accuracy of the estimate.  In particular, the key ingredients in this
analysis are the following:
\begin{description}
\item[Kernel and stationary distribution:] The kernel function $\Ker$
  interacts with the MRP's stationary distribution $\distr$ so as
  to determine the eigenvalues $\{\eig_j\}_{j=1}^\infty$ of the
  kernel-integral operator.
\item[Effective horizon:] The discount factor $\discount \in (0, 1)$
  enters via the effective horizon $\EffHorizon \defn \tfrac{1}{1 -
    \discount}$.
\item[Structural properties of fixed point:] The structural properties
  of the projected fixed point $\thetastar$ are captured by a
  user-defined radius $\newrad$ such that
  \begin{align}
 \label{eq:def_R}
\newrad & \geq \max \big \{ \hilnorm{\thetastar - \reward}, \tfrac{2
  \supnorm{\thetastar}}{\bou} \big \}.
\end{align}  
\item[Bellman residual variance:] Playing the role of the noise level
  is the variance of the Bellman residual error, when evaluated at
  $\thetastar$.  It is given by
  \begin{align}
\label{EqnBellmanResVar}    
\stdfun^2(\thetastar) & \defn \Exs\Big[ \big(\thetastar(\State) -
  \reward(\State) - \discount \thetastar(\Statetwo)\big)^2 \Big],
\end{align}
where $(\State, \Statetwo)$ are successive samples from the Markov
chain, with the starting state $\State$ drawn according to the
stationary distribution.
\end{description}

%%%%%%%%%%%%%%%%%%%%%%%%%%%%%%%%%%%%%%%%%%%%%%%%%%%%%%%%%%%%%%%%%%%%%%%%%%%%%%%%%%

\subsubsection{Kernel-based critical inequality}

We now turn to the critical inequality that determines the estimation
error of the kernel LSTD estimate.  It is an inequality that involves
the kernel eigenvalues $\{\mu_j\}_{j=1}^\infty$, the radius $\newrad$,
and the discount $\discount$ via the effective horizon
$\EffHorizon(\discount) = (1 - \discount)^{-1}$.  More precisely, we
consider positive solutions $\delta > 0$ to the
\emph{$\SpecialConstant$-based critical inequality}
\begin{align}
  \label{eq:criticalineq}
\tag{$\CI(\SpecialConstant)$} \qquad \KerComplex(\delta) \defn
\sqrt{\sum_{j=1}^{\infty} \min \Big \{ \frac{\eig_j}{\delta^2}, 1 \Big
  \}} \leq \underbrace{\frac{\sqrt{\numobs} \; \newrad}{
    \EffHorizon(\discount) \; \SpecialConstant}}_{\mbox{Slope (SNR)} }
\, \delta \; ,
\end{align}
where $\SpecialConstant > 0$ is a parameter to be specified.  Note
that the function on the left-hand side is decreasing in $\delta$,
whereas the right-hand side is linear in $\delta$ with the indicated
slope.  Consequently, inequality~\eqref{eq:criticalineq} has a unique
smallest positive solution, which we denote by
$\delcrit(\SpecialConstant)$.  To be clear, in addition to depending
on the sample size $\numobs$ and $\SpecialConstant$, this smallest
positive solution also depends on the eigenvalues as well as the pair
$(\newrad, \discount)$, but we suppress this dependence so as to
simplify notation.

To be clear, the relevance of the kernel complexity function
$\KerComplex$ on the left-hand side~\eqref{eq:criticalineq} is
well-known from past work on kernel ridge regression; in particular,
it arises from an analysis of the local Rademacher complexity of a
kernel class (e.g.,~\cite{mendelson2002geometric,wainwright2019high}).
Equally important for understanding kernel-based LSTD methods are the
structural parameters on the right-hand of the critical inequality; as
our results show, these choices capture precisely how the statistical
estimation error of kernel LSTD methods depend on various aspects of
the problem structure.

Since the critical inequality~\eqref{eq:criticalineq} plays a central
role in our analysis, it is worth gaining intuition for how different
components of the MRP affect the solution
$\delcrit(\SpecialConstant)$.  Panel (a) in \Cref{FigKernelComplex}
illustrates the basic geometry of the critical inequality.
\begin{figure}[h]
  \begin{center}
    \begin{tabular}{ccc}
      \widgraph{0.45\textwidth}{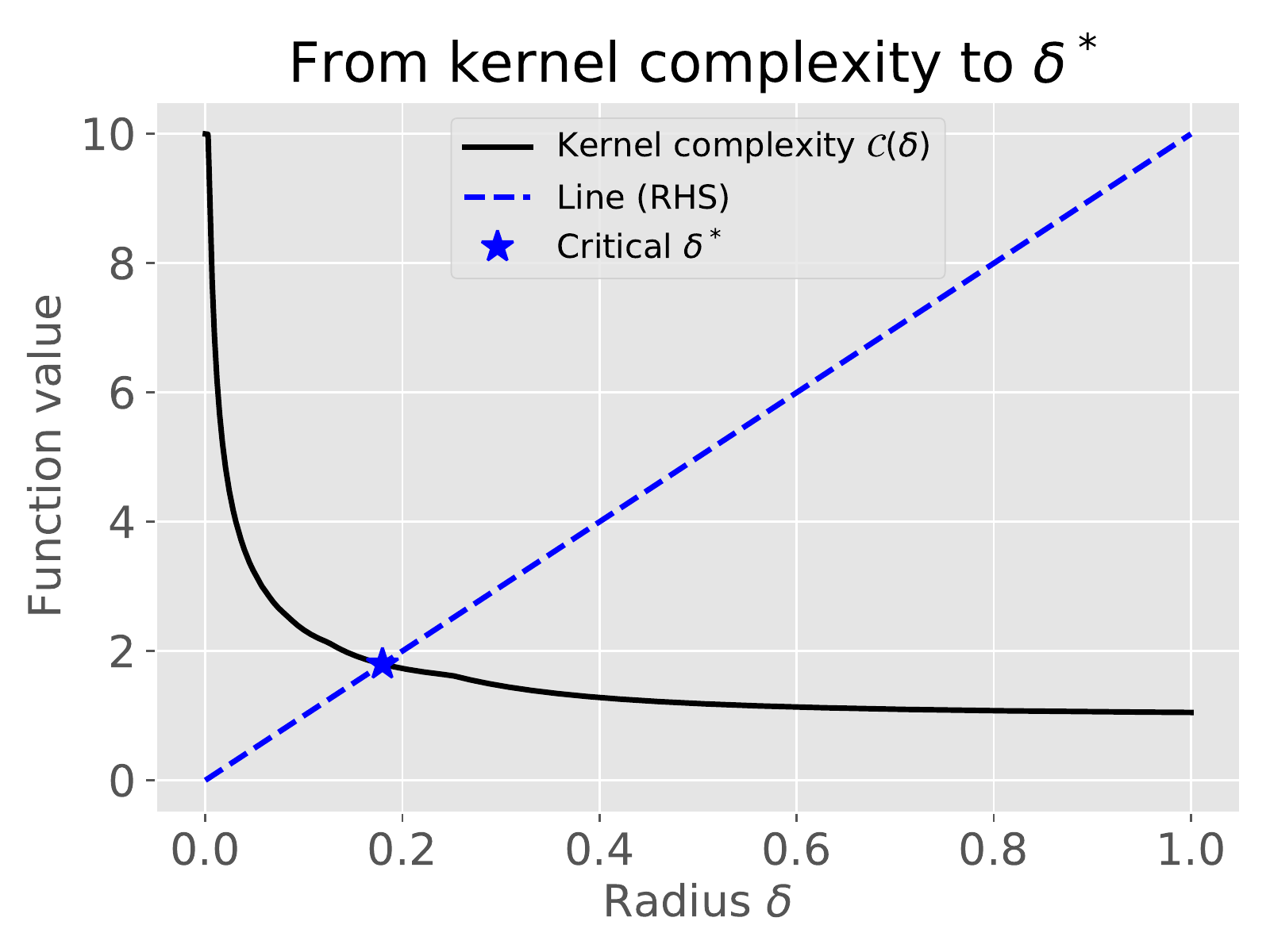} &&
      \widgraph{0.45\textwidth}{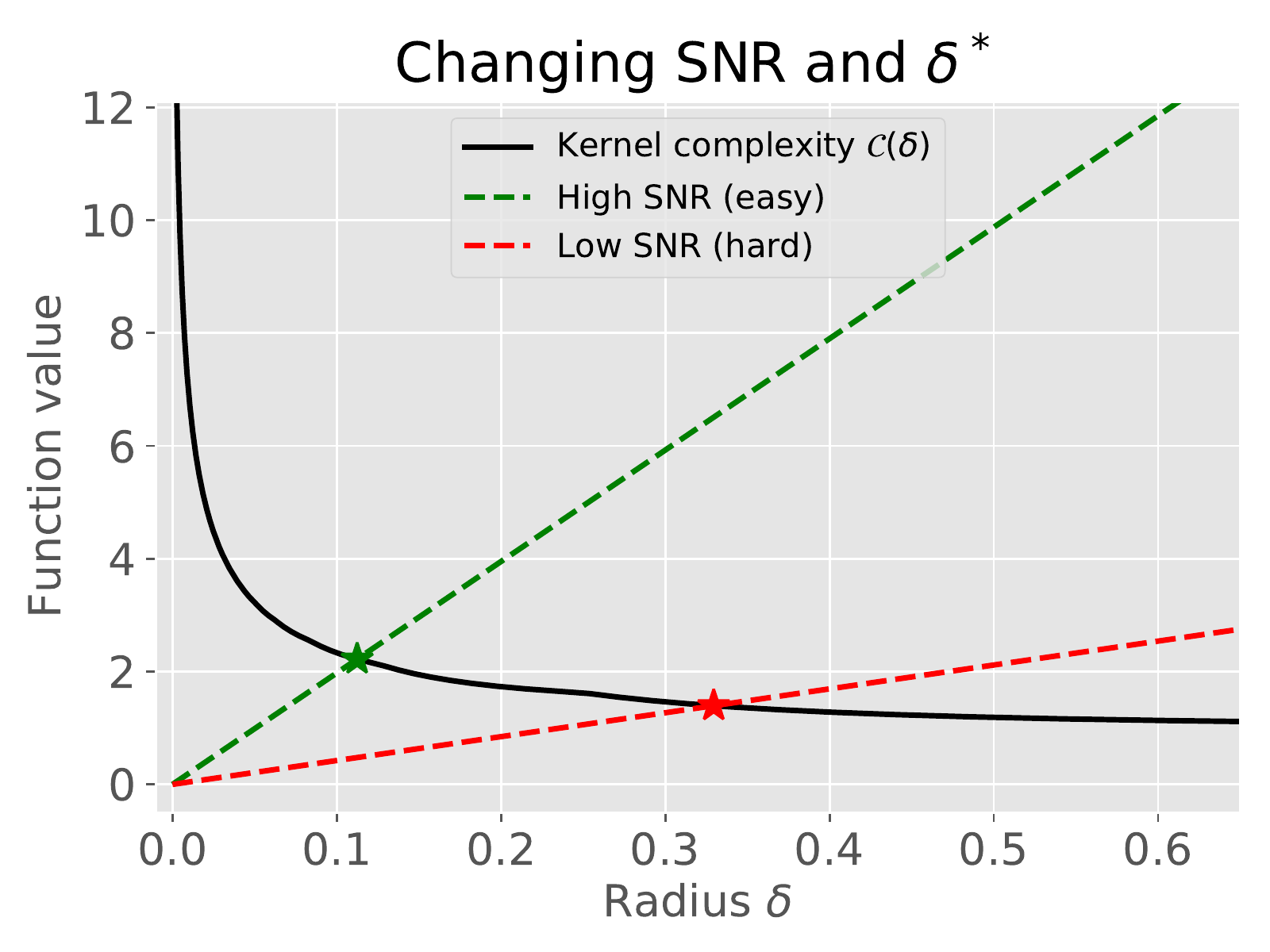} \\
      (a) && (b)
    \end{tabular}
    \caption{Illustrations of the structure of the critical
      inequality~\eqref{eq:criticalineq}.  (a) Plots of the kernel
      complexity $\delta \rightarrow \KerComplex(\delta)$ on the
      left-hand side, along with the linear function on the right-hand
      side.  The critical $\delta^* = \delcrit(\SpecialConstant)$ is
      found at the intersection of this curve and line as marked in a
      blue star.  (b) Effects of changing the slope of the right-hand
      side line, which corresponds to a type of signal-to-noise ratio
      (SNR).  As the SNR decreases, leading to a harder problem, the
      critical $\delta^*$ shifts rightwards to larger values.}
    \label{FigKernelComplex}
  \end{center}
\end{figure}
One instance of the kernel complexity function $\delta \mapsto
\KerComplex(\delta)$, obtained from a kernel with $1$-polynomial
decaying eigenvalues (see equation~\eqref{EqnPolyDecay} in the
sequel), is plotted in black.  Note that this function is
monotonically decreasing in $\delta$.  The dotted blue line
corresponds to the right-hand side, obtained for a particular value of
the slope parameter.  The critical radius $\delta^* \equiv
\delcrit(\SpecialConstant)$, obtained at the intersection of the
kernel complexity of this line, is marked with a star.

The slope on the right-hand side of the
inequality~\eqref{eq:criticalineq} corresponds to a type of
signal-to-noise ratio (SNR).  Panel (b) in \Cref{FigKernelComplex}
shows the effect of changing this SNR parameter.  As the SNR
decreases---so that the slope decreases---the fixed point $\delta^*$
shifts rightward to larger values.  One consequence of our analysis is
that we are able to show precisely the rate at which these leftward
and rightward shifts in the statistical estimation error occur, as a
function of the MRPs structural parameters (in addition to the sample
size $\numobs$).

%%%%%%%%%%%%%%%%%%%%%%%%%%%%%%%%%%%%%%%%%%%%%%%%%%%%%%%%%%%%%%%%%%%%%%%%%%%%%%%%%

\subsubsection{Non-asymptotic upper bounds}

With this set-up and intuition in place, let us now turn to the
statement of our non-asymptotic upper bounds on the quality of the
kernel LSTD estimate.  We provide two guarantees, both of which
involve solutions to the the critical inequality~\ref{eq:criticalineq}
but with different choices of $\SpecialConstant$.  In each case, the
tightest bound is afforded by $\delcrit(\SpecialConstant)$.  We make
two different choices of $\SpecialConstant$. First, we establish a
bound, one that holds for all sample sizes, with the choice
$\SpecialConstant = \bou \newrad$.  We then prove a sharper result,
one that holds for a finite sample size that is suitably lower
bounded, and involves setting $\SpecialConstant = \unibou
\stdfun(\thetastar)$, where $\stdfun^2(\thetastar)$ is the variance of
the Bellman residual error~\eqref{EqnBellmanResVar}.

Both parts of our theorem guarantee that the kernel LSTD estimator
satisfies a bound of the form
\begin{align}
\label{EqnMetaBound}  
    \munorm{\thetahat - \thetastar}^2 & \leq \plaincon_1 \newradsq \:
    \left \{ \delta^2 + \frac{\ridge}{1-\discount} \right \}
  \end{align}
with probability at least $1 - 2 \exp \big(-\frac{\plaincon_2 \numobs
  \delta^2 \SqDis^2}{\bou^2}\big)$, where $(c_1, c_2)$ are universal
constants.  The two parts differ in the allowable settings of $\delta$
and $\ridge$ for which the bound~\eqref{EqnMetaBound} holds.

\begin{theorem}[Non-asymptotic upper bounds]
\label{thm:ub}
There is a universal constant $c_0$ such that:
\begin{enumerate}
\item[(a)] \underline{Slow rate:} Under the kernel boundedness
  condition~\eqref{as:b}, the bound~\eqref{EqnMetaBound} holds for any
  solution \mbox{$\delta = \delta(\numobs, \newrad, \discount, \bou)$}
  to the critical inequality $\CI(\bou \newrad)$ and any $\ridge \geq
  \plaincon_0 \delta^2 \SqDis$.

\item[(b)] \underline{Fast rate:} Suppose in addition that the kernel
  eigenfunctions are uniformly bounded~\eqref{EqnKerEigBound}.  Let
  $\delcrit(\unibou\stdfun(\thetastar))$ be the smallest solution to
  the critical inequality $\CI(\unibou \stdfun(\thetastar))$, and
  suppose that $\numobs$ is large enough to ensure that
  \begin{align}
\label{EqnNbound}    
\newrad^2 \delcritsq( \unibou \stdfun(\thetastar)) & \leq
\frac{\unibou \, \stdfun^2(\thetastar)}{200 \SqDis \, \sqrt{\numobs}}.
\end{align}
Then the bound~\eqref{EqnMetaBound} holds for any solution $\delta =
\delta(\numobs, \newrad, \discount, \stdfun(\thetastar))$ to the
critical inequality $\CI(\unibou \stdfun(\thetastar))$ and any
\mbox{$\ridge \geq \plaincon_0 \delta^2 \SqDis$.}
\end{enumerate}
\end{theorem}
\noindent The proof of this result, given in \Cref{SecProofThmUB},
involves first proving a ``basic inequality'' that is satisfied by the
error $\thetahat - \thetastar$.  We then use empirical process theory
and concentration inequalities to establish high probability bounds on
the terms in this basic inequality.\\

\subsection{A simpler bound and some corollaries}

It should be noted that the bounds in \Cref{thm:ub} hold if we set
$\delta = \delcrit$, corresponding to the smallest positive solution
to the critical inequality~\ref{eq:criticalineq}, along with $\ridge =
\plaincon_0 \SqDis \delcritsq$.  We are then guaranteed to have
\begin{align}
\label{EqnSimplerUpper}
\munorm{\thetahat - \thetastar}^2 & \leq \underbrace{\plaincon_1 (1 +
  \plaincon_0)}_{\defn c'} \, \newradsq \delcritsq
\end{align}
with probability at least $1 - 2 \exp \big(- \tfrac{\plaincon_2
  \numobs \delcritsq \SqDis^2}{\bou^2} \big)$.  Let us consider some
examples of this simpler upper bound to illustrate.

\subsubsection{Linear kernels and standard LSTD}
\label{SecExaLinear}

We begin by considering the special case of a linear kernel, in which
case the kernel LSTD estimate reduces to the classical linear LSTD
estimate.  Given a $\usedim$-dimensional feature map of the form
$\varphi: \StateSp \rightarrow \real^\usedim$, let us consider linear
value functions $\theta(\state) \defn \inprod{\theta}{\varphi(\state)}
\; = \; \sum_{j=1}^\usedim \theta_j \varphi_j(\state)$.  Here we have
overloaded the notation in letting $\theta \in \real^\usedim$ denote a
parameter vector.  Similarly, we write the reward function as
$\reward(\state) = \inprod{\reward}{\varphi(\state)}$ for some vector
$\reward \in \real^\usedim$.

In this case, the Hilbert space can be identified with $\real^\usedim$
equipped with the Euclidean inner product as the Hilbert inner
product, and the vector $\varphi(\state) \in \real^\usedim$ plays the
role of the representer of evaluation.  Note that we have
$\hilnorm{\thetastar - \reward} = \|\thetastar - \reward\|_2$, and
since $\Ker(\state, \statenew) =
\inprod{\varphi(\state)}{\varphi(\statenew)}$, the covariance operator
takes the form $\CovOp = \Exp \big[ \varphi(\State)
  \varphi(\State)^\top \big] \; = \; \sum_{j=1}^\usedim \mu_j v_j
v_j^{\top}$, a $\usedim$-dimensional symmetric positive semidefinite
matrix with eigenvalues $\{\eig_j\}_{j=1}^\usedim$, and eigenvectors
$\{v_j\}_{j=1}^\usedim$.  We have $\Ker(\state,\statenew) =
\inprod{\varphi(\state)}{\varphi(\statenew)}$, and so
\begin{align*}
  \bou = \sup_{\state \in \StateSp} \sqrt{\Ker(\state,\state)} =
  \max_{\state} \|\varphi(\state)\|_2 \quad \mbox{and} \quad \unibou =
  \sup_{\state \in \StateSp} \max_{j \geq 1}
  |\inprod{v_j}{\varphi(\state)}|.
\end{align*}

We now study the structure of the critical
inequality~\ref{eq:criticalineq}, and derive two bounds
for the standard LSTD estimate.  Both bounds are of the form
\begin{align}
  \label{EqnLinearBound}
  \munorm{\thetahat - \thetastar}^2 & = \Exp \big[ \inprod{ \thetahat
      - \thetastar}{\varphi(\State)}^2 \big] \leq
  \simpleub (\SpecialConstant) \defn \plaincontwo
  \frac{\SpecialConstant^2}{\SqDis^2} \; \frac{\usedim}{\numobs}
\end{align}
for different choices of $\SpecialConstant$, and hold with probability
at least $1 - 2 \exp \big(- \tfrac{\plaincon_2 \numobs
  \simpleub(\SpecialConstant) \SqDis^2}{\bou^2 \newrad^2} \big)$.  We
summarize as follows:
\begin{corollary}[Linear kernels and standard LSTD]
  \label{CorLinear}
For the linear kernel and associated standard LSTD estimate:
\begin{enumerate}
\item[(a)] For any sample size $\numobs$, the
  bound~\eqref{EqnLinearBound} holds with  
\begin{subequations}  
  \begin{align}
\label{EqnLinearSlow}    
\simpleub (\bou \newrad) & = \plaincontwo \frac{\bou^2
  \newrad^2}{\SqDis^2} \; \frac{\usedim}{\numobs}
\end{align}

\item[(b)] For a sample size lower bounded as $\sqrt{\numobs} \geq
  \frac{200 \unibou \usedim}{1-\discount}$, the
  bound~\eqref{EqnLinearBound} holds with
\begin{align}  
\label{EqnLinearFast}        
\simpleub \big(\unibou \stdfun(\thetastar) \big) & = \plaincontwo
\frac{\unibou^2 \stdfun^2(\thetastar)}{\SqDis^2} \;
\frac{\usedim}{\numobs}.
\end{align}
\end{subequations}
\end{enumerate}
\end{corollary}
\begin{proof}
For any $\delta > 0$, we have $
  \sum_{j=1}^{\usedim} \min \big\{ \frac{\eig_j}{\delta^2}, 1 \big\} \leq \usedim$.  Consequently, the critical
inequality~\ref{eq:criticalineq} is satisfied as long as $
\sqrt{\usedim} \leq \tfrac{\sqrt{\numobs} \newrad \,
  \SqDis}{\SpecialConstant} \, \delta$.  The smallest $\delta =
\delta(\SpecialConstant)$ is given by
\begin{align*}
\newrad^2 \delta^2(\SpecialConstant) & = \frac{
  \SpecialConstant^2}{\SqDis^2} \frac{\usedim}{\numobs}.
\end{align*}
Setting $\SpecialConstant = \bou \newrad$ yields the
claim~\eqref{EqnLinearSlow}.

As for the faster rate claimed in the bound~\eqref{EqnLinearFast}, we
need to check when the requirement of \Cref{thm:ub}(b)---in particular
the sample size condition~\eqref{EqnNbound}---is satisfied. In this
case, we have $\newrad^2 \delcritsq(\unibou\stdfun(\thetastar)) \leq
\frac{\unibou^2 \stdfun^2(\thetastar)}{\SqDis^2}
\frac{\usedim}{\numobs}$, so that in order to satisfy the
bound~\eqref{EqnNbound}, it suffices to have $\sqrt{\numobs} \geq
\frac{200 \unibou \usedim}{1-\discount}$.  The
claim~\eqref{EqnLinearFast} then follows.
\end{proof}

%%%%%%%%%%%%%%%%%%%%%%%%%%%%%%%%%%%%%%%%%%%%%%%%%%%%%%%%%%%%%%%%%%%%%%%%%%%%%%%%

\subsubsection{Kernels with $\alpha$-polynomial decay}
\label{SecExaSpline}

Let us now consider a ``richer'' class of kernel functions, for which
the kernel estimator is truly non-parametric. In particular, let us consider the
class of kernels that satisfy the \emph{$\alpha$-polynomial decay
condition}
\begin{align}
\label{EqnPolyDecay}  
\eig_j & \leq c \, j^{-2 \alpha} \quad \mbox{for some exponent $\alpha
  > \tfrac{1}{2}$.}
\end{align}
There are many examples of kernels used in practice that satisfy a
decay condition of this form, including the Laplacian kernel
$\Ker(x,x') = \exp(-\|x - x'\|_1)$, as well as various types of
Sobolev and spline kernels that are used in non-parametric regression
and density estimation.  See Chapters 12 and 13 in the
book~\cite{wainwright2019high} for more details on such kernels.

Let us study the structure of the critical
inequality~\ref{eq:criticalineq} for kernels whose eigenvalues satisfy
the $\alpha$-polynomial decay condition~\eqref{EqnPolyDecay}.  We
derive two bounds, both of which are of the form
\begin{align}
  \label{EqnSplineBound}
  \munorm{\thetahat - \thetastar}^2 & \leq \simpleub(\SpecialConstant)
  \defn \newrad^2 \; \underbrace{\plaincontwo \Big(
    \frac{\SpecialConstant^2}{\newrad^2 \SqDis^2} \frac{1}{\numobs}
    \Big)^{\tfrac{2 \alpha}{2 \alpha +
        1}}}_{\delta^2(\SpecialConstant)} \; = \; c \,
  \newrad^{\frac{2}{2 \alpha +1}} \Big(
  \frac{\SpecialConstant^2}{\SqDis^2} \frac{1}{\numobs} \Big)^{\frac{2
      \alpha}{2 \alpha+1}},
\end{align}
for different choices of $\SpecialConstant$, and hold with probability
at least $1 - 2 \exp \big(- \tfrac{\plaincon_2 \numobs
  \delta^2(\SpecialConstant) \SqDis^2}{\bou^2} \big)$.

\begin{corollary}
\label{CorSpline}
\begin{enumerate}
\item[(a)] For any sample size $\numobs$, the
  bound~\eqref{EqnSplineBound} holds with  
\begin{subequations}  
  \begin{align}
\label{EqnSplineSlow}    
\simpleub(\bou \newrad) & = \plaincontwo \newrad^2 \Big( \frac{\bou^2}{\SqDis^2}
\frac{1}{\numobs} \Big)^{\tfrac{2 \alpha}{2 \alpha + 1}}.
\end{align}

\item[(b)] Suppose that the sample size $\numobs$ is large enough to
  ensure that $\newrad^2 \delcritsq(\stdfun(\thetastar)) \leq
  \frac{\unibou \, \stdfun^2(\thetastar)}{\SqDis \, \sqrt{\numobs}}$.
  Then the bound~\eqref{EqnSplineBound} holds with
\begin{align}  
\label{EqnSplineFast}        
\simpleub \big(\unibou \stdfun(\thetastar) \big) & = \plaincontwo \newrad^2
\Big(
\frac{\unibou^2 \stdfun^2(\thetastar)}{\newrad^2 \SqDis^2}
\frac{1}{\numobs} \Big)^{\tfrac{2 \alpha}{2 \alpha + 1}}.
\end{align}
\end{subequations}
\end{enumerate}
\end{corollary}

\begin{proof}
Let us find a solution to the critical
inequality~\ref{eq:criticalineq} for a kernel satisfying the
$\alpha$-polynomial decay condition~\eqref{EqnPolyDecay}.  Let $k$ be
the largest positive integer such that $\delta^2 \leq c k^{-2
  \alpha}$.  With this choice, we have
\begin{align*}
\sqrt{\sum_{j=1}^{\infty} \min \Big\{ \frac{\eig_j}{\delta^2}, 1 \Big\}} & \leq \sqrt{ k +
  \frac{c}{\delta^2} \sum_{j=k+1}^\infty j^{-2 \alpha }}.
\end{align*}
Now we have
\begin{align*}
  \sum_{j=k+1}^\infty j^{-2 \alpha} & \leq \int_{k}^\infty t^{-2
    \alpha} dt \; \leq \; \frac{1}{2 \alpha - 1} (1/k)^{2 \alpha - 1}.
\end{align*}
Consequently, there is a universal constant $c'$, depending only on
$\alpha$, such that the critical inequality~\ref{eq:criticalineq} will
be satisfied for a $\delta > 0$ such that $c' 
\delta^{- \tfrac{1}{2\alpha}} \; \leq \sqrt{\numobs} \, \frac{\newrad
  \SqDis}{\SpecialConstant} \delta$.  Solving this inequality yields
that
\begin{align*}
\delta^2 & \asymp \Big( \frac{\SpecialConstant^2}{\newrad^2
  \SqDis^2} \frac{1}{\numobs} \Big)^{\tfrac{2 \alpha}{2 \alpha
    + 1}}
\end{align*}
satisfies the critical inequality~\ref{eq:criticalineq}.

Putting together the pieces, we conclude that there is a universal
constant $c$ such that
\begin{align}
  \label{EqnNewBound}
\munorm{\thetahat - \thetastar}^2 & \leq c \newrad^2 \Big(
\frac{\SpecialConstant^2}{\newrad^2 \SqDis^2} \frac{1}{\numobs}
\Big)^{\frac{2 \alpha}{2 \alpha+1}} \; = \; c \, \newrad^{\frac{2}{2
    \alpha +1}} \Big( \frac{\SpecialConstant^2}{\SqDis^2}
\frac{1}{\numobs} \Big)^{\frac{2 \alpha}{2 \alpha+1}}
\end{align}
with high probability.  This bound holds with $\SpecialConstant = \bou
\newrad$ for all sample sizes, and it holds with $\SpecialConstant =
\unibou \stdfun(\thetastar)$ once the sample size is sufficiently
large to ensure that
\begin{align*}
  \newrad^2 \delcritsq(\stdfun(\thetastar)) & \asymp \newrad^{\frac{2}{2
      \alpha + 1}} \Big( \frac{ \unibou^2
    \stdfun^2(\thetastar)}{\SqDis^2} \frac{1}{\numobs}
  \Big)^{\frac{2 \alpha}{2 \alpha+1}} \; \lesssim \;
  \frac{\unibou \, \stdfun^2(\thetastar)}{\SqDis \, \sqrt{\numobs}}.
\end{align*}
Since $\tfrac{2 \alpha}{2 \alpha + 1} > \tfrac{1}{2}$, this bound will
hold once $\numobs$ exceeds a finite threshold.
\end{proof}

\subsection{Some illustrative simulations}
\label{SecSimulations}

Some simulations are useful in illustrating the predictions of our
theory, and most concretely the sharpness of
Corollary~\ref{CorSpline}.  In particular, from the
bound~\eqref{EqnSplineBound}, the error depends on the eigenvalue
exponent $\alpha$ from equation~\eqref{EqnPolyDecay} in two distinct
ways.  On one hand, the dependence on the effective horizon $H =
\tfrac{1}{1-\discount}$ worsens as the exponent $\alpha$ increases.
On the other hand, the dependence on the inverse sample size
$(1/\numobs)$---corresponding to how quickly the estimation error
vanishes---improves as $\alpha$ increases.  Corollary~\ref{CorSpline}
makes very explicit predictions about these dependencies, and the
sharpness of these predictions can be verified empirically.

In order to do so, we constructed three different kernels $\Ker_i$, $i
= 1, 2, 3$ with eigenvalues $\{\eig_j\}_{j=1}^\infty$ decaying as
\begin{align}
\eig_j(\Ker_i) & = \begin{cases} j^{-6/5} & \mbox{for $i = 1$} \\
j^{-2} & \mbox{for $i = 2$} \\
 \exp \big( - (j-1)^2 \big) & \mbox{for $i = 3$.}
  \end{cases}
\end{align}
Note that $\Ker_1$ has $\alpha$-polynomial decay~\eqref{EqnPolyDecay}
with $\alpha = 3/5$, $\Ker_2$ with $\alpha = 1$, and the exponential
decay of $\Ker_3$ can be viewed as a limiting case $\alpha = +\infty$.

In parallel, we constructed two different probability transition
functions that allowed us to vary the dependence of the radius
$\newrad$ and the Bellman residual variance $\stdfun^2(\thetastar)$ on
the effective horizon.
\begin{description}
\item[``Hard'' ensemble]: The transition function underlying our hard
  ensemble is constructed so that
  \begin{align*}
    \newrad \asymp \stdfun^2(\thetastar) \asymp
    \frac{1}{1-\discount},
  \end{align*}
  where the notation $\asymp$ means bounded above and below by
  constants independent of $\discount$.
\item[``Easy'' ensemble:] For our easy ensemble, we construct the
  probability transition matrix and rewards so that both $\newrad$ and
  $\stdfun^2(\thetastar)$ remain of constant order as $\discount$ is
  varied.
\end{description}
See \Cref{AppSimDetails} for more details on these constructions.  In
all cases, we implemented the kernel LSTD estimate using the
regularization parameter $\ridge = c (1 - \discount) \delcritsq$ for a
fixed constant $c = 0.01$.

%%%%%%%%%%%%%%%%%%%%%%%%%%%%%%%%%%%%%%%%%%%%%%%%%%%%%%%%%%%%%%%%%

\subsubsection{Dependence on sample size}

We begin by studying the dependence of the kernel LSTD estimator on
the sample size.  For any kernel with $\alpha$-polynomial
decay~\eqref{EqnPolyDecay}, Corollary~\ref{CorSpline} predicts that
the mean-squared error should decay as
\begin{align}
  \munorm{\thetahat - \thetastar}^2 & \asymp \left(
  \frac{1}{\numobs} \right)^{\frac{2 \alpha}{2 \alpha +1}},
\end{align}
where, for this particular comparison, we disregard other terms that
are independent of the sample size $\numobs$.  This decay rate is a
standard one in the context of non-parametric
regression~\cite{Stone82,wainwright2019high}, so to be expected here
as well.

\begin{figure}[h]
  \begin{center}
\begin{tabular}{ccc}
\widgraph{0.45\textwidth}{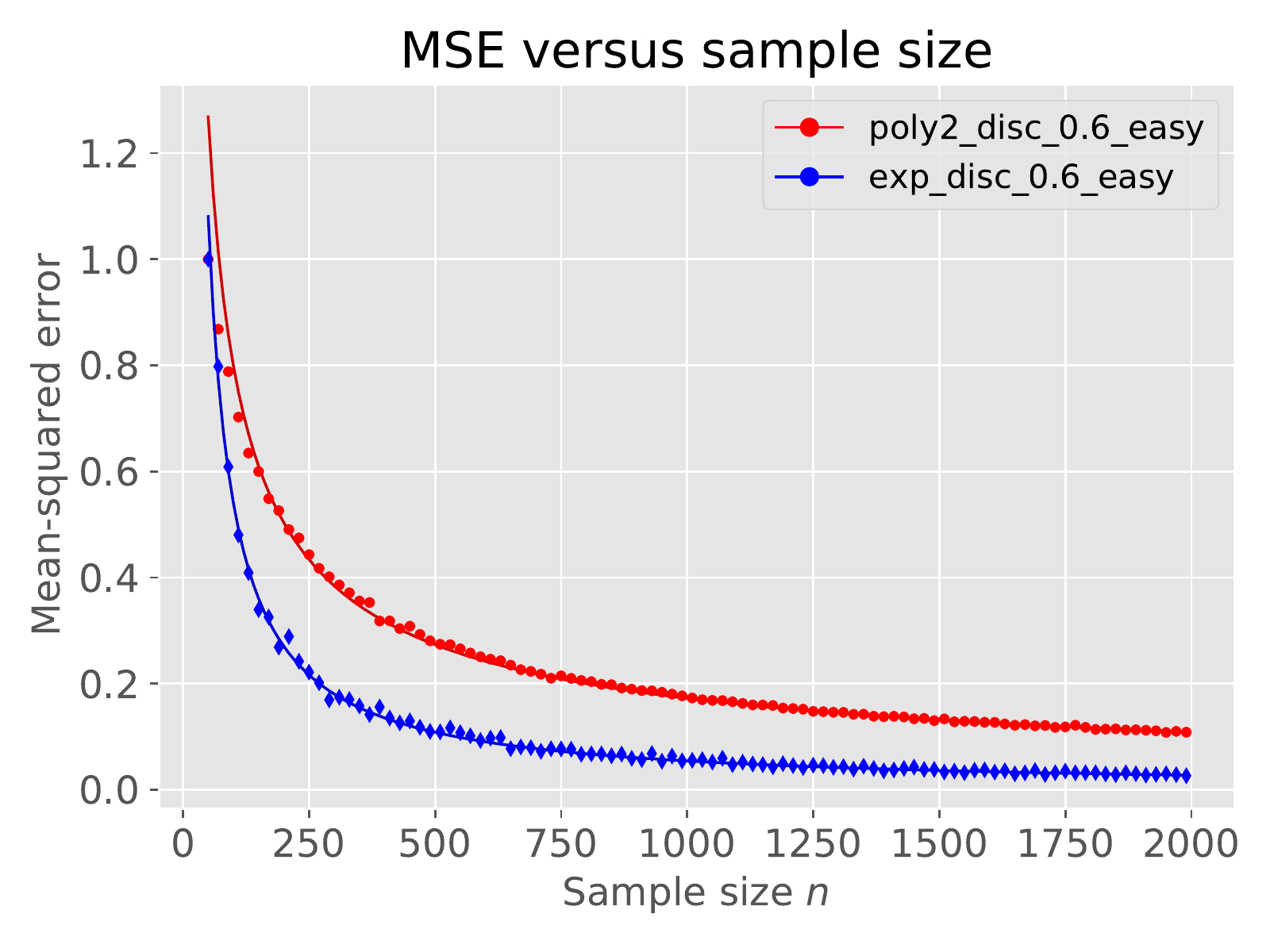} &&
\widgraph{0.45\textwidth}{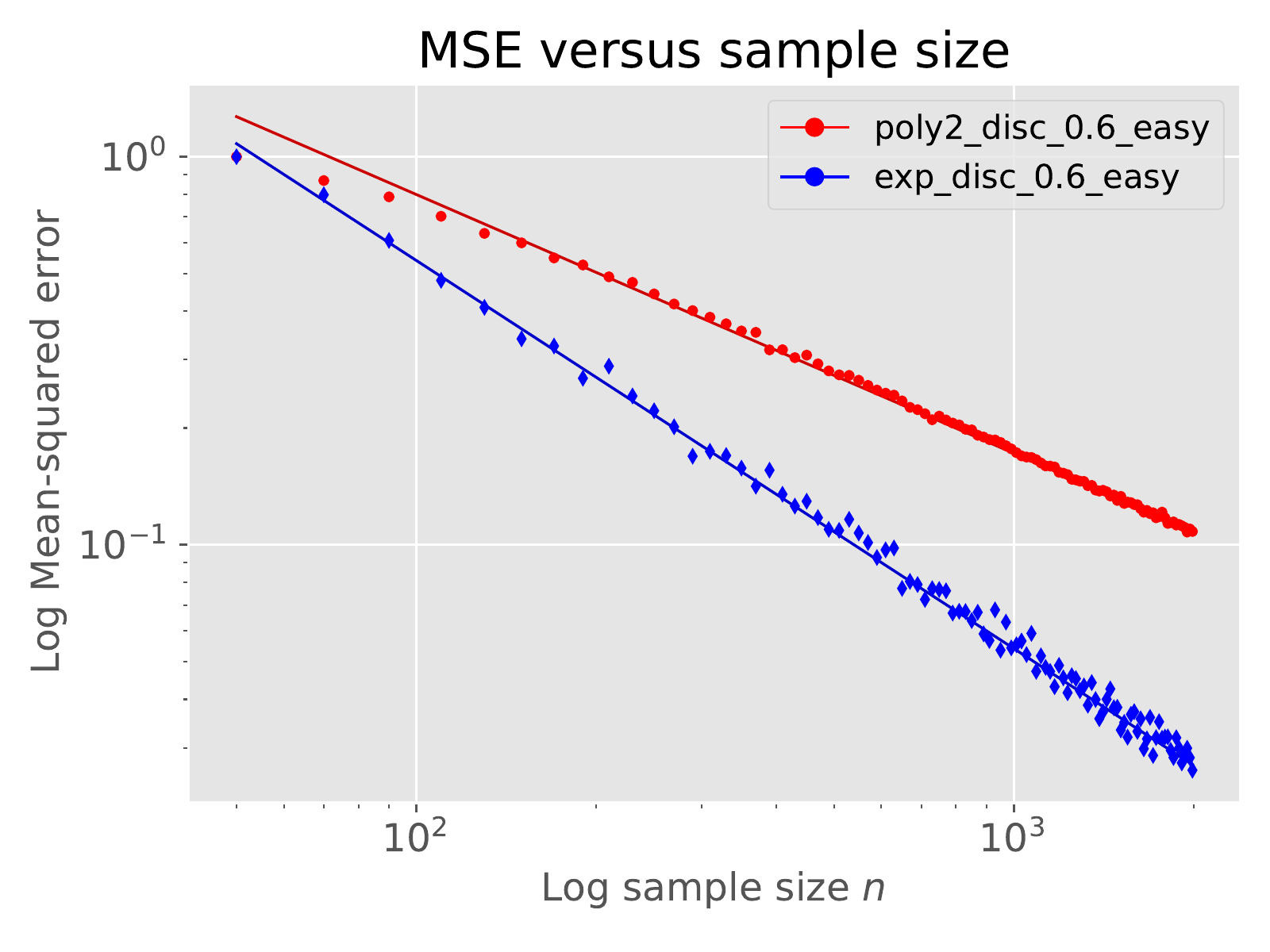} \\
(a) && (b)
\end{tabular}
% Some weird LATEX bug just below
    \caption{Plots of the mean-square error $\Exs \norm{\thetahat -
        \thetastar}_{\distr}^2$ versus the sample size $\numobs$ for two
      different kernels.  For each point (for each curve on each
      plot), the MSE was approximated by taking a Monte Carlo average
      over $T = 2000$ trials with the sample standard deviations shown
      as error bars.  Our theory predicts that the mean-squared error
      should drop off as $\big(\frac{1}{\numobs} \big)^\specexp$ for
      an exponent $\specexp > 0$ determined by the kernel.  Solid
      curves correspond to these theoretical predictions. (a) MSE
      versus sample size on ordinary scale.  Our theory predicts that
      (disregarding logarithmic factors), the MSE should scale as
      $\big(\frac{1}{\numobs} \big)$ for the exponential kernel
      $\Ker_3$, and as $\big(\frac{1}{\numobs} \big)^{2/3}$ for the
      $1$-polynomial decaying kernel $\Ker_2$.  Note that these
      theoretical predictons align very well with the empirical
      behavior.  (b) Plots of the same data on a log-log scale,
      showing the expected linear relationship between log MSE and log
      sample size.}
\label{FigNsampFits}
  \end{center}
\end{figure}

\subsubsection{Dependence on effective horizon}

In our second simulation study, we examine the behavior of the
$\Lmu$-error as a function of the effective horizon $\EffHorizon \defn
\frac{1}{1-\discount}$.  For kernels with eigenvalues that exhibit
$\alpha$-polynomial decay, our theory---in particular via the
bound~\eqref{EqnSplineFast} from Corollary~\ref{CorSpline}---gives
specific predictions about this dependence as well.
\begin{itemize}
\item With the probablity transitions from the ``hard'' ensemble, it
  can be shown that $\newrad \asymp \stdfun^2(\thetastar) \asymp
  \EffHorizon =\tfrac{1}{1 - \discount}$.  As a consequence, our
  theory predicts that for a fixed sample size $\numobs$, we should
  observe the following scaling
  \begin{subequations}
    \begin{align}
      \label{EqnSlopeA}
  \munorm{\thetahat - \thetastar}^2 & \asymp \EffHorizon^{\frac{2 (3
      \alpha + 1)}{2 \alpha + 1}}.
    \end{align}      
\item With the probability transitions from our ``easy'' ensemble, for
  which $\newrad \asymp \stdfun^2(\thetastar) \asymp 1$, the predicted
  slope of this linear scaling is
  \begin{align}
    \label{EqnSlopeB}
  \munorm{\thetahat - \thetastar}^2 & \asymp \EffHorizon^{\frac{4
      \alpha}{2 \alpha + 1}}.
  \end{align}
  \end{subequations}
\end{itemize}
See \Cref{AppSimDetails} for the calculations of both of these
theoretical predictions.  Note that predictions for the kernel
$\Ker_3$, with its exponentially decaying values, can be obtained as a
limiting case with $\alpha \rightarrow +\infty$.

\begin{figure}[h]
  \begin{center}
    \begin{tabular}{ccc}
      \widgraph{0.45\textwidth}{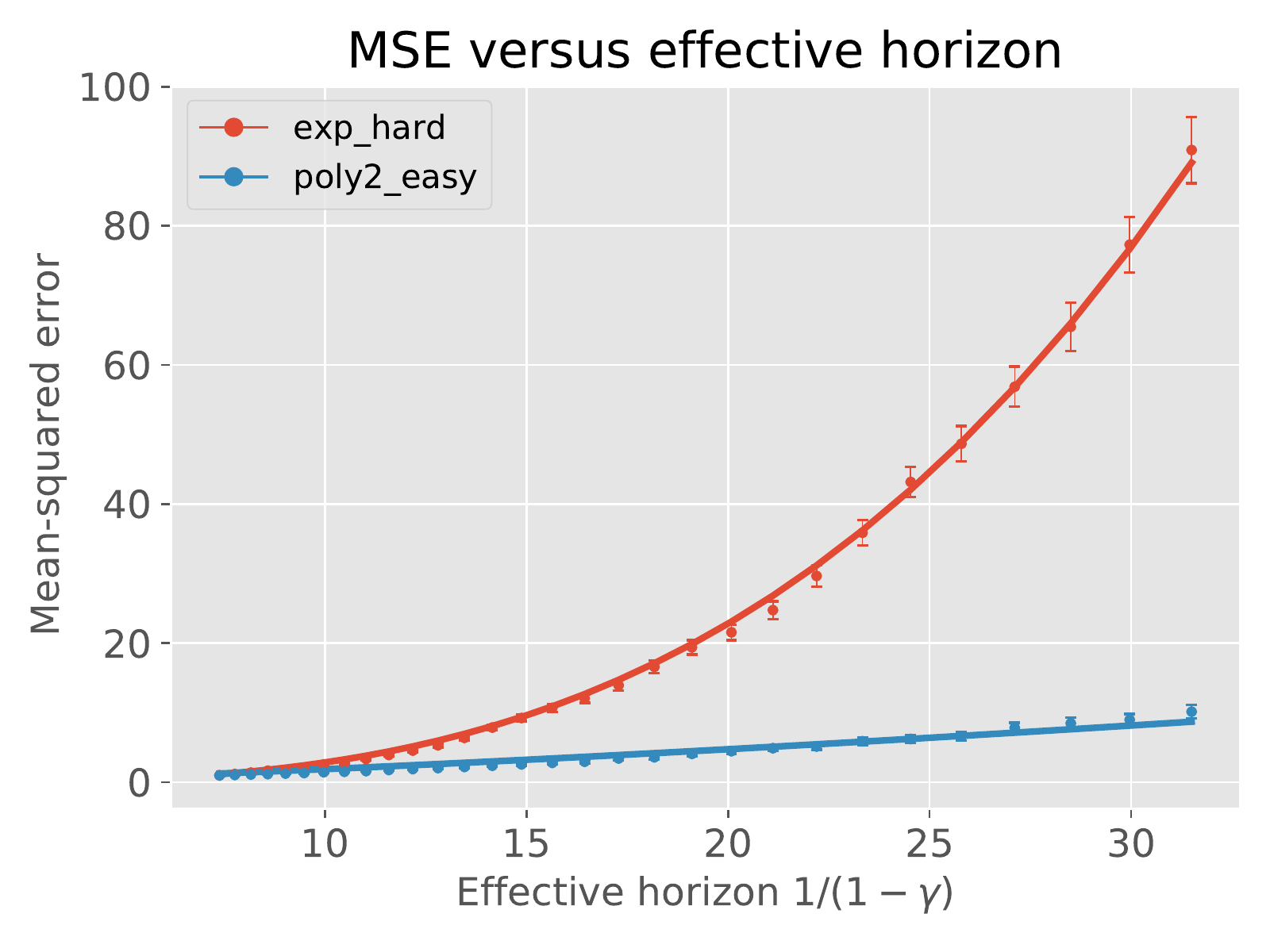} &&
      \widgraph{0.45\textwidth}{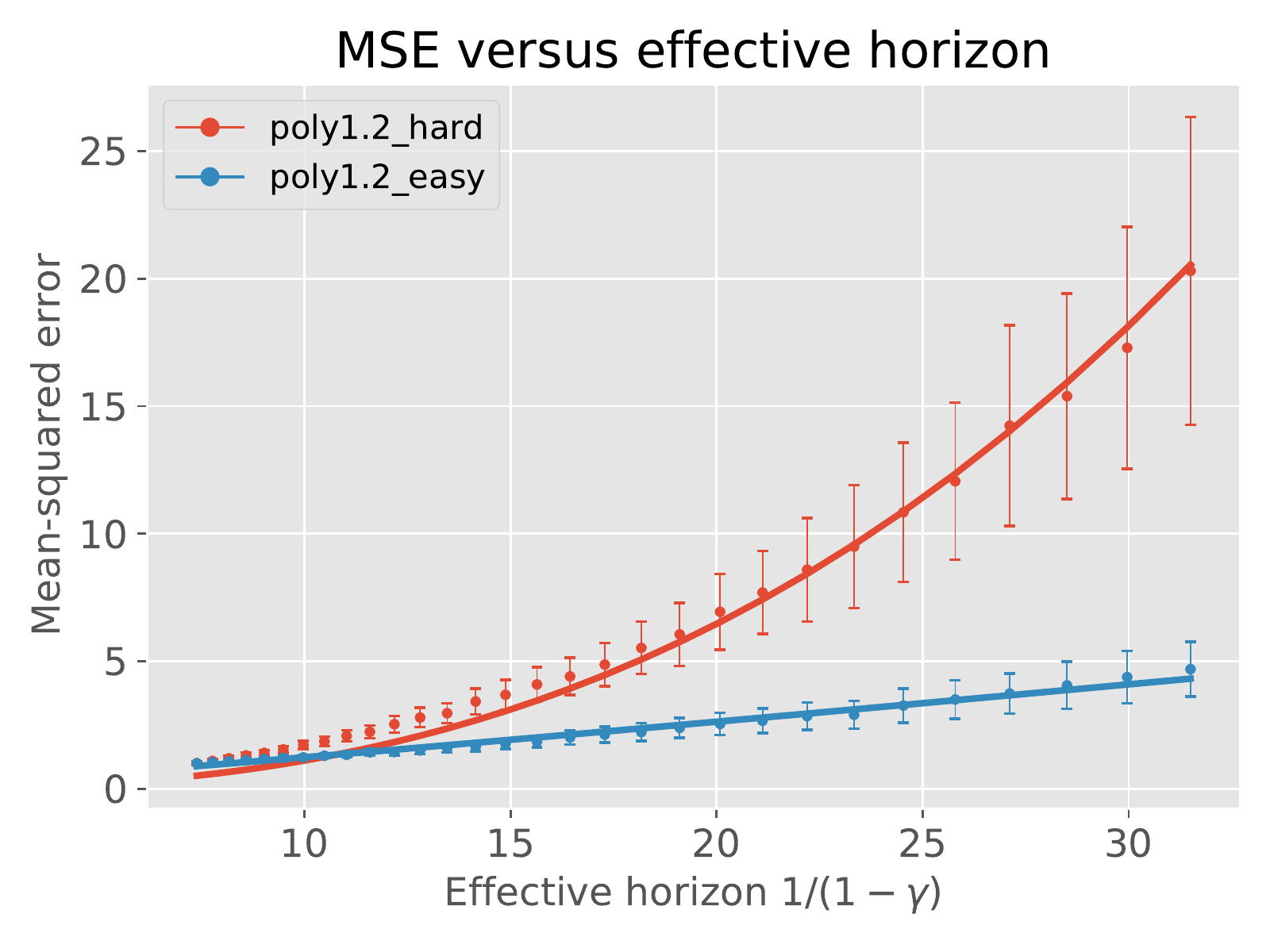} \\
      (a) && (b)
    \end{tabular}
    % Some weird Latex bug just below here
    \caption{ Plots of the mean-square error $\Exs\norm{\thetahat
        - \thetastar}_{\distr}^2$ versus the effective horizon $\EffHorizon =
      \tfrac{1}{1-\discount}$ for different ensembles of problems.
      For each point (on each curve in each plot), the MSE was
      approximated by taking a Monte Carlo average over $T = 1000$
      trials; sample standard deviations are shown as error bars.  Our
      theory predicts that the MSE should grow as a function of the
      form $\EffHorizon^\eta$, where the exponent $\eta > 0$ is
      determined by the kernel and the ensemble type.  Solid curves
      correspond to these theoretical predictions. (a) Plots comparing
      exponential decay kernel $\Ker_3$ under the ``hard'' ensemble to
      the \mbox{$1$-polynomial} decay kernel $\Ker_3$ under the
      ``easy'' ensemble.  Theory predicts that the MSE scales as
      $\EffHorizon^3$ and $\EffHorizon^{1.33}$ in these two cases
      respectively; as shown, these theoretical predictions agree well
      with the empirical results.  (b) Plots comparing the behavior of
      the kernel $\Ker_1$ (with $0.6$-polynomial decay) under the
      ``hard'' ensemble versus the ``easy'' ensemble.  Theory predicts
      that the MSE should scale as $\EffHorizon^{2.55}$ and
      $\EffHorizon^{1.09}$ in these two cases.}  
    \label{FigDiscountFits}
  \end{center}
\end{figure}

%%%%%%%%%%%%%%%%%%%%%%%%%%%%%%%%%%%%%%%%%%%%%%%%%%%%%%%%%%%%%%%%%%%%%%%%%%%%%%%%

\subsection{Minimax lower bounds}
\label{SecLower}

Thus far, we have established some upper bounds on the performance of
a specific estimator.  To what extent are these bounds improvable?  In
order to answer this question, it is natural to investigate the
fundamental (statistical) limitations of kernel-based value function
estimation.  In this section, we do so by deriving some minimax lower
bounds on the behavior of any procedures for estimating the value
function.

Minimax lower bounds are obtained by assessing the performance of any
estimator in a uniform sense over a particular class of problems.  In
particular, for classes of MRPs $\MRPclass$ to be defined, we prove
lower bounds of the following type.  For a given MRP instance $\MRP$,
we assume that we observe a dataset $\{ (\state_i, \statetwo_i) \}_{i=1}^\numobs$
of $\numobs$ i.i.d. samples generated from the given MRP.  An
estimator $\thetahat$ of the value function is any measurable function
of the data mapping into $\Real^{\StateSp}$.  For suitable classes
$\MRPclass$ indexed by pairs of parameters $(
\newradbarreward, \stdbarreward)$, we prove that the squared-$\Lmu$ error of any
estimator, when measured in a uniform sense over the family, is lower
bounded as $\plaincon_1 \newradbarreward^2 \delcrit^2$.  Here
$\plaincon_1 > 0$ is a universal constant, and the error parameter
$\delcrit$ %that 
is determined in same way as the critical
inequality~\eqref{eq:criticalineq} that specifies our upper bounds;
see equation~\eqref{eq:CI_lb} for the precise definition.

%%%%%%%%%%%%%%%%%%%%%%%%%%%%%%%%%%%%%%%%%%%%%%%%%%%%%%%%%%%%%%%%%%%%%%%%%%%%

\subsubsection{Families of MRPs and regular kernels}

We begin by describing the families of MRPs over which we prove
minimax lower bounds.  In all of our constructions, both the reward
function $\reward$ and the optimal value function $\thetastar$ are
members of a Hilbert space with a set of eigenfunctions
$\{\base_j\}_{j=1}^\infty$, and a sequence of eigenvalues $\{\eig_j
\}_{j=1}^\infty$ that vary as part of the construction.  In all
cases, our construction ensures that the eigenfunction
bound~\eqref{EqnKerEigBound} holds with $\unibou = 2$, along with the
kernel being trace class.  In particular, we have
\begin{subequations}
  \begin{align}
\label{EqnCondA}    
\max \limits_{j \geq 1} \supnorm{\base_j} \leq \unibou = 2, \quad
\mbox{and} \quad \sum_{j=1}^\infty \eig_j \leq \tfrac{\bou^2}{4}.
\end{align}
Note that
  these conditions imply that
\begin{align*}
 \sup_{\state \in \StateSp} \sqrt{\Ker(\state,\state)} = \sup_{\state
   \in \StateSp} \Big( \sum_{j=1}^\infty \eig_j \base_j^2(\state)
 \Big)^{1/2} \leq \bou,
\end{align*}
so that the $\bou$-boundedness condition~\eqref{as:b} from our upper
bound holds.  In addition, our families of MRPs are also defined by
the constraints
\begin{align}
\label{EqnCondB}  
  \max
  \big\{ \hilnorm{\thetastar - \reward}, \tfrac{2
  	\supnorm{\thetastar}}{\bou} \big\} \leq \newradbarreward, \quad \mbox{and} \quad \stdfun(\thetastar) \leq \stdbarreward.
\end{align}
\end{subequations}
We say that a family $\MRPclass$ of MRPs is
\emph{$(\newradbarreward,
\stdbar)$-valid} if its members satisfy the bound~\eqref{EqnCondB},
along with the conditions~\eqref{EqnCondA}.

So as to match our upper bounds, we prove lower bounds that involve an
error term $\delcrit$ defined as the smallest positive solution to the
inequality
\begin{align}
  \label{eq:CI_lb}
  \sqrt{\sum_{j=1}^{\infty} \min \big\{ \frac{\eig_j}{\delta^2}, 1
    \big\}} & \leq \sqrt{\numobs} \; \frac{\newradbar \, \SqDis}{2
    \stdbar } \; \delta.
\end{align}
From past work on kernel ridge regression~\cite{yang2017randomized},
it is known that such lower bounds cannot hold for kernels with
eigenvalues that decay in pathological ways.  The notion of a regular
kernel, which we define here, precludes such pathology.  For a given
$\delcrit$, the associated statistical dimension $\statdim \equiv
\statdim(\delcrit)$ is given by $\statdim(\delcrit) \defn \max\big\{ j
\mid \eig_j \geq \delcritsq \big\}$.  The kernel is regular if there
is a universal constant $c$ such that
\begin{align}
\label{eq:kernel_reg}
\Big\{ \frac{ 2 \stdbar }{\newradbar \, \SqDis} \Big\}^2
\statdim \geq c \; \numobs \, \delcritsq.
\end{align}
Standard kernels, including the linear kernel and more general kernels
with eigenvalues that decay at a polynomial or exponential rate, are
all regular.

%%%%%%%%%%%%%%%%%%%%%%%%%%%%%%%%%%%%%%%%%%%%%%%%%%%%%%%%%%%%%%%%%%%

\subsubsection{Statement of bounds}

With this set-up, we are now ready to state our minimax lower bounds.
For a given \emph{$(\newradbarreward, \stdbar)$-valid} family of MRPs,
we say that the lower bound $\LB(\newradbarreward, \stdbarreward,
\delcrit)$ holds if
\begin{align}
\label{EqnLB}  
\inf_{\thetahat} \sup_{\MRP \in \MRPclass(
  \newradbarreward, \stdbarreward)} \Prob_\MRP \Big( \munorm{\thetahat -
  \thetastar}^2 \; \geq \; \plaincon_1 \; \newradbarreward^2
\delcrit^2 \Big) \geq \plaincon_2. \quad \tag{$\LB(\newradbarreward,  \stdbarreward, \delcrit)$}
\end{align}
In this statement, the quantities $(\plaincon_1, \plaincon_2)$ are
universal constants.

We prove minimax lower bounds in two regimes of parameters
$(\newradbarreward, \stdbarreward)$, depending on how these parameters
scale with the effective horizon $\tfrac{1}{1-\discount}$.  In Regime
A, this scaling is linear in the effective horizon---namely
\begin{subequations}
\begin{align}
\label{EqnRegimeA}  
\newradbarreward \geq \tfrac{1}{6\SqDis} \max\big\{
\tfrac{\discount}{\sqrt{\eig_1}}, \tfrac{2}{\bou} \big\}, \quad \mbox{and} \quad \stdbarreward^2 \in \Big[ {\tfrac{1+\discount}{5\SqDis}},
{\tfrac{1+\discount}{1-\discount}}\Big].
\end{align}
In Regime B, by contrast, both of these quantities can be order one
with the effective horizon---viz.
\begin{align}
\label{EqnRegimeB}  
\newradbarreward \geq \max\big\{ \tfrac{1}{2 \sqrt{\eig_1}},
\tfrac{2}{\discount \bou} \big\}, \quad \mbox{and} \quad \stdbarreward^2 \in \big(\tfrac{1}{8}, 1].
\end{align}
\end{subequations}
We discuss the motivation for considering these two regimes following
the statement of our bounds.
\begin{theorem}[Minimax lower bounds]
  \label{thm:lb}
  \begin{enumerate}
  \item[(a)] For any pair $(\newradbarreward, \stdbarreward)$ in
    Regime A~\eqref{EqnRegimeA}, there is a $(\newradbarreward, \stdbarreward)$-valid family of MRPs such that the lower bound
    $\LB(\newradbarreward, \stdbarreward, \delcrit)$ holds for any
    sample size $\numobs$ such that
    \begin{subequations}
    \begin{align}
      \label{eq:ncond}
      \newradbarreward^2 \delcrit^2 \leq \frac{2 \; \unibou
        \stdbarreward^2}{\SqDis^{\frac{3}{2}} \sqrt{\numobs}}.
    \end{align}
  \item[(b)] Consider any pair $(\stdbarreward, \newradbarreward)$ in
    Regime B~\eqref{EqnRegimeB}, and suppose that the eigensequence
    satisfies \mbox{$\min \limits_{3 \leq j \leq \statdim} \big\{
      \sqrt{\eig_{j-1}} - \sqrt{\eig_j} \big\} \geq
      \frac{\delcrit}{2\statdim}$.}  Then there is a $( \newradbarreward, \stdbarreward)$-valid family of MRPs such that the lower bound
    $\LB(\newradbarreward, \stdbarreward, \delcrit)$ holds for a
    sample size $\numobs$ large enough such that
\begin{align}
  \label{eq:ncond_2}
\newradbarreward^2 \delcrit^2 \leq \frac{12 \; \unibou
  \stdbarreward^2}{\SqDis \sqrt{\numobs}} \qquad \text{and} \qquad
\newradbar \delcrit \leq 10 \unibou\stdbar \, \big( 1 -
\tfrac{\eig_2}{\eig_1} \big) \, \min\Big\{\frac{\unibou
  \stdbar/(\sqrt{\eig_1}\newradbar)}{\SqDis^2\log \numobs },
\frac{\sqrt{\eig_1}}{\bou} \Big\}.
\end{align}
    \end{subequations}
 \end{enumerate}
\end{theorem}
\noindent See \Cref{SecProofThmLB} for the proof of \Cref{thm:lb}. \\

The main take-away from this result is the following: by comparing the
bounds in \Cref{thm:lb} with the achievable rate from
\Cref{thm:ub}(b), we see that the kernel LSTD estimator is an optimal
procedure.  More precisely, it achieves the minimax-optimal scaling
$\newradbarreward^2 \delcrit^2$ of the squared-$\Lmu$ norm. As we
discuss below, there are some differences in the minimum sample size
required for the bounds to be valid, with the lower bound requirements
being less stringent than our upper bounds from \Cref{thm:ub}. \\

A few high-level comments on the proof: it is based on the Fano method
for proving minimax lower bounds (see Chapter 15 in the
book~\cite{wainwright2019high} for background).  This method involves
constructing a family of MRP instances that are ``well-separated'',
and arguing that any method with relatively low estimation error is
capable of solving a multi-way testing problem defined over this
family.  Our construction of the family of MRPs is relatively simple;
each instance has state space $\StateSp = [0, 1)$, and the kernels
  are designed with eigenfunctions defined by the Walsh basis.

\paragraph{Some differences:}  Our upper and lower bounds
differ in terms of their required lower bounds on sample size; as we
discuss in \Cref{AppLower}, the requirements of the lower bounds in
\Cref{thm:lb} are milder than our corresponding condition for the
kernel LSTD estimate.  Apart from the sample size conditions,
\Cref{thm:lb} also requires the kernel regularity
condition~\eqref{eq:kernel_reg}, along with the eigensequence
condition in part (b).  As we discuss in more detail in
\Cref{AppLower}, this conditions are relatively mild, and satisfied by
various kernels used in practice (including any kernel with
eigenvalues that exhibit $\alpha$-polynomial
decay~\eqref{EqnPolyDecay}).

\paragraph{Regimes of $(\newradbarreward, \stdbarreward)$:}  Let us
now discuss the two regimes of parameters.
\bcar
\item
The scalings in Regime A~\eqref{EqnRegimeA} arise naturally when we
assume only that the reward function is uniformly bounded---say
$\supnorm{\reward} \leq 1$.  In this case, by the law of total
variance~\cite{sobel1982variance}, we have the bound
$\stdfun^2(\thetastar) \leq \tfrac{2}{1-\discount}$, and there exist
MRPs for which this $(1-\discount)^{-1}$ is achieved, consistent with
the first inclusion in condition~\eqref{EqnRegimeA}.  In terms of the
choice of $\newradbarreward$, we can construct MRPs with bounded
reward functions such that $\munorm{\thetastar-\reward} \lesssim
\frac{1}{1-\discount}$ and $\supnorm{\thetastar} \lesssim
\frac{1}{1-\discount}$.  With these scalings, the constraint on
$\newradbarreward$ in condition~\eqref{EqnRegimeA} is satisfied.
\item Turning to Regime B~\eqref{EqnRegimeB}, it corresponds to a
  class of problems for which estimation is much easier.  Instances
  with this scaling arise when we impose a constraint of the form
  $\discount\munorm{\thetastar} \leq 1$.  This constraint ensures that
  $\stdfun^2(\thetastar) \leq 1$ because the variance is dominated by
  the second moment.  As for the parameter $\newradbarreward$, the
  RKHS norm is connected with the $\Lmu$-norm via inequality
  $\hilnorm{\thetastar-\reward} \geq \frac{1}{\sqrt{\eig_1}}
  \munorm{\thetastar-\reward}$.  Therefore, we can ensure that the
  constraint $\newradbar \geq \max\big\{ \hilnorm{\thetastar -
    \reward}, \tfrac{2\supnorm{\thetastar}}{\bou} \big\}$ holds by
  constructing MRPs with $\munorm{\thetastar-\reward} \lesssim 1$ and
  $\supnorm{\thetastar} \lesssim \frac{1}{\discount}$.  With these
  choices, we can ensure that $\newradbar \gtrsim \max\big\{
  \frac{1}{\sqrt{\eig_1}}, \frac{1}{\discount\bou} \big\}$, as
  required in the definition~\eqref{EqnRegimeB}.
\ecar

%%%%%%%%%%%%%%%%%%%%%%%%%%%%%%%%%%%%%%%%%%%%%%%%%%%%%%%%%%%%%%%%%%%%%%%%%%%%%%%

\section{Proofs}
\label{SecProof}

We now turn to the proofs of our main results.  \Cref{SecProofThmUB}
is devoted to the proofs of the upper bounds stated in \Cref{thm:ub},
whereas \Cref{SecProofThmLB} contains the proofs of the lower bounds
stated in \Cref{thm:lb}.

%%%%%%%%%%%%%%%%%%%%%%%%%%%%%%%%%%%%%%%%%%%%%%%%%%%%%%%%%%%%%%%

\subsection{Proof of Theorem~\ref{thm:ub}}
\label{SecProofThmUB}
	
The proof of the finite-sample upper bounds stated in \Cref{thm:ub}
consists of three steps. First, we use the definition of the estimator
to derive a basic inequality to give an upper bound on the the squared
$\Lmu$ error. Then we use techniques from empirical process theory and
concentration of measure to upper bound the terms on the right-hand
side of this inequality.  Finally, we exploit this analysis to choose
the regularization parameter $\ridge$ in a manner that yields an
optimal trade-off between the bias and variance terms.

%%%%%%%%%%%%%%%%%%%%%%%%%%%%%%%%%%%%%%%%%%%%%%%%%%%%%%%%%%%%%%%%%%%%%%%%%%%%%%%
	
\subsubsection{The building blocks}

Recall that $\thetahat$ denotes our estimate, whereas $\thetastar$
denotes the population-level kernel LSTD solution.  We begin our
analysis by deriving an inequality that must be satisfied by the the
error $\Deltahat = \thetahat - \thetastar$.  We state our results in
terms of the functional
\begin{align}
\label{EqnDefnSpecfun}  
\specfun(f) & \defn \Big( \Exs[f^2(\State) - \discount f(\State) f(\Statetwo)]
\Big)^{1/2},
\end{align}
where $(\State, \Statetwo)$ are successive states sampled from the
Markov chain, with $\State$ drawn according to the stationary
distribution.  As shown in the proof of \Cref{lemma:decomp} below, it
follows from the Cauchy-Schwarz inequality that we always have the
lower bound
\begin{align}
\label{EqnUsefulLower}  
\Exs[f^2(\State) - \discount f(\State) f(\Statetwo)] & \geq \SqDis
\munorm{f}^2 \; \geq \; 0.
\end{align}
so that our definition of $\specfun$ is meaningful.

\paragraph{A basic inequality on the error:}
We begin by stating an inequality that must be satisfied by the
error. It lies at the foundation of our analysis:
\begin{lemma}[Basic inequality]
\label{lemma:decomp}
The error $\Deltahat = \thetahat - \thetastar$ satisfies the
inequality
\begin{align}
  \label{eq:decomp}
  \SqDis \munorm{\Deltahat}^2 \stackrel{(i)}{\leq}
  \specfun^2(\Deltahat) & \stackrel{(ii)}{=} \Big \{ \sum_{j=1}^3
  \Term_j \Big \} - \ridge \hilnorm{\Deltahat}^2,
\end{align}
where
\begin{subequations}
  \begin{align}
    \Term_1 = & \hilin[\big]{\Deltahat}{ \CovOphat(\reward -
      \thetastar) + \discount \CrOphat \thetastar}, \\
\Term_2 = & \ridge \hilin[\big]{\Deltahat}{\reward - \thetastar}, \\
\Term_3 = & \hilin[\big]{\Deltahat}{(\GamOp - \GamOpHat) \Deltahat},
  \end{align}
where $\GamOp = \CovOp - \discount \CrOp$ and $\GamOpHat = \CovOphat -
\discount \CrOphat$.
\end{subequations}
\end{lemma}
\noindent See \Cref{sec:proof:lemma:decomp} for the proof of this
claim.

%%%%%%%%%%%%%%%%%%%%%%%%%%%%%%%%%%%%%%%%%%%%%%%%%%%%%%%%%%%%%%%%

\paragraph{Controlling the terms:}

Our next step is to derive upper bounds on the three terms on the
right-hand side of our basic inequality~\eqref{eq:decomp}.  Recall
that $\hilnorm{\thetastar - \reward} \leq \newrad$ by assumption. 
%In
%order to simplify matters, we assume in this analysis that $\bou = 1$;
%this scaling can be achieved by rescaling the kernel by $1/\bou$.
%(Accordingly, any terms involving $\newrad$ should be converted to
%$\newrad/\bou$ to handle the general case.)

The quantity $\Term_2$ is easily handled: we have
\begin{align}
  \label{eq:T2}
  \Term_2 \stackrel{(i)}{\leq} \ridge \hilnorm{\Deltahat}
  \hilnorm{\reward - \thetastar} \stackrel{(ii)}{\leq}
  \frac{\ridge}{2} \left \{ \hilnorm{\Deltahat}^2 + \newrad^2 \right
  \},
\end{align}
where step (i) follows from the Cauchy-Schwarz inequality, and step
(ii) follows from the Fenchel-Young inequality.

As for the terms $\Term_1$ and $\Term_3$, we state some auxiliary
lemmas that bound them with high probability.
\begin{lemma}
  \label{lemma:T1}
  Let $\delcrit = \delcrit(\SpecialConstant)$ for either
  $\SpecialConstant = \bou \newrad$ or $\SpecialConstant = \unibou
  \stdfun(\thetastar)$.  There are are universal constants
  $(\plaincon, c')$ such that
\begin{align}
 \label{eq:T1}
 \Term_1 \leq  \plaincon \, \SqDis \; \delcritsq \, \Big
 \{\hilnorm{\Deltahat}^2 + \newrad^2 \Big \} + \plaincon \, \newrad \, \SqDis \; \delcrit
 \munorm{\Deltahat}
  \end{align}
with probability at least $1 - \exp \big ( - \plaincontwo \tfrac{\numobs
  \delcritsq \SqDis^2}{\bou^2} \big )$. 
\end{lemma}
\noindent See \Cref{sec:proof:lemma:T1} for the proof of this
claim. \\

\begin{lemma}
  \label{lemma:T3}
(a) With the choice $\delcrit = \delcrit(\bou \newrad)$ there are
  universal constants $(c, c')$ such that
\begin{align}
\label{eq:T3}
\Term_3 \leq \plaincon \, \SqDis \;  \delcritsq \, \Big \{
\hilnorm{\Deltahat}^2 + \newrad^2 \Big \} + \frac{\specfun^2(\Deltahat)}{2},
\end{align}
  with probability at least $1 - \exp \big( - \plaincontwo
  \frac{\numobs \delcritsq \SqDis}{\bou^2} \big)$. \\
(b) If, in addition, the sample size condition~\eqref{EqnNbound}
  holds, then the same bound holds with \mbox{$\delcrit =
    \delcrit(\unibou \stdfun(\thetastar))$.}
\end{lemma}
\noindent See \Cref{sec:proof:lemma:T3} for the proof of this claim.

%%%%%%%%%%%%%%%%%%%%%%%%%%%%%%%%%%%%%%%%%%%%%%%%%%%%%%%%%%%%%%%%%%%%%%%%%%%%%%%

\subsubsection{Putting together the pieces}
\label{sec:choice_ridge}

We now put together the pieces in order to complete the proof of
\Cref{thm:ub}.  In particular, we use \Cref{lemma:T1,lemma:T3} to
bound the terms $\{ \Term_j \}_{j=1}^3$ on the right hand side of the
bound~\eqref{eq:decomp} from \Cref{lemma:decomp}. Applying all of
these bounds and combining all the terms, we find that with probablity
at least $1 - 2 \exp\big( - c' \frac{\numobs \delcritsq \SqDis^2}{\bou^2}
\big)$, we have
\begin{align*}
\specfun^2(\Deltahat) \leq & \underbrace{\plaincon \SqDis \delcritsq \Big \{\hilnorm{\Deltahat}^2 + \newrad^2
	\Big \} + \plaincon \newrad \SqDis \,
  \delcrit \munorm{\Deltahat}}_{\mbox{\small{Bound on $\Term_1$}}} +
\underbrace{\frac{\ridge}{2} \Big \{\hilnorm{\Deltahat}^2 + \newrad^2
  \Big \}}_{\mbox{\small{Bound on $\Term_2$}}} \\ & + \underbrace{\plaincon
  \SqDis \, \delcritsq \Big \{\hilnorm{\Deltahat}^2 + \newrad^2
  \Big \} + \tfrac{\specfun^2(\Deltahat)}{2}}_{\mbox{\small{Bound on
      $\Term_3$}}} - \ridge \hilnorm{\Deltahat}^2.
\end{align*}
Re-arranging terms yields
\begin{align*}
  \tfrac{1}{2} \specfun^2(\Deltahat) & \leq c \newrad
  \, \SqDis \, \delcrit \munorm{\Deltahat} +
  \hilnorm{\Deltahat}^2 \Big \{ 2 \plaincon \SqDis \delcritsq -
  \frac{1}{2} \ridge \Big \} + \newrad^2 \Big \{ 2 \plaincon
  \SqDis \delcritsq + \frac{1}{2} \ridge \Big \}.
\end{align*}
Setting $\ridge \geq 4 \plaincon \SqDis  \delcritsq$
ensures that the second term is negative.  Combining with the lower
bound $\specfun^2(\Deltahat) \geq \SqDis
\munorm{\Deltahat}^2$, we find that
\begin{align*}
  \frac{1-\discount}{2} \munorm{\Deltahat}^2 & \leq \plaincon \newrad
  \SqDis \, \delcrit \munorm{\Deltahat} + \ridge \newrad^2.
\end{align*}
By the Fenchel-Young inequality, we have
\begin{align*}
\plaincon \newrad \, \SqDis \, \delcrit \munorm{\Deltahat} + \ridge
\newrad^2 & \leq \tfrac{1-\discount}{4} \munorm{\Deltahat}^2 + \plaincon^2 \newrad^2 \, \SqDis \, \delcritsq + \ridge \newrad^2.
\end{align*}
Putting together the pieces, we conclude that there is a universal
constant $\cbar$ such that
\begin{align*} % \label{eq:final_ineq}
\munorm{\Deltahat}^2 & \leq \cbar \newrad^2 \left \{ \delcrit^2 + \frac{\ridge}{1-\discount} \right \},
\end{align*}
as claimed. This concludes the proof of \Cref{thm:ub} for $\delta =
\delcrit$.

We note that all of the same steps actually hold for any $\delta \geq
\delcrit$, so that the bound given in the theorem is also valid.

%%%%%%%%%%%%%%%%%%%%%%%%%%%%%%%%%%%%%%%%%%%%%%%%%%%%%%%%%%%%%%%%%%%%%%%%%%%%%%

\subsection{Proof of Theorem~\ref{thm:lb}}
\label{SecProofThmLB}

We now turn to the proof of the minimax lower bounds stated in
\Cref{thm:lb}. 
In \Cref{sec:def_MRPclass}, we explicitly define the MRP families $\MRPclassA$ and $\MRPclassB$. 
\Cref{sec:lb_Fano} then provides a
high-level overview of the proof structure which works for both Regimes A and B.
\Cref{sec:2stateMRP,sec:RKHS,sec:baseMRP,sec:familyofMRP} are devoted to the detailed arguments, including the constructions of
RKHSs $\RKHSA$ and $\RKHSB$ and MRP
instances in model families $\MRPclassA$ and
$\MRPclassB$.

%%%%%%%%%%%%%%%%%%%%%%%%%%%%%%%%%%%%%%%%%%%%%%%%%%%%%%%%%%%%%%%%%%%%%%%%%%%%%%

\subsubsection{Full specification of the minimax lower bound} \label{sec:def_MRPclass}

We begin with the full specification of minimax lower bound \ref{EqnLB}, regarding the set-up of problem instance $\MRP$ and definitions of MRP family $\MRPclass$.

We first precisely define a problem instance $\MRP$, especially the
roles of stationary distribution therein.  In either Regime A or B, we fix an
RKHS $\RKHS = \RKHSA$ or $\RKHSB$ and a reward function $\reward =
\rewardA$ or $\rewardB$ such that $\reward \in \RKHS$, and then
consider MRPs of the form $\MRP(\TransOp, \reward, \discount)$. Throughout the proof of lower bounds, we let
$\distr(\TransOp)$ be the stationary distribution associated with
transition kernel $\TransOp$ and always use notation $\distr$ to denote the Lebesgue measure. The stationary distribution
$\distr(\TransOp)$ plays multiple roles.  First, the observation pairs
$\{ (\state_i, \statetwo_i) \}_{i=1}^{\numobs}$ are generated by
drawing $\state_i$ from distribution $\distr(\TransOp)$, and then the
successor state $\statetwo_i$ from the probability transition.  Note
moreover that $\munorm{\cdot}$ in equation~\eqref{EqnLB} is an
abbreviation of the
$L^2\big(\distr(\TransOp)\big)$-norm. Specifically, we measure the
estimation error by
\begin{align*}
\norm{\thetahat - \thetastar}_{\distr(\TransOp)}^2 \defn
\Exp_{\distr(\TransOp)} \big[ \big( \thetahat(\State) -
  \thetastar(\State) \big)^2 \big].
\end{align*}
Finally, the covariance operator $\CovOp(\TransOp)$ is induced by
$\distr(\TransOp)$, i.e. \mbox{$\CovOp(\TransOp) =
  \Exp_{X \sim \distr(\TransOp)} [ \Rep{\State} \otimes \Rep{\State} ]$} with
$\Rep{\State}$ denoting the representer of evaluation. Below, we take
$\big\{ \big( \eig_j(\TransOp), \base_j(\TransOp) \big)
\big\}_{j=1}^{\infty}$ as the eigenpairs associated with
$\CovOp(\TransOp)$.

\begin{subequations} \label{eq:def_MRPclass}
As alluded to above, the lower bounds require precise definitions of
the MRP families over which they hold.  We define two collections of
problem instances $\MRPclassA(\newradbar, \stdbar)$ and
$\MRPclassB(\newradbar, \stdbar)$ that are considered in Regimes A and
B respectively.  In Regime A, we suppose the reward function
$\rewardA$ is uniformly bounded, i.e. $\supnorm{\rewardA} \leq 1$, and
define a $(\newradbar, \stdbar)$-valid MRP family
\begin{multline}
\label{eq:def_MRPclassA}
\MRPclassA \equiv \MRPclassA ( \newradbar, \stdbar ) \equiv \MRPclassA
\big(\newradbar, \stdbar, \{ \eig_j \}_{j=1}^{\infty}; \rewardA,
\discount, \RKHSA \big) \\
\defn \big\{ \text{MRP $\MRP(\TransOp, \rewardA, \discount)$} \mid
\text{(i) The value function $\thetastar \in \RKHSA$ and
  inequalities~\eqref{EqnCondB} hold.}  \\
\text{(ii) The eigenpairs satisfy $\eig_j(\TransOp) = \eig_j$ for $j =
  1,2,\ldots$ and $\sup\nolimits_{j \in \Int_+}\supnorm{
    \base_j(\TransOp) } \leq 2$.}  \big \}.
\end{multline}
In parallel, the $(\newradbar, \stdbar)$-valid MRP family $\MRPclassB$
is given by
\begin{multline}
\label{eq:def_MRPclassB}
\MRPclassB \equiv \MRPclassB ( \newradbar, \stdbar ) \equiv
\MRPclassB \big(\newradbar, \stdbar, \{ \eig_j
\}_{j=1}^{\infty}; \rewardB, \discount, \RKHSB \big) \\ \defn
\big\{ \text{MRP $\MRP(\TransOp, \rewardB, \discount)$} \mid
\text{(i) $\discount \norm{\thetastar}_{\distr(\TransOp)} \leq
  1$. (ii) $\thetastar \in \RKHSB$ and
  inequalities~\eqref{EqnCondB} hold.} \\ \text{(iii) The
  eigenpairs satisfy $\eig_j(\TransOp) \leq \eig_j$ for any $j
  \geq 2$ and $\sup\nolimits_{j \in \Int_+} \supnorm{
    \base_j(\TransOp) } \leq 2$.} \big\}.
\end{multline}
\end{subequations}        
The major differences between MRP families $\MRPclassA$ and
$\MRPclassB$ are the regularity conditions and the eigenvalue
constraints.  In Regime A, the reward function $\rewardA$ is properly
normalized so that \mbox{$\supnorm{\rewardA} \leq 1$}, whereas we
impose an upper bound on the value function norm $\discount
\norm{\thetastar}_{\distr(\TransOp)}$ in \mbox{Regime B}.  Besides, in Regime A,
the pre-specified parameters $\{ \eig_j \}_{j=1}^{\infty}$ in
definition~\eqref{eq:def_MRPclassA} are exactly the eigenvalues for
instances in family $\MRPclassA$. In contrast, in Regime B, $\{ \eig_j \}_{j=1}^{\infty}$ are approximations
of eigenvalues in the definition~\eqref{eq:def_MRPclassB} of family
$\MRPclassB$.
	
As a point of clarification, we recall that the critical inequality
\eqref{eq:CI_lb} is defined by pre-specified constants $\{ \eig_j
\}_{j=1}^{\infty}$, not the eigenvalues $\{ \eig_j(\TransOp)
\}_{j=1}^{\infty}$. In other words, the critical radius $\delcrit$ in
the minimax lower bound \ref{EqnLB} fully depends on the pre-specified
parameters of the MRP families $\MRPclassA$ and $\MRPclassB$.
Moreover, in Regime B, the lower bound \ref{EqnLB} further implies
\mbox{$\norm{\thetahat - \thetastar}_{\distr(\TransOp)}^2 \geq
  \plaincon_1 \, \newradbar^2 \delcritsq(\TransOp)$}, where $\delcrit(\TransOp)$ is the critical radius induced by eigenvalues $\{
\eig_j(\TransOp) \}_{j=1}^{\infty}$. See \Cref{AppLower} for a proof of this claim.

%%%%%%%%%%%%%%%%%%%%%%%%%%%%%%%%%%%%%%%%%%%%%%%%%%%%%%%%%%%%%%%%%%%%%%%%%%%%%%

\subsubsection{High-level overview}
\label{sec:lb_Fano}

We provide a high-level overview of the proof structure.  The main
argument is based on Fano's method.  As in the standard use of Fano's
method for proving minimax bounds~\cite{wainwright2019high}, a key
step is the construction of an ensemble of value estimation problems
that are ``well-separated''.  In particular, we construct a collection
$\{\MRP_m \}_{m=1}^\PackNum$ of MRP instances, all of which share the
same state space $\StateSp = [0,1)$ and reward function $\reward$.
  Let $\distr$ denote the Lebesgue measure over $\StateSp$, and let
  $\TransOp_m$, $\thetastar_m$ and $\distrm{m}\equiv
  \distr(\TransOp_m)$ denote (respectively) the transition kernel,
  value function and stationary distribution associated with $\MRP_m$.
  Let $\TransOp_m^{1:\numobs}$ be the distribution of data $\{
  (\state_i, \statetwo_i) \}_{i=1}^\numobs$ when the ground-truth
  model is $\MRP_m$.

Suppose that an index $J$ is uniformly distributed over $[\PackNum]$
and observations $\{ (\state_i, \statetwo_i) \}_{i=1}^\numobs$ are
generated i.i.d. from $\MRP_J$.  Given this set-up, an application of
Fano's method (cf. \S 15.3.2 in the book~\cite{wainwright2019high} for
details) yields the lower bound
\begin{align*}
\inf_{\thetahat} \max_{\mdagger \in [\PackNum]} \Prob_{\mdagger}\Bigg[
  \munorm[\big]{\thetahat - \thetastar_\mdagger} \geq \frac{1}{2}
  \min_{m \neq \mprime} \munorm[\big]{\thetastar_{m} -
    \thetastar_{\mprime}} \Bigg] \geq 1 - \frac{\log 2 + \max
  \nolimits_{m, \mprime \in [\PackNum]}
  \kull[\big]{\TransOp_{m}^{1:\numobs}}{
    \TransOp_{\mprime}^{1:\numobs}}}{\log \PackNum} \, .
\end{align*}
Moreover, by further assuming that
\begin{align} \label{eq:density}
\frac{\diff \distrm{m}}{\diff \distr}(\state) \geq \frac{1}{2} \qquad \text{for all $\state \in \StateSp$ and $m \in [\PackNum]$},
\end{align} 
we connect the $\Lmu$ error with the $L^2( \distrm{m})$
error via inequality $\mumnorm{m}{\thetahat -
  \thetastar_m} \geq \frac{1}{\sqrt{2}}\munorm{\thetahat -
  \thetastar}$. It then follows that
\begin{align}
\label{eq:Fano}
\inf_{\thetahat} \max_{\mdagger \in [\PackNum]} \Prob_{\mdagger}\Bigg[
  \mumnorm[\big]{\mdagger}{\thetahat - \thetastar_\mdagger} \geq
  \frac{1}{2\sqrt{2}} \min_{m \neq \mprime}
  \munorm[\big]{\thetastar_{m} - \thetastar_{\mprime}} \Bigg] \geq 1 -
\frac{\log 2 + \!\!\!  \max \limits_{m, \mprime \in [\PackNum]}
  \kull[\big]{\TransOp_{m}^{1:\numobs}}{
    \TransOp_{\mprime}^{1:\numobs}}}{\log \PackNum} \, .
\end{align}

Exploiting the Fano inequality so as to obtain a ``good'' lower bound
involves constructing a suitable family of models.  Recalling the
statistical dimension $\statdim$.  In our proof of either Regime A or
B, we establish the existence of a family $\{ \MRP_m
\}_{m=1}^\PackNum$ with log cardinality $\log \PackNum \geq
\frac{\statdim}{10}$, and such that
\begin{subequations}
	\label{EqnInformal}  
	\begin{align}
	\label{EqnInformalKL}  
	\max_{m, \mprime \in [\PackNum]}
        \kull[\big]{\TransOp_{m}^{1:\numobs}}{\TransOp_{\mprime}^{1:\numobs}}
        & \leq \frac{\statdim}{40} \qquad \text{and} \\
	\label{EqnInformalL2}    
	\min_{m \neq \mprime} \munorm[\big]{\thetastar_m -
		\thetastar_{\mprime}} & \geq c_1' \sqrt{\plaincon} \; \newrad \delcrit
	\end{align}
\end{subequations}
where $c_1'$ is a universal constant and $\plaincon$ is given in
condition~\eqref{eq:kernel_reg}.  See Lemmas~\ref{lemma:lb_KL}
and~\ref{lemma:lb_gap} at the end of \Cref{sec:familyofMRP} for the
precise statement of these claims.  \\

Given these claims, we can combine the pieces to prove \Cref{thm:lb}.
Given the condition~\eqref{EqnInformalKL} and the bound $\log \PackNum
\geq \frac{\statdim}{10}$, we have $\frac{1}{\log \PackNum} \max_{m,
  \mprime \in [\PackNum]}
\kull[\big]{\TransOp_{m}^{1:\numobs}}{\TransOp_{\mprime}^{1:\numobs}}
\leq \frac{1}{4}$. Additionally, given that $\statdim \geq 10$, it
holds that $\log \PackNum \geq \frac{\statdim}{10} \geq 1$ and
therefore $\frac{\log 2}{\log \PackNum} \leq \log 2$. Combining these
inequalities, we find that the right hand side of
inequality~\eqref{eq:Fano} is larger than a positive constant $\big\{
1 - \frac{1}{4} - \log 2 \big\}$. We then substitute the minimum value
function distance $\min_{m \neq \mprime} \munorm[\big]{\thetastar_m -
  \thetastar_{\mprime}}$ in the left hand side of
inequality~\eqref{eq:Fano} by its lower bound in the
inequality~\eqref{EqnInformalL2}. This completes the high-level
overview of the proof of \Cref{thm:lb}. \\

With this perspective in place, the remaining steps---and the
technically challenging portion of the argument---should be clear.  In
particular, the remainder of our argument involves:
\begin{itemize}
\item constructing two reproducing kernel Hilbert spaces, denoted by
  $\RKHSA$ and $\RKHSB$, along with two subsets
  $\{\MRP_m\}_{m=1}^{\PackNum}$ belonging to either $\MRPclassA$ or
  $\MRPclassB$.
\item verifying that both groups of the MRP instances satisfy the
  claims~\eqref{eq:density},~\eqref{EqnInformalKL}
  and~\eqref{EqnInformalL2}.
\end{itemize}

%%%%%%%%%%%%%%%%%%%%%%%%%%%%%%%%%%%%%%%%%%%%%%%%%%%%%%%%%%%%%%%%%%%%%%%%

\subsubsection{Construction of simple two-state MRPs}
\label{sec:2stateMRP}

\begin{figure}[t]
	\renewcommand*\thesubfigure{\alph{subfigure}}
	\begin{center}
		\begin{subfigure}{0.5 \linewidth}
			\centering
                        \widgraph{\linewidth}{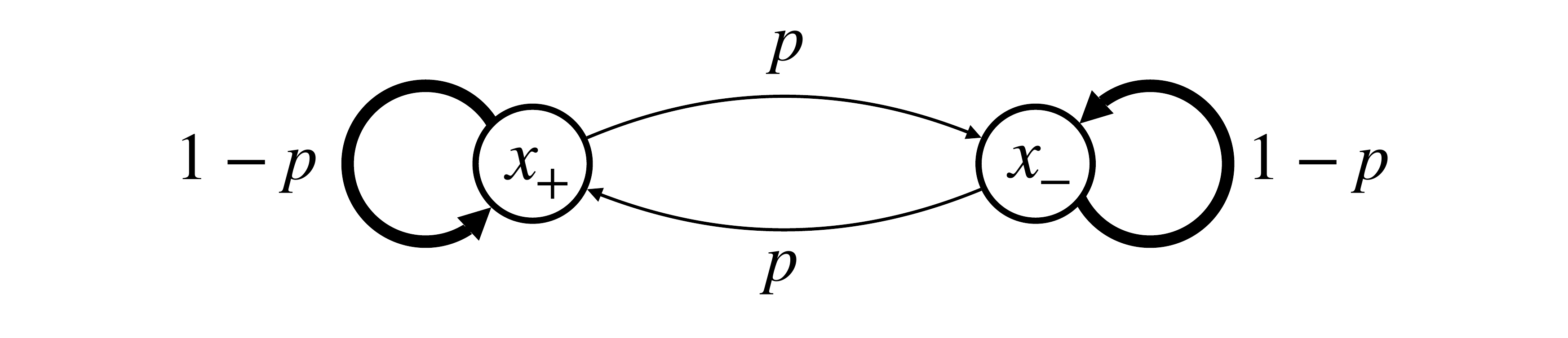} \caption{Base
                          Markov chain
                          $\bPbase(\parap)$.} \label{fig:2stateMRP_0}
		\end{subfigure} \\ \vspace{.5em}
		\begin{subfigure}{0.45 \linewidth}
			\widgraph{\linewidth}{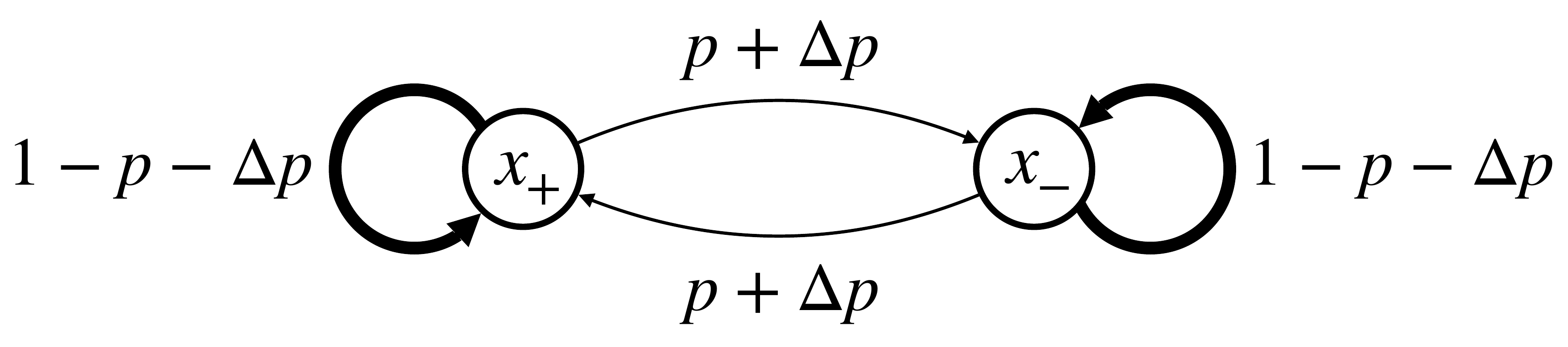} \caption{Construction
                          of $\bPA(\parap, \Deltap)$.}
		\end{subfigure}
		\hspace{2em}
		\begin{subfigure}{0.45 \linewidth}
			\widgraph{\linewidth}{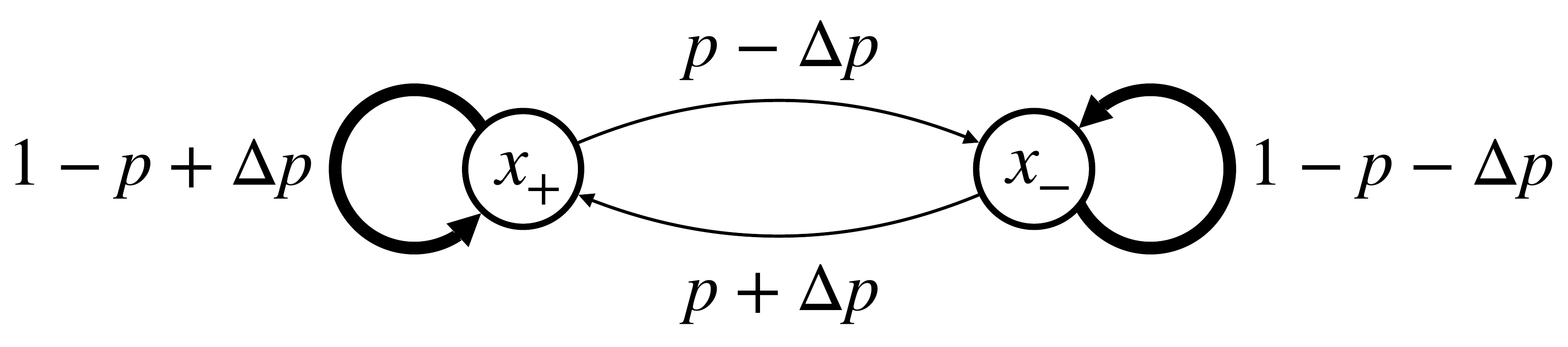} \caption{Construction
                          of $\bPB(\parap, \Deltap)$.}
		\end{subfigure}
		\caption{Two-state MRP instances $\bPbase$, $\bPr$ and
                  $\bPt$. The MRP instances are parameterized by
                  scalars $\parap \in \big[ 0, \tfrac{1}{2} \big]$ and
                  $\Deltap \in [-\parap,\parap]$. The parameters are
                  chosen so as to ensure that all edges are labeled
                  with valid probabilities.}
		\label{fig:2stateMRP}
	\end{center}
\end{figure}

As a warm-up, our first step is to construct a very simple two-state
Markov chain and two perturbed variants of it. Each variant is the
basic building block that underlies our full-scale ``hard'' instances
in $\MRPclassA$ or $\MRPclassB$.  Denote the states by $\xplus$ and
$\xneg$.  Given a scalar $\parap \in \big[0, \tfrac{1}{2}\big]$, the
base Markov chain is defined by the $2 \times 2$ transition matrix
\begin{subequations}
\begin{align}
\label{eq:def_2stateMRP_0}
	\bPbase \equiv \bPbase(\parap) \defn \begin{pmatrix} 1 - \parap &
          \parap \\ \parap & 1 - \parap \end{pmatrix}.
\end{align}
See \Cref{fig:2stateMRP_0} for an illustration of the transition
dynamics of the base model.

We further define the perturbed variants $\bPr$ and $\bPt$ of model
$\bPbase$.  In addition to the parameter $\parap$, we introduce
another scalar $\Deltap \in [-\parap, \parap]$. The Markov chains are
constructed as follows:
\begin{align}
  \label{eq:def_2stateMRP}
  \bPr \equiv \bPr(\parap, \Deltap) \defn \begin{pmatrix} 1 \!-\!  \parap
    \!-\! \Deltap & \parap + \Deltap \\ \parap + \Deltap & 1 \!-\!
    \parap \!-\! \Deltap \end{pmatrix}, ~~ \bPt \equiv \bPt(\parap,
  \Deltap) \defn \begin{pmatrix} 1 \!-\!  \parap \!+\! \Deltap &
    \parap - \Deltap \\ \parap + \Deltap & 1 \!-\! \parap \!-\!
    \Deltap \end{pmatrix}.
\end{align}
\end{subequations}
Panels (b) and (c) in Figure~\ref{fig:2stateMRP} represent these two
processes respectively.

Consider a reward function $\br$ given by $\br(\xplus) \defn \reward$
and $\br(\xneg) \defn -\reward$ where $\reward \in \Real$ is a
scalar. Let $\bthetabase$, $\bthetar$ and $\bthetat$ be the value
functions associated with transition kernels $\bPbase$, $\bPr$ and
$\bPt$. Then $\bthetar$ and $\bthetat$ can be viewed as perturbations
of $\bthetabase$ in two different directions. Specifically, we perform
calculations and find that $\bthetabase =
(1-\discount+2\discount\parap)^{-1} \br$ and the differences $\Delta
\bthetar = \bthetar - \bthetabase$ and $\Delta\bthetat = \bthetat -
\bthetabase$ satisfy the relations
\begin{align}
  \label{eq:diff_theta}
\Delta\bthetar (\xplus) = - \Delta\bthetar (\xneg) \qquad \text{and}
\qquad \Delta \bthetat (\xplus) = \Delta \bthetat(\xneg).
\end{align}	
In the sequel, we construct full-scale MRP instances using the Markov
chains $\bPr$ and $\bPt$.

%%%%%%%%%%%%%%%%%%%%%%%%%%%%%%%%%%%%%%%%%%%%%%%%%%%%%%%%%%%%%%%%%%%%%%%%%%%%%%%%%%%%%

\subsubsection{Construction of MRPs over state space $\StateSp = [0,1)$}
\label{sec:baseMRP}

In this part, we assemble $\numint$ different two-state Markov chains
$\{ \bP^{(k)} \}_{k=1}^\numint$ into a full-scale model $\TransOp$
over state space $\StateSp = [0,1)$.  In our constructions, matrix
  $\bP^{(k)}$ takes the form of $\bPr\big(\parap, \Deltap^{(k)}\big)$
  in Regime A and $\bPt\big(\parap, \Deltap^{(k)}\big)$ in Regime B,
  where the parameters $\numint$, $\parap$ and
  $\{\Deltap^{(k)}\}_{k=1}^{\numint}$ will be specified later.

We evenly partition the state space $\StateSp = [0,1)$ into $2
  \numint$ intervals
\begin{align}
\label{eq:def_partition}
\interval_+^{(k)} \defn \big[\tfrac{k-1}{2\numint},
  \tfrac{k}{2\numint} \big) \quad \text{and} \quad \interval_-^{(k)}
  \defn \big[\tfrac{1}{2} + \tfrac{k-1}{2\numint}, \tfrac{1}{2} +
    \tfrac{k}{2\numint} \big) \qquad \text{for $k =
      1,2,\ldots,\numint$.}
\end{align}
For each index $k \in [\numint]$, the dynamics of $\TransOp$ on
intervals $\interval_+^{(k)}$ and $\interval_-^{(k)}$ follow the local
model $\bP^{(k)}$.  With a slight abuse of notation, we denote the two
states of Markov chain $\bP^{(k)}$ by $\xplus$ and $\xneg$ for any $k
\in [\numint]$.  The transition kernel $\TransOp$ is then defined as
\footnote{A technical side-comment: note that the Markov
chain~\eqref{eq:def_p} is not ergodic.  However, this issue can be
remedied with a slight modification of the transition kernel
$\TransOp$. Let $\widetilde{\distr}$ be a stationary distribution of
$\TransOp$.  We fix a number $\epsilon \in (0,1)$.  At each time step,
let the Markov chain follow $\TransOp$ with probability $1-\epsilon$,
and transit to a next state according to $\widetilde{\distr}$ with
probability $\epsilon$.  This procedure defines a new transition
kernel $\widetilde{\TransOp}(\cdot \mid \state) \defn \epsilon
\widetilde{\distr}(\cdot) + (1-\epsilon) \TransOp(\cdot \mid \state)$,
which induces a new Markov chain that is ergodic, and has a unique
stationary distribution.  Since $\epsilon > 0$ can be chosen
arbitrarily close to zero, we can recover statements about the
original model in this way.  The $\epsilon$-modification would induce
unnecessary clutter, so that we focus on model~\eqref{eq:def_p} in the
following discussion.}
\begin{align}
\label{eq:def_p}
\TransOp(\statetwo \mid \state) \defn \begin{cases} 2 \numint \; {\bf
    P}^{(k)}(\xplus \mid \xplus) \quad & \text{if $\state, \statetwo
    \in \interval_+^{(k)}$}, \\ 2 \numint \; {\bf P}^{(k)}(\xneg \mid
  \xplus) \quad & \text{if $\state \in \interval_+^{(k)}$, $\statetwo
    \in \interval_-^{(k)}$}, \\ 2 \numint \; {\bf P}^{(k)}(\xneg \mid
  \xneg) \quad & \text{if $\state, \statetwo \in \interval_-^{(k)}$},
  \\ 2 \numint \; {\bf P}^{(k)}(\xplus \mid \xneg) \quad & \text{if
    $\state \in \interval_-^{(k)}$, $\statetwo \in
    \interval_+^{(k)}$}, \\ 0 \qquad & \text{otherwise}.  \end{cases}
\end{align}
\Cref{fig:MRPembedding} illustrates our construction of model $\TransOp$.

In the full-scale MRP $\MRP(\TransOp, \reward, \discount)$, we take a
reward function
\begin{align}
  \label{eq:def_reward}
  \reward(\state) \defn \reward \, \big( \mathds{1}\big\{ \state \in
         [0, \tfrac{1}{2}) \big \} - \mathds{1}\big\{ \state \in
           [\tfrac{1}{2}, 1) \big \} \big)
\end{align}
with a scalar $\reward \in \Real$. By our construction, the transition kernel
$\TransOp$ and reward function $\reward$ produce a value function
$\thetastar$ that is piecewise constant over intervals
$\interval_+^{(k)}$ and $\interval_-^{(k)}$. Moreover, we have
\begin{align} \label{eq:def_theta}
  \thetastar(\state) = \begin{cases} \btheta^{(k)}(\xplus) \quad
    & \text{if $\state \in \interval_+^{(k)}$},
    \\ \btheta^{(k)}(\xneg) \quad & \text{if $\state \in
      \interval_-^{(k)}$},
  \end{cases}
\end{align}
where $\btheta^{(k)} \defn (\bI - \discount \bP^{(k)})^{-1} \br \in
\Real^2$ is the value vector given by transition matrix $\bP^{(k)}$
and reward vector $\br = [\reward, -\reward]^{\top}$.

We consider the form of the full-scale value function $\thetastar$
when taking $\bP^{(k)} = \bPbase(\parap), \bPr\big(\parap,
\Deltap^{(k)}\big)$ or $\bPt\big(\parap, \Deltap^{(k)}\big)$. If we
set $\bP^{(k)} = \bPbase(\parap)$, then the value function is given by
\begin{align}
\label{eq:def_thetabase}
\thetastar(\state) = \thetastar_0(\state) \defn (1 - \discount +
2\discount \parap)^{-1} \, \reward(\state).
\end{align}
We refer to $\thetastar_0$ as the base value function.  If $\bP^{(k)}
= \bPr\big(\parap, \Deltap^{(k)}\big)$, then due to
equation~\eqref{eq:diff_theta}, we have $\thetastar = \thetastar_0 +
\Delta\thetastar$ with function $\Delta\thetastar$ satisfying
$\Delta\thetastar(\state) = - \Delta\thetastar(\state+\tfrac{1}{2})$
for any $\state \in [0, \tfrac{1}{2})$. When $\bP^{(k)} =
  \bPt\big(\parap, \Deltap^{(k)}\big)$, the value function
  $\thetastar$ admits a decomposition $\thetastar = \thetastar_0 +
  \Delta\thetastar$ with $\Delta\thetastar(\state) =
  \Delta\thetastar(\state+\tfrac{1}{2})$ for any $\state \in [0,
    \tfrac{1}{2})$.  In the following, we construct function spaces
    $\RKHSA$ and $\RKHSB$ of which the elements possess these
    properties.

\begin{figure}[h]
	\renewcommand*\thesubfigure{\alph{subfigure}}
	\begin{center}
		\begin{subfigure}{0.8\linewidth}
			\centering
                        \widgraph{0.8\linewidth}{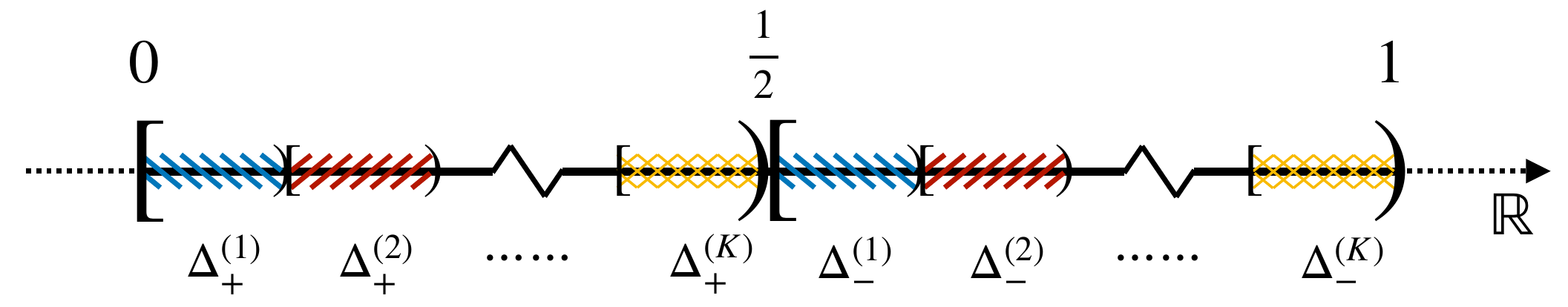} \caption{Partition
                          of state space $\StateSp = [0,1)$.}
		\end{subfigure} \\ \vspace{.5em}
		\begin{subfigure}{0.8\linewidth}
			\centering
                        \widgraph{0.7\linewidth}{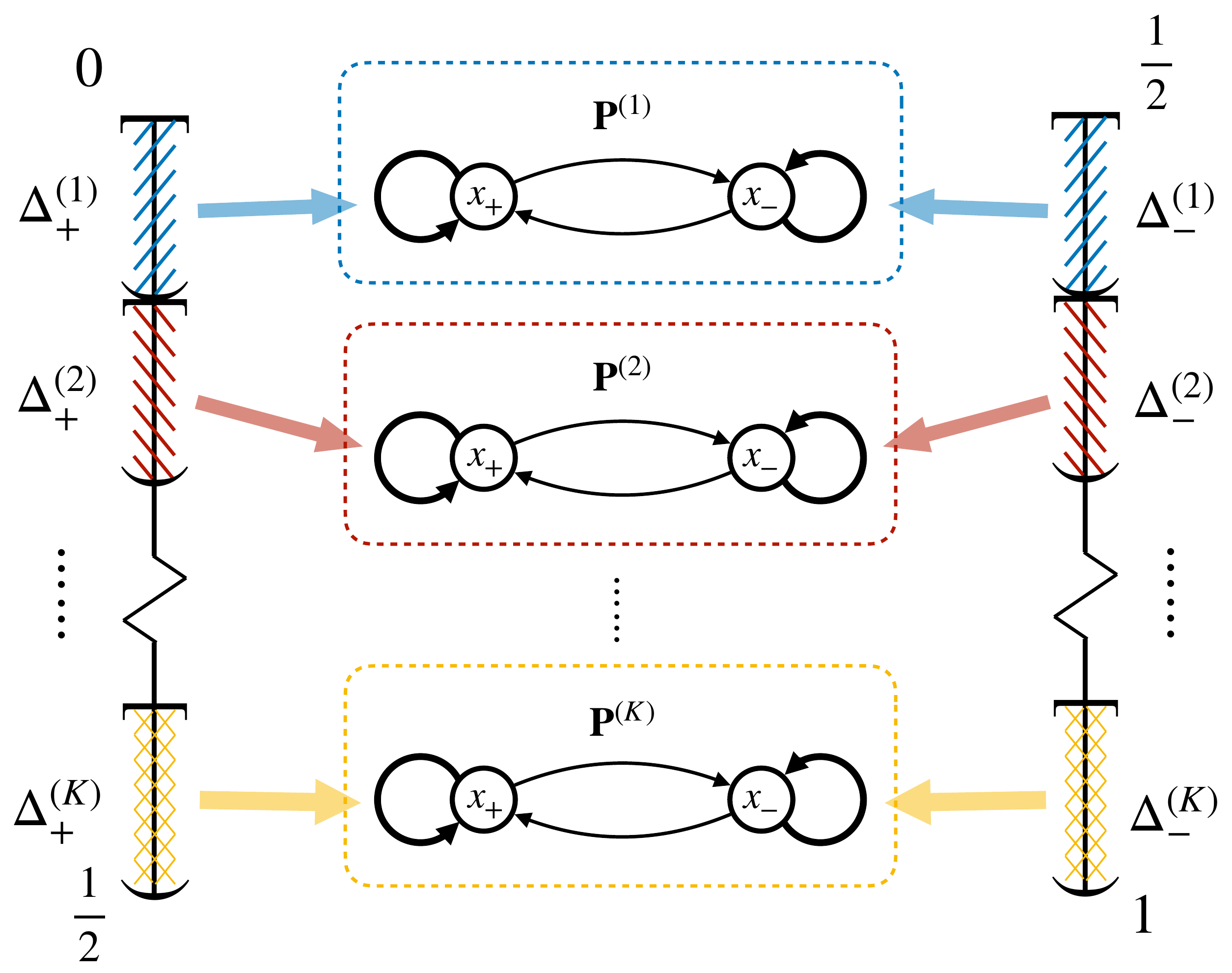}
                        \caption{Construction of transition kernel
                          $\TransOp$ over state space $\StateSp =
                          [0,1)$ using 2-state Markov chains
                            $\bP^{(1)}, \bP^{(2)}, \ldots,
                            \bP^{(K)}$.}
		\end{subfigure}
		\caption{Embedding of two-state Markov chains $\{
                  \bP^{(k)} \}_{k=1}^\numint$ into state space
                  $\StateSp = [0,1)$. Up: partition of state space
                    $\StateSp$ into intervals $\{ \interval_+^{(k)},
                    \interval_-^{(k)} \}_{k=1}^\numint$. Bottom: the
                    transitions on intervals $\interval_+^{(k)}$ and
                    $\interval_-^{(k)}$ follow a local Markov chain
                    $\bP^{(k)}$.}
		\label{fig:MRPembedding}
	\end{center}
	\vspace{2em}
\end{figure}

%%%%%%%%%%%%%%%%%%%%%%%%%%%%%%%%%%%%%%%%%%%%%%%%%%%%%%%%%%%%%%%%%%%%%%%%%%%%%%%%%%%%%%%%%%%%%%

\subsubsection{Constructing the Hilbert spaces
  $\RKHSA$ and $\RKHSB$}
\label{sec:RKHS}

\begin{figure}[!ht]
  \begin{center}
    \widgraph{.48\linewidth}{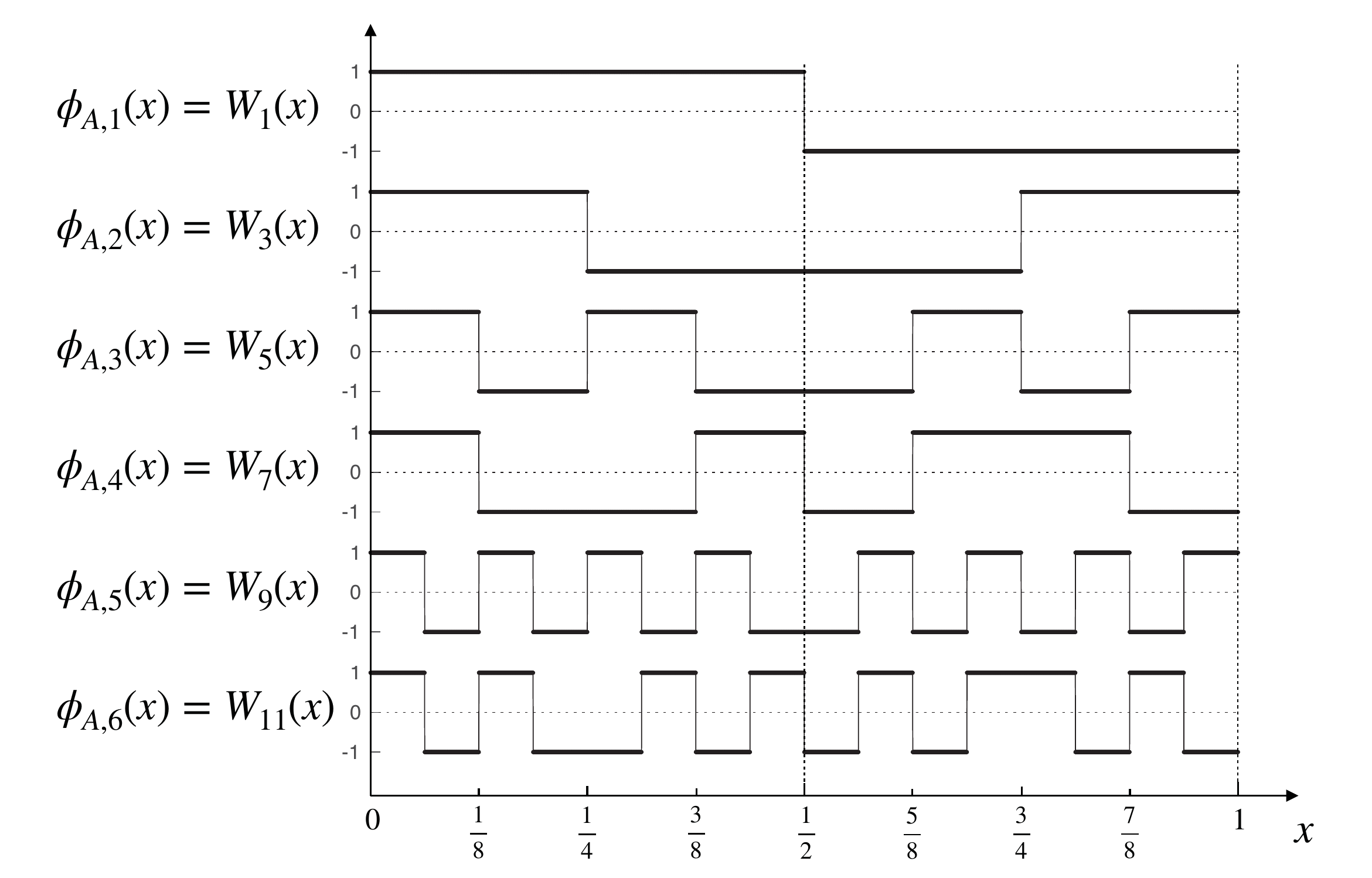}
    \widgraph{.48\linewidth}{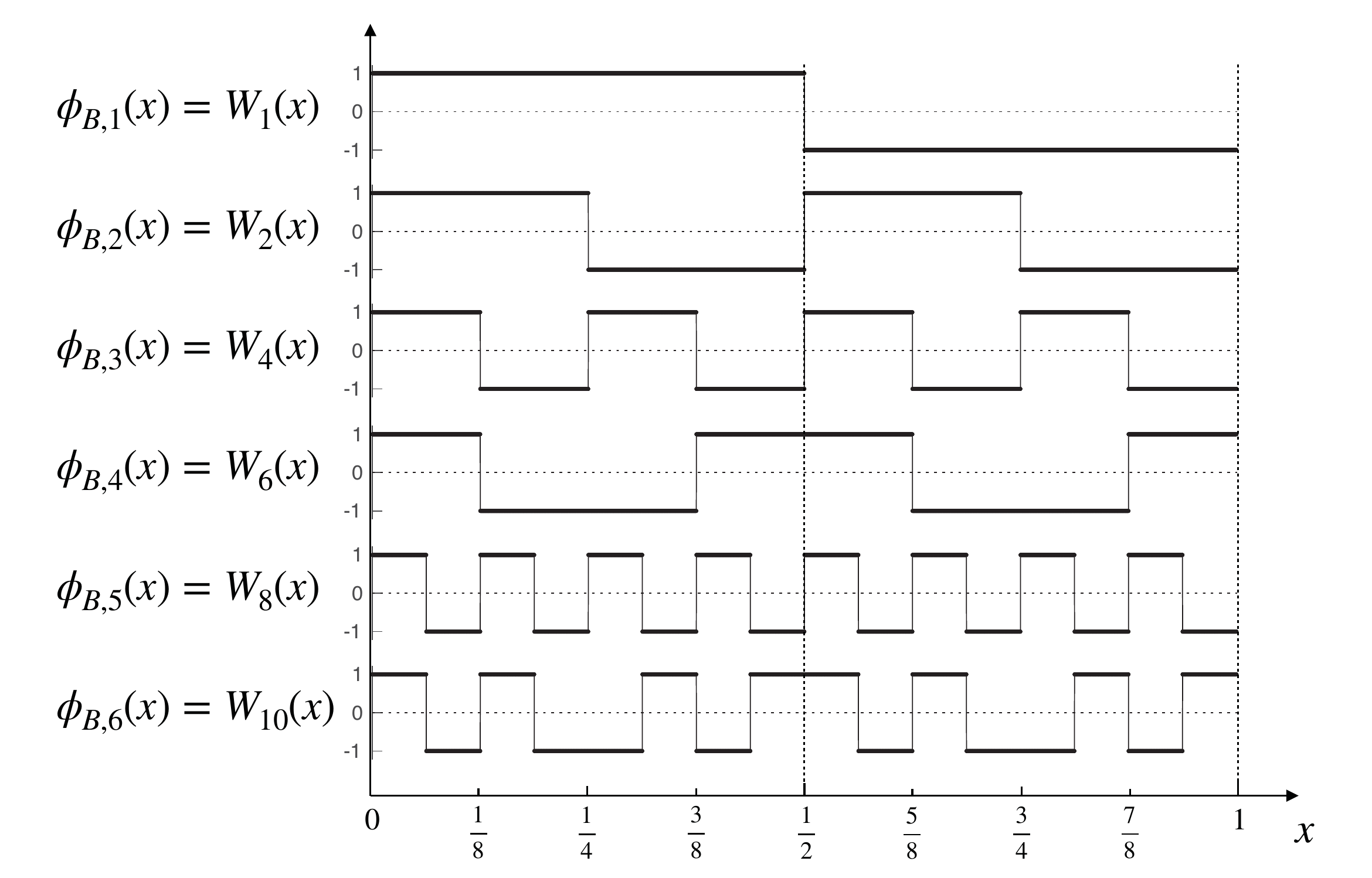}
    \caption{Construction of functions $\{ \baseA{j}
      \}_{j=1}^{\infty}$ and $\{ \baseB{j} \}_{j=1}^{\infty}$.  Left:
      Basis functions $\baseA{1}, \baseA{2}, \ldots,
      \baseA{6}$. Right: Basis functions $\baseB{1}, \baseB{2},
      \ldots, \baseB{6}$. }
    \label{fig:Walsh}
  \end{center}
	\vspace{1em}
\end{figure}

Given a sequence $\{\eig_j \}_{j=1}^\infty$ of non-negative numbers,
we construct two RKHSs $\RKHSA$ and $\RKHSB$ of functions with domain
$\StateSp = [0,1)$, such that both the associated kernels have
  eigenvalues $\{\eig_j\}_{j=1}^\infty$ under the Lebesgue measure
  $\distr$.  Our construction is designed to produce kernels that are
  especially amenable to analysis, and can easily connect to the MRPs
  defined in equation~\eqref{eq:def_p}.  In particular, we leverage
  the Walsh system, an orthonormal basis of $\Lmu$ that can represent
  discrete functions conveniently.  For any $j \in \Natural$, the
  $j$-th Walsh function is given by
\begin{align*}
W_j(\state) \defn (-1)^{\sum_{i=0}^{\infty} k_i x_{i+1}} \quad
\text{for $j = \sum_{i=0}^{\infty} k_i 2^i, \ \state = x_0 +
  \sum_{i=1}^{\infty} \state_i 2^{-i}$ with $k_i, \state_i \in \{ 0, 1
  \}$ and $x_0 \in \Int$.}
\end{align*}
Specifically, the first Walsh function takes the form $W_1(\state) =
\mathds{1}\big\{ \state \in [0, \tfrac{1}{2}) \big\} -
  \mathds{1}\big\{ \state \in [\tfrac{1}{2}, 1) \big\}$.

Below we construct two groups of functions $\{ \baseA{j}
\}_{j=1}^{\infty}$ and $\{ \baseB{j} \}_{j=1}^{\infty}$ that are bases
of $\RKHSA$ and $\RKHSB$ respectively:
\begin{subequations}
  \label{eq:def_basis}
  \begin{align}
    & \baseA{j}(\state) \; \defn \; W_{2j-1}(\state) = W_{j-1}(2x) \;
    W_1(\state) \qquad \text{for $j = 1,2,3,\ldots$} \qquad
    \text{and} \label{eq:phi_r} \\
    & \baseB{j}(\state) \; \defn \; \begin{cases} W_1(\state) \qquad &
      \text{if $j = 1$}, \\
      W_{2(j-1)}(\state) = W_{j-1}(2x) \qquad & \text{if $j =
        2,3,\ldots$}. \end{cases} \label{eq:phi_theta}
\end{align}
\end{subequations}
See \Cref{fig:Walsh} for an illustration of the top $6$ basis
functions in each group.  Based on $\{ \baseA{j} \}_{j=1}^{\infty}$
and $\{ \baseB{j} \}_{j=1}^{\infty}$, we define the kernel functions
$\KerA$ and $\KerB$ as
\begin{align} \label{eq:def_Ker} \Ker_{\iota}(\state, \statenew) \;\defn\; \sum_{j=1}^{\infty} \; \eig_j \; \base_{\iota,j}(\state) \, \base_{\iota,j}(\statenew) \qquad \text{for $\iota = A \text{ or } B$} \end{align}
and let $\RKHSA$ and $\RKHSB$ be the RKHSs induced by $\KerA$ and $\KerB$.

The function classes $\{ \baseA{j} \}_{j=1}^{\infty}$ and $\{
\baseB{j} \}_{j=1}^{\infty}$ above are both orthonormal in
$\Lmu$. Indeed, we have $\int_{\StateSp} \base_{\iota, i}(\state)
\base_{\iota, j}(\state) \distr(\dx) = \mathds{1}\{ i = j \}$ for
$\iota = A \text{ or } B$ and any $i,j \in \Int_+$. Hence, the kernels
$\KerA$ and $\KerB$ have eigenpairs $\{ (\eig_j, \baseA{j})
\}_{j=1}^{\infty}$ and $\{ (\eig_j, \baseB{j}) \}_{j=1}^{\infty}$
associated with the Lebesgue measure $\distr$. \\

Our choice of bases $\{ \baseA{j} \}_{j=1}^{\infty}$ and $\{ \baseB{j}
\}_{j=1}^{\infty}$ is especially tailored to the MRP construction in
\Cref{sec:baseMRP}.  Suppose $\numint$ is a power of $2$ and let $\{
\interval_+^{(k)}, \interval_-^{(k)} \}_{k=1}^\numint$ be a partition
of state space $\StateSp = [0,1)$ given in
  equation~\eqref{eq:def_partition}. Then for any function $f$ that is piecewise constant with respect to
  the partition and satisfies $f(\state) = - f(\state+\tfrac{1}{2})$
  for any $\state \in \big[ 0, \tfrac{1}{2} \big)$, it can always be linearly expressed by functions $\{ \baseA{1},
  \ldots, \baseA{\numint} \}$. Similarly, the
    function set $\{ \baseB{2}, \ldots, \baseB{K} \}$ is capable of
    representing any discrete function $f$ that is adapted to the
    partition and satisfies $f(\state) = f(\state+\tfrac{1}{2})$ for
    any $\state \in \big[ 0, \tfrac{1}{2} \big)$.

%%%%%%%%%%%%%%%%%%%%%%%%%%%%%%%%%%%%%%%%%%%%%%%%%%%%%%%%%%%%%%%%%%%%%%%%%%%%%%%%%%%
	
\subsubsection{Two families of MRPs} \label{sec:familyofMRP}

We now construct a family $\{\MRP_m\}_{m=1}^\PackNum$ of MRP instances
using the transition kernel and reward function defined in
equations~\eqref{eq:def_p}~and~\eqref{eq:def_reward}, with value functions belonging to either $\RKHSA$ or $\RKHSB$. Recall our
definition $\statdim = \max\big\{ j \mid \eig_j \geq \delcritsq
\big\}$ of the effective dimension (at sample size $\numobs$) of the
underlying kernel class. Consider the Boolean hypercube
$\{0,1\}^{\statdim-1}$, and let $\{ \alphabold_m \}_{m=1}^\PackNum$ be
a $\tfrac{1}{4}$-(maximal) packing of it with respect to the
(rescaled) Hamming metric
\begin{align}
  \label{eq:def_Hamming}
\rho_H( \alphabold, \alphabold') \defn \frac{1}{\statdim - 1}
\sum_{k=1}^{\statdim - 1} \mathds{1}\{\alpha_k \neq \alpha_k' \}.
\end{align}
It is known from standard results on metric entropy (e.g., see Example
5.3 in the book~\cite{wainwright2019high}) that there exists such a
set with log cardinality lower bounded as $\log \PackNum \geq
\frac{\statdim}{10}$.  Using this packing of the Boolean hypercube, we
now show how to construct the MRP instance $\MRP_m$ based on the
binary vector $\alphabold_m$. \\

In either Regime A or B, the MRP instances
$\{\MRP_m\}_{m=1}^{\PackNum}$ share the same reward function
$\reward(\state) = \rewardA(\state)$ or $\rewardB(\state)$. Each model
$\MRP_m$ has a transition kernel $\TransOp_m$ that lies within a
neighborhood of a base Markov chain $\TransOp_0$. The difference between
$\TransOp_m$ and $\TransOp_0$ is encoded by vector $\alphabold_m$.
Specifically, we pick a transition kernel $\TransOp_m$ such that the
difference in value functions $\thetastar_m - \thetastar_0$ is a
linear combination of functions $\{ \baseA{j} \}_{j=2}^{\statdim}$ or $\{
\baseB{j} \}_{j=2}^{\statdim}$, with vector
$\alphabold_m$ determining the linear coefficients. Here, $\thetastar_0$ is the base value
function given by equation~\eqref{eq:def_thetabase}.

In our constructions below, we take \mbox{$\numint \defn 2^{\lceil
    \log_2 \statdim \rceil}$}. It is ensured that the functions $\{
\baseA{j} \}_{j=1}^{\statdim}$ and $\{ \baseB{j} \}_{j=1}^{\statdim}$
are piecewise constant with respect to the partition $\big\{
\interval_+^{(k)}, \interval_-^{(k)} \big\}_{k=1}^\numint$.  Recall from definition~\eqref{eq:def_p}
that transition kernel $\TransOp_m$ is determined by local models $\{
\bP_m^{(k)} \}_{k=1}^\numint$. In the sequel, we specify the choices
of $\{ \bP_m^{(k)} \}_{k=1}^\numint$ so that the value function
$\thetastar_m$ has the desired form.

%%%%%%%%%%%%%%%%%%%%%%%%%%%%%%%%%%%%%%%%%%%%%%%%%%%%%%%%%%%%%%%%%%%%%%%%%%%%%%%

\paragraph{Regime A:}

We first construct MRP instances $\{ \MRP_m \}_{m=1}^{\PackNum}$ that
belong to the model class $\MRPclassA$. In order that the regularity
condition $\supnorm{\rewardA} \leq 1$ holds, we simply set parameter
$\reward \defn 1$ in equation~\eqref{eq:def_reward} so that the reward
function $\rewardA(\state) = W_1(\state)$.

In our design of the transition kernel $\TransOp_m$, the local Markov
chains are set as \mbox{$\bP_m^{(k)} \defn \bPr\big( \parap,
  \Deltap_m^{(k)} \big)$} where $\bPr$ is given in
equation~\eqref{eq:def_2stateMRP} and the parameter $\parap$ is chosen as $\parap \defn
\tfrac{3\SqDis}{\discount}$.  We remark that the uniform distribution
$\distr$ is stationary under model $\TransOp_m$, so we pick
$\distrm{m} = \distr(\TransOp_m) = \distr$.  We take parameters $\{ \Deltap_m^{(k)} \}_{k=1}^K$ such that the value
function $\thetastar_m$ of MRP $\MRP_m$ satisfies
\begin{align}
  \label{eq:def_thetam}
  \thetastar_m = \thetastar_0 -
  \frac{2\discount}{(1-\discount+2\discount\parap)^2} \; f_m,
\end{align}
where $\thetastar_0 = (1-\discount+2\discount\parap)^{-1} W_1$ and
\begin{align}
  \label{eq:def_fm}
  f_m (\state) \defn \sqrt{\frac{\parap \, (1-\parap)}{120 \,
      \numobs}} \; \sum_{j=2}^{\statdim} \alpha_m^{(j-1)} \,
  \baseA{j}(\state).
\end{align}
In order to do so, we set
\begin{align} \label{eq:def_tau}
	\Deltap_m^{(k)} & \defn \frac{1 - \discount + 2
          \discount\parap}{1 - \discount + 2 \discount \parap - 2
          \discount f_m(\state_k)} \; f_m(\state_k) \qquad \text{for $m \in
          [\PackNum]$ and $k \in [K]$}
\end{align}
in the local Markov chain $\bP_m^{(k)} = \bPr\big( \parap,
\Deltap_m^{(k)} \big)$. Recall from equation~\eqref{eq:def_theta} that $\thetastar_m(\state) = \btheta_m^{(k)}(\xplus) = -
\btheta_m^{(k)}(\xneg)$ for any $\state \in
\interval_+^{(k)}$, where $\btheta_m^{(k)} \in \Real^2$ is the value vector induced by model
$\bP_m^{(k)}$ and reward vector $\br=[1,-1]^{\top}$. Under our choice of $\Deltap_m^{(k)}$ in equation~\eqref{eq:def_tau}, the value function $\thetastar_m$ has the desired form as in equation~\eqref{eq:def_thetam}.

%%%%%%%%%%%%%%%%%%%%%%%%%%%%%%%%%%%%%%%%%%%%%%%%%%%%%%%%%%%%%%%%%%%%%%%%%%%%%

\paragraph{Regime B:}

We now construct MRP instances $\{ \MRP_m \}_{m=1}^{\PackNum}$ in family $\MRPclassB$.  In this scenario, we take the parameter
$\parap \defn \frac{1}{8}$ in local models $\bP_m^{(k)} \defn
\bPt\big( \parap, \Deltap_m^{(k)} \big)$ and $\reward \defn \parap +
\tfrac{1-\discount}{2\discount}$ in the
definition~\eqref{eq:def_reward} of reward function so that $\rewardB(\state) = \big( \parap +
\tfrac{1-\discount}{2\discount} \big) \, W_1(\state)$.
Moreover, we set \mbox{$\Deltap_m^{(k)} \defn f_m(\state_k)$} in model $\bP_m^{(k)} =
\bPt\big( \parap, \Deltap_m^{(k)} \big)$,
where
\begin{align}
	\label{eq:def_fm_2}
	f_m (\state) \defn \frac{\parap}{25 \sqrt{\numobs}} \;
	\sum_{j=2}^{\statdim} \alpha_m^{(j-1)} \,
	\baseB{j}(\state)
\end{align}
and $\state_k$ is any point in
interval $\interval_k^+$ or $\interval_k^-$.
The value function $\thetastar_m$ then satisfies
\begin{align}
\label{eq:def_thetam_2}
\thetastar_m = \thetastar_0 + \frac{1}{1-\discount} \; f_m\, .
\end{align}
We observe that in this
case, the transition kernel $\TransOp_m$ has a stationary distribution
\begin{align}
\label{eq:def_mum}
\distrm{m}(\state) \defn
\begin{cases} 
1 + \frac{f_m(\state)}{\parap} & \qquad \text{if $\state \in \interval_k^+$},
\\ 1 - \frac{f_m(\state)}{\parap} & \qquad \text{if $\state \in \interval_k^-$}.
\end{cases}
\end{align}
The measure $\distrm{m}$ is not the uniform distribution $\distr$;
however, our construction ensures that \mbox{$\frac{\diff
    \distrm{m}}{\diff \distr}(\state) \geq \tfrac{1}{2}$.} See
\Cref{append:density_2} for the proof of this claim. \\

\noindent We claim that both of our constructions yield MRPs that
belong to the desired classes:
\begin{lemma}
 \label{lemma:welldefn}
The previously described constructions yield MRP instances $\MRP_m$
such that \mbox{$\{ \MRP_m \}_{m=1}^\PackNum \subset
  \MRPclassA$} in Regime A, and \mbox{$\{ \MRP_m
  \}_{m=1}^\PackNum \subset \MRPclassB$} in Regime B.
\end{lemma}
\noindent We prove the Regime A claim in \Cref{append:lb_welldefn},
and the Regime B claim in \Cref{append:lb_welldefn_2}.

\vspace{1em}

%%%%%%%%%%%%%%%%%%%%%%%%%%%%%%%%%%%%%%%%%%%%%%%%%%%%%%%%%%%%%%%%%%%

We now need to establish upper bounds on the pairwise KL divergences,
and lower bounds on the pairwise $\Lmu$-distances, as stated
informally in equations~\eqref{EqnInformal}.  The precise statements
are as follows:
\begin{lemma}
\label{lemma:lb_KL}
For either of the two classes ($\MRPclassA$ in Regime A, or
$\MRPclassB$ in Regime B), our construction ensures that
\begin{align}
  \kull[\big]{\TransOp_{m}^{1:\numobs}}{\TransOp_{\mprime}^{1:\numobs}}
  \leq \frac{\statdim}{40} \qquad \text{for any $m, \mprime \in
    [\PackNum]$}.
    \end{align}
\end{lemma}
\noindent See \Cref{append:lb_KL} and \Cref{append:lb_KL_2},
respectively, for the proofs corresponding to the classes
$\MRPclassA$ and $\MRPclassB$.
\begin{lemma}
\label{lemma:lb_gap}
Our construction ensures that there exists a universal constant
$\plaincon_1'$ such that
\begin{align}
\label{eq:gap>=}
\min_{m \neq \mprime} \munorm[\big]{\thetastar_m -
  \thetastar_{\mprime}} \geq \plaincon_1' \sqrt{\plaincon} \;
\newradbarreward \delcrit.
\end{align}
The claim holds for both $\{ \MRP_m \}_{m=1}^{\PackNum} \subset
\MRPclassA$ and $\{ \MRP_m \}_{m=1}^{\PackNum} \subset
\MRPclassB$.
\end{lemma}
\noindent This claim is proved in
Appendices~\ref{append:lb_gap}~and~\ref{append:lb_gap_2} for
$\MRPclassA$ and $\MRPclassB$ respectively.

%%%%%%%%%%%%%%%%%%%%%%%%%%%%%%%%%%%%%%%%%%%%%%%%%%%%%%%%%%%%%%%%%%%%%%%%%%%%%%%%%%%%%%%%%%%%%%%%

\section{Discussion}
\label{SecConclusion}

In this paper, we have analyzed the performance of a regularized
kernel-based least-squares temporal difference (LSTD) estimator for
policy evaluation.  Our main contribution was to prove non-asymptotic
upper bounds on the statistical estimation error, along with guidance
for the choices of the regularization parameter required to achieve
such bounds.  Notably, our upper bounds depend on the problem
structure via the sample size, the effective horizon, the eigenvalues
of the kernel operator, and the variance of the Bellman residual.  As
we show, the bounds show a wide range of behavior as these different
structural components are altered.  Moreover, we prove a matching
minimax lower bounds over distinct subclasses of problems that
demonstrate the sharpness of our upper bounds.

Our study leaves open a number of intriguing questions; let us mention
a few of them here to conclude.  First, although our bounds are
instance-dependent, this dependence is not as refined as recent
results in the simpler tabular and linear function
settings~\cite{khamaru2020temporal,mou2020optimal}.  In particular,
our current results do not explicitly track the mixing properties of
the transition kernel, which should enter in any such refined
analysis.  Second, the analysis of this paper was carried out under
i.i.d. assumptions on transition sampling model.  However, in
practice, the data may be collected from Markov chain trajectories or
adaptive experiments and the transition pairs are no longer
independent. It would be interesting to see how the dependence in data
affects sample complexity of policy evaluation.  Third, this paper
assumes that samples are drawn from the stationary distribution of the
Markov chain; in practice, such data may not be available, so that it
is interesting to consider extensions of this kernel LSTD estimator
suitable for the off-policy setting.  Last, the results in this paper
use the $\Lmu$-norm to quantify the error.  In applications of policy
evaluation, other error metrics may be of interest, including
pointwise errors ($\big| \thetahat(\state) - \thetastar(\state) \big|$
for a fixed state $\state \in \StateSp$), or sup-norm guarantees
($\supnorm{\thetahat - \thetastar}$).  These are interesting
directions for future study.

%%%%%%%%%%%%%%%%%%%%%%%%%%%%%%%%%%%%%%%%%%%%%%%%%%%%%%%%%%%%%%%%%%%%%%%%%%%%%%

\subsection*{Acknowledgements}
This work was partially supported by NSF-DMS grant 2015454, NSF-IIS
grant 1909365, NSF-FODSI grant 202350, and DOD-ONR Office of Naval
Research N00014-21-1-2842 to MJW.

%%%%%%%%%%%%%%%%%%%%%%%%%%%%%%%%%%%%%%%%%%%%%%%%%%%%%%%%%%%%%%%%%%%%%%%%%%%%%%%%%

\appendix

\section{Details of simulations}
\label{AppSimDetails}

In this appendix, we provide the details of the families of MRPs used
for the simulation results in Section~\ref{SecSimulations}.

%%%%%%%%%%%%%%%%%%%%%%%%%%%%%%%%%%%%%%%%%%%%%%%%%%%%%%%%%%%%%%%%%%%%%%%%%%%%%%%%%%

\subsection{Families of MRPs}

We constructed families of MRPs all with state space $\StateSp =
[0,1)$. In all cases, the reward function takes the form
  \begin{subequations}
  \label{eq:exp}
  \begin{align}
    \label{eq:exp_reward}
    \reward(\state) & \defn \mathds{1}\big\{ \state \in \big[0,\tfrac{1}{2}\big)
      \big\} - \mathds{1}\big\{ \state \in \big[ \tfrac{1}{2}, 1 \big) \big
        \},
  \end{align}
  whereas the transition operator is given by
\begin{align}  
    \label{eq:exp_TransOp}
          \TransOp(\statetwo \mid \state) & \defn
	\begin{cases}
	  2 (1-\parap), & \! \text{if $\state, \statetwo \in \big[0, \tfrac{1}{2}
              \big)$ or $\state, \statetwo \in \big[ \tfrac{1}{2}, 1 \big)$}, \\
     2 \parap, & \!  \text{if $\begin{cases} \state \in
         \big[0,\tfrac{1}{2}\big) \\ \statetwo \in
           \big[\tfrac{1}{2},1\big) \end{cases} \!\!\!\!\!\!$ or
       $\begin{cases} \state \in \big[\tfrac{1}{2},1\big) \\ \statetwo \in
           \big[0,\tfrac{1}{2}\big). \end{cases}$}
	\end{cases}
\end{align}
\end{subequations}
\begin{figure}[h]
  \begin{center}
    \includegraphics[keepaspectratio, width=0.25\textwidth]{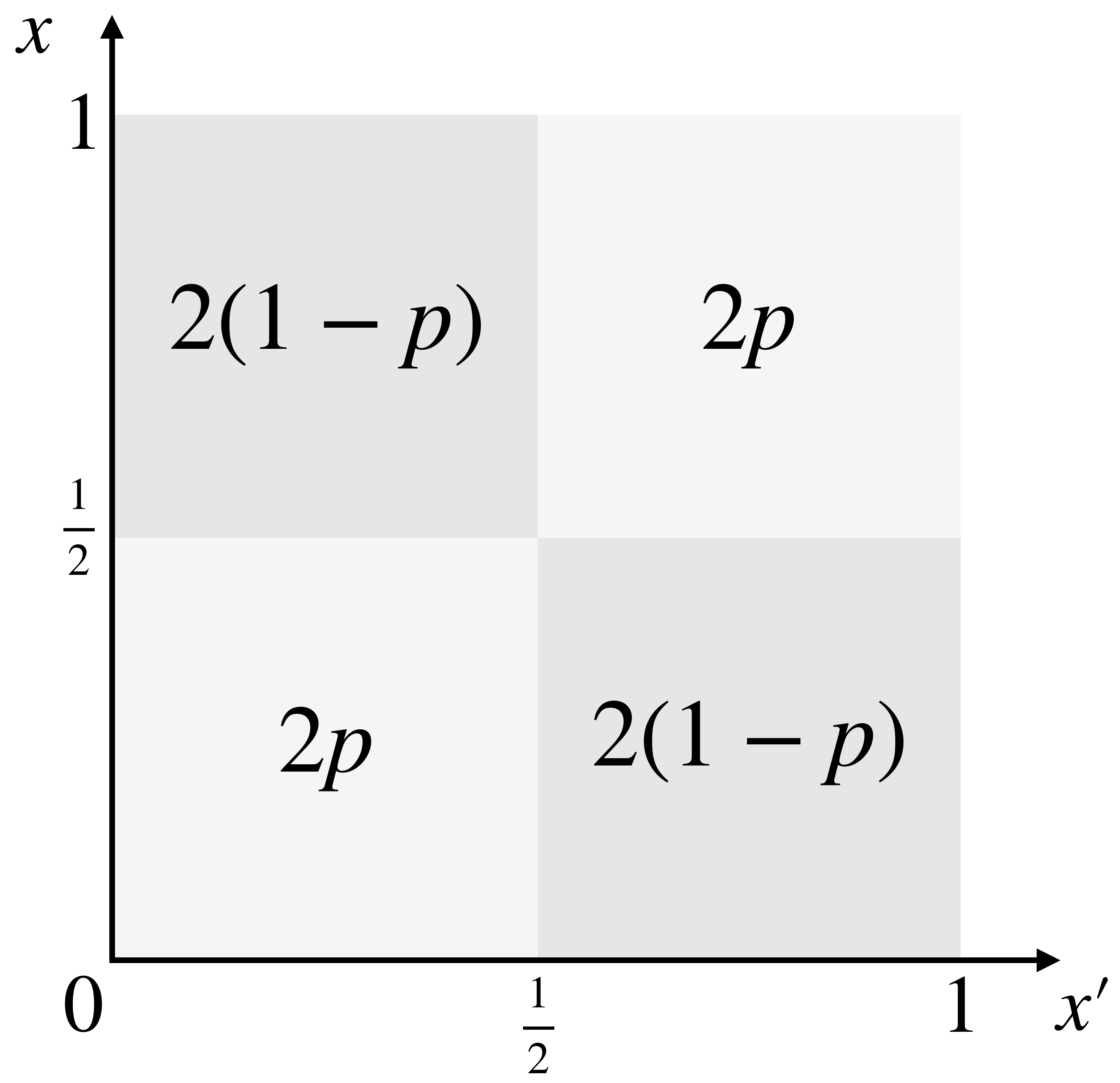}
    \caption{The density of data $\{ (\state_i, \statetwo_i)
      \}_{i=1}^{\numobs}$.} \label{fig:ex_density}
    \vspace{-1em}
  \end{center}
\end{figure}
See Figure~\ref{fig:ex_density} for an illustration of the structure
of this transition function.  By construction, the uniform
distribution $\mu$ is stationary.  The samples $\{ (\state_i, \statetwo_i)
\}_{i=1}^{\numobs}$ were i.i.d. drawn from the pair $(\distr,
\TransOp)$.  The two ensembles of probability transitions (Ensembles A
and B) used in our simulations are distinguished by the choice $\parap
\in \big\{ \frac{1}{4}, \frac{1-\discount}{\discount} \big\}$.
	
In addition to the two ensembles of transition functions, our
experiments involve comparisons between three different kernels, all
of which were constructed based on the Walsh system.  Let $W_j: [0,1)
  \rightarrow \{ -1, 1 \}$ be the $j$-th Walsh function.  For each $i
  = 1, 2, 3$, we define a kernel $\Ker_i: \StateSp \times \StateSp
  \rightarrow \Real$ as
\begin{subequations}  
\begin{align}
  \label{eq:exp_Ker}
  \Ker_i(\state, \statenew) & \defn \sum_{j=1}^{\infty} \eig_j(\Ker_i) \; W_{j-1}(\state)
  W_{j-1}(\statenew).
\end{align}
This choice ensures that each $\Ker_i$ has $\{W_{j-1}\}_{j=1}^\infty$
as its eigenfunctions.  We choose the associated kernel eigenvalues as
\begin{align}
\label{EqnKerEigApp}  
  \eig_j(\Ker_i) & = \begin{cases}  j^{-6/5} & \mbox{for $i = 1$} \\
    j^{-2} & \mbox{for $i = 2$} \\
    \exp \big( - (j-1)^2 \big) & \mbox{for $i = 3$.}
  \end{cases}
\end{align}
\end{subequations}
Let $\RKHS_i$ be the RKHS associated with kernel $\Ker_i$.

\paragraph{Calculations of predicted slopes:}

Let us now calculate the theoretically predicted slopes given in
equations~\eqref{EqnSlopeA} and~\eqref{EqnSlopeB}.  Note that
Corollary~\ref{CorSpline}(b)---see in particular the
bound~\eqref{EqnNewBound}---predicts that the $\Lmu$ error
should scale as
\begin{align}
\label{EqnSlopeScaling}  
\newrad^{\frac{1}{2 \alpha +1}} \Big( \frac{\unibou^2
  \stdfun^2(\thetastar)}{\SqDis^2} \frac{1}{\numobs} \Big)^{\frac{
    \alpha}{2 \alpha+1}}.
\end{align}
Since $\unibou = 1$ for our construction, in order to understand the
scaling with the effective horizon, we need to calculate the
quantities $\stdfun^2(\thetastar)$ and $\newrad$.  Some calculations
show that the value function $\thetastar$ is given by
\begin{align}
\thetastar(\state) & = \tfrac{1}{1 - \discount + 2 \discount \parap} \big(
\mathds{1}\big\{ \state \in \big[0,\tfrac{1}{2}\big) \big\} -
  \mathds{1}\big\{ \state \in \big[\tfrac{1}{2},1\big) \big\} \big).
\end{align}
Consequently, we can see that $\reward, \thetastar \in \RKHS_i$ for $i
= 1,2,3$.

Moreover, we find that variance term $\stdfun^2(\thetastar)$ takes the
form
\begin{subequations}
\begin{align}
\stdfun^2(\thetastar) & = \frac{4 \discount^2 \, \parap (1 -
  \parap)}{(1-\discount+2\discount\parap)^2}.
\end{align}
For each RKHS $\RKHS_i$, the radius $\newrad(\Ker_i)$ is given by
\begin{align}
 \newrad(\Ker_i) & \defn \max \big\{ \norm{\thetastar -
   \reward}_{\RKHS_i}, \tfrac{2 \supnorm{\thetastar}}{\bou(\Ker_i)}
 \big\} = \tfrac{1}{1-\discount+2\discount\parap}\max\Big\{
 \tfrac{\discount(1-2\parap)}{\sqrt{\eig_1(\Ker_i)}},
 \tfrac{2}{\bou(\Ker_i)} \Big\}
\end{align}
with $\bou(\Ker_i) = \sqrt{\sum_j \eig_j(\Ker_i)}$.
\end{subequations}

For the choice $\parap = \frac{1-\discount}{\discount}$, it can be
seen that both $\stdfun^2(\thetastar)$ and $\newrad(\Ker_i)$ scale as
$\tfrac{1}{1 - \discount}$.  Substituting these scalings into
equation~\eqref{EqnSlopeScaling} (and retaining only the dependence on
the effective horizon) yields
\begin{align*}
\Big(\frac{1}{1 - \discount}\Big)^{\tfrac{1}{2 \alpha + 1}} \; \Big(
\frac{1}{(1 - \discount)^3} \Big)^{\tfrac{\alpha}{2 \alpha + 1}} \; =
\; \Big(\frac{1}{1 - \discount} \Big)^{\tfrac{3 \alpha + 1}{2 \alpha +
    1}},
\end{align*}
as claimed in equation~\eqref{EqnSlopeA}.  

For the choice $\parap = \tfrac{1}{4}$, both $\stdfun^2(\thetastar)$
and $\newrad(\Ker_i)$ remain bounded as the effective horizon grows,
so that the corresponding slope is $\Big( \frac{1}{(1 - \discount)^2}
\Big)^{\tfrac{\alpha}{2 \alpha + 1}} \; = \; \Big(\frac{1}{1 -
  \discount} \Big)^{\frac{2 \alpha}{2 \alpha + 1}}$, as claimed in
equation~\eqref{EqnSlopeB}.

%%%%%%%%%%%%%%%%%%%%%%%%%%%%%%%%%%%%%%%%%%%%%%%%%%%%%%%%%%%%%%%%

\section{Technical results for \Cref{thm:ub}}
\label{append:EmpiricalProcess}

In this part, we prove the technical lemmas that underlie the proof of
\Cref{thm:ub}.  \Cref{sec:proof:lemma:decomp} is devoted to the proof
of \Cref{lemma:decomp}, which provides the basic inequality on the
error.  \Cref{sec:proof:lemma:T1,sec:proof:lemma:T3} are devoted,
respectively, to the proof of \Cref{lemma:T1,lemma:T3} that are used
in the proof of \Cref{thm:ub}.  Recall that these two lemmas provide
high-probability upper bounds on the quantities $T_1$ and $T_3$,
respectively, as defined in \Cref{lemma:decomp}.

%%%%%%%%%%%%%%%%%%%%%%%%%%%%%%%%%%%%%%%%%%%%%%%%%%%%%%%%%%%%%%%%%%%%%%%%%%%%%%%%%%%%%%%%%%%%%%%%

\subsection{Proof of Lemma~\ref{lemma:decomp}}
\label{sec:proof:lemma:decomp}
Recall that the estimate $\thetahat$ is defined by the estimating
equation~\eqref{eq:def_thetahat}, whereas the actual population-level estimate
$\thetastar$ satisfies
equation~\eqref{eq:Bellman_project_op}.  Subtracting these two equations
yields
\begin{align*}
\big( \CovOphat + \ridge \IdOp \big) \, \thetahat - \CovOp \thetastar = \big( \CovOphat + \ridge \IdOp - \CovOp \big)
\reward + \discount \big(\CrOphat \thetahat - \CrOp \thetastar\big).
  \end{align*}
We substitute $\thetahat = \thetastar + \Deltahat$ to find
  \begin{align*}
    \big( \CovOphat + \ridge \IdOp - \discount \CrOphat \big) \Deltahat = \big(\CovOphat + \ridge \IdOp -
    \CovOp)(\reward - \thetastar \big) + \discount \big(\CrOphat - \CrOp\big) \thetastar.
  \end{align*}
Again making use of equation~\eqref{eq:Bellman_project_op}, we have $\CovOp(\reward -
\thetastar) + \discount \CrOp \thetastar = 0$, which implies that
  \begin{align*}
    \big( \CovOphat + \ridge \IdOp - \discount \CrOphat \big) \Deltahat = \big( \CovOphat + \ridge \IdOp \big)
    (\reward - \thetastar) + \discount \CrOphat \thetastar.
  \end{align*}
Taking the Hilbert inner product of both sides with $\Deltahat$ then
yields
\begin{align}
    \label{eq:decomp1}
    \hilin[\big]{\Deltahat}{\big( \CovOphat + \ridge \IdOp - \discount \CrOphat \big) \Deltahat} =\hilin[\big]{\Deltahat}{\big( \CovOphat + \ridge \IdOp \big) (\reward -
      \thetastar) + \discount \CrOphat \thetastar}.
\end{align}
The left hand side of equation~\eqref{eq:decomp1} can then be written
as
  \begin{align*} 
    \hilin[\big]{\Deltahat}{\big( \CovOphat + \ridge \IdOp - \discount \CrOphat \big)
      \Deltahat} = \hilin[\big]{\Deltahat}{ (\CovOp - \discount \CrOp)
      \Deltahat} + \ridge \hilnorm{\Deltahat}^2 +
    \hilin[\big]{\Deltahat}{(\GamOpHat - \GamOp) \Deltahat},
  \end{align*}
  where $\GamOp \defn \CovOp - \discount \CrOp$, and $\GamOpHat \defn
  \CovOphat - \discount \CrOphat$.
The right hand side of equation~\eqref{eq:decomp1} satisfies
  \begin{align*} 
    \hilin[\big]{\Deltahat}{\big( \CovOphat + \ridge \IdOp \big) (\reward -
    	\thetastar) + \discount \CrOphat \thetastar} = \hilin[\big]{\Deltahat}{
      \CovOphat(\reward - \thetastar) + \discount \CrOphat \thetastar}
    + \ridge \hilin[\big]{ \Deltahat}{ \reward - \thetastar}.
  \end{align*}
In this way, we reduce equation~\eqref{eq:decomp1} to
\begin{multline}
    \label{eq:decomp3}
 \specfun^2(\Deltahat) \; = \; \hilin[\big]{\Deltahat}{(\CovOp -
   \discount \CrOp) \Deltahat} = \hilin[\big]{\Deltahat}{\CovOphat
   (\reward - \thetastar) + \discount \CrOphat \thetastar} \\
    + \ridge \hilin[\big]{\Deltahat}{\reward - \thetastar} +
    \hilin[\big]{\Deltahat}{(\GamOp - \GamOpHat) \Deltahat} - \ridge
    \hilnorm{\Deltahat}^2.
\end{multline}
We have thus established equality
(ii) in equation~\eqref{eq:decomp} from the lemma statement.

It remains to prove the lower bound (i) in equation~\eqref{eq:decomp}.
Letting $\State \sim \distr$ and $\Statetwo \sim \TransOp(\cdot \mid \State)$, we can
write
\begin{align*}
  \Exp[f(\State)f(\Statetwo)] = \hilin[\big]{ f}{ \CrOp f } \quad \text{and} \quad
  \Exp [f^2(\State)] = \hilin[\big]{ f}{ \CovOp f }.
\end{align*}
Consequently, by applying Young's inequality, we find that
\begin{align}
\label{EqnCrossCovInequal}  
    \underbrace{\Exp [f(\State) f(\Statetwo)]}_{\hilin{f}{\CrOp f}} & \leq
    \frac{1}{2} \Big \{ \Exp[f^2(\State)] + \Exp [f^2(\Statetwo)] \Big \} \; = \;
    \underbrace{\Exp [f^2(\State)]}_{\hilin{f}{\CovOp f}},
  \end{align}
where the equality follows since $\State$ and $\Statetwo$ have the same marginal
distributions, due to the stationarity of $\distr$. 
This completes the proof of \Cref{lemma:decomp}.

%%%%%%%%%%%%%%%%%%%%%%%%%%%%%%%%%%%%%%%%%%%%%%%%%%%%%%%%%%%%%%%%%%%

\subsection{Proof of Lemma~\ref{lemma:T1}}
\label{sec:proof:lemma:T1}

Define the i.i.d. random variables $\svar_i = \reward(\state_i) -
\thetastar(\state_i) + \discount \thetastar(\statetwo_i)$.  Since $\CovOp
\thetastar = \CovOp \reward + \discount \CrOp \thetastar$, we can
write $\Term_1$ as
\begin{align*}
\Term_1 & = \frac{1}{\numobs} \sum_{i=1}^\numobs \Big(\Deltahat(\state_i)
\svar_i - \Exs[\Deltahat(\state_i) \svar_i] \Big).
\end{align*}
For scalars $t > 0$, we define the family of random variables
\begin{align}
Z_\numobs(t) & \defn \sup_{ \substack{ \munorm{f} \leq t
    \\ \hilnorm{f} \leq \newrad}} \Big| \frac{1}{\numobs}
\sum_{i=1}^\numobs \big(f(\state_i) \svar_i - \Exs[f(\state_i) \svar_i] \big)
\Big|,
\end{align}
and let $\tcrit > 0$ be the smallest positive solution to the
inequality
\begin{align*}
  \Exs[Z_\numobs(t)] \leq \SqDis \, \frac{t^2}{4}.
\end{align*}
Our first step is to relate $\tcrit$ to the critical radius $\delcrit$
involved in \Cref{thm:ub}.
\begin{lemma}
\label{lemma:tcritone}  
There is a universal constant $\plaincon_0$ such that,
for any $\SpecialConstant \in \{ \bou \newrad, \unibou
\stdfun(\thetastar) \}$, we have
\begin{align}
\label{eq:tcritone}    
\tcrit & \leq \ucrit(\SpecialConstant) \defn \plaincon_0 \; \newrad \;
\delcrit(\SpecialConstant).
\end{align}
\end{lemma}

The remainder of the proof applies to both $\ucrit(\bou \newrad)$ or
$\ucrit(\unibou \stdfun(\thetastar))$ without any differences, so we
adopt the generic notation $\ucrit$ for either. Our next step is to
use $\ucrit$ to define an event that allows us to establish the claim
of \Cref{lemma:T1}.  For a given $f \in \RKHS$, we say that inequality
\ref{EqnDefnIf} holds when
\begin{align}
\label{EqnDefnIf}
\Big| \frac{1}{\numobs} \sum_{i=1}^\numobs \big(f(\state_i) \svar_i -
\Exs[f(\state_i) \svar_i] \big) \Big| & \; \geq \; \SqDis \, \ucritsq \max \Big\{ 1, \tfrac{\hilnorm{f}}{\newrad} \Big\} + \SqDis \, \ucrit \,
\munorm{f} \; . \tag{$I(f)$}
\end{align}
Here $\plaincon > 0$ is a universal constant to be specified as part
of the proof.  Now consider the event \mbox{$\Aevent \defn \big \{
  \mbox{$\exists$ $f \in \RKHS$ s.t \ref{EqnDefnIf} holds} \big \}$.}  Note
that conditioned on $\Aevent^c$, we have the bound
\begin{align*}
\Term_1 & \leq \; \SqDis \, \ucritsq \max \Big\{ 1, \tfrac{\hilnorm{\Deltahat}}{\newrad} \Big\} + \SqDis \, \ucrit \,
\munorm{\Deltahat} \\ & \leq \plaincon_0^2 \, \SqDis \, \delcritsq \; \Big \{\hilnorm{\Deltahat}^2 + \newrad^2
\Big \} + \plaincon_0 \, \SqDis \, \newrad \, \munorm{\Deltahat} \delcrit, 
\end{align*}
as desired.

Consequently, the remainder of our proof is directed at bounding
$\Prob[\Aevent]$.  We do so by relating the event $\Aevent$ to a tail
event associated with the random variable $Z_\numobs(\ucrit)$.  In
particular, we make the following claim:
\begin{lemma}
\label{lemma:Aevent:tail}  
We have the upper bound
\begin{align}
  \Prob[\Aevent] & \; \leq \; \Prob \big[ Z_\numobs(\ucrit) \geq
    \SqDis \, \ucrit^2 \big].
  \end{align}
\end{lemma}

\noindent Our final lemma provides control on the upper tail of
$Z_\numobs(\ucrit)$.
\begin{lemma}
\label{lemma:Ztail}
There is a universal constant $\plaincon_1$ such that
  \begin{align}
 \Prob \big[Z_\numobs(\ucrit) \geq \SqDis \, \ucritsq \big] & \leq \exp \big(
 -\plaincon_1 \, \numobs \, \tfrac{\ucritsq \, \SqDis^2}{\bou^2 \newrad^2}\big) \; = \; \exp
 \big( -\plaincon_1 \plaincon_0^2 \, \tfrac{\numobs \delcritsq \,
   \SqDis^2}{\bou^2}\big).
  \end{align}
\end{lemma}

\noindent Combining \Cref{lemma:Aevent:tail,lemma:Ztail} yields the
conclusion of \Cref{lemma:T1} with $\plaincon^\prime = \plaincon_1 \plaincon_0^2$. \\

\noindent It remains to prove our three auxiliary lemmas, and we prove
\Cref{lemma:tcritone,lemma:Aevent:tail,lemma:Ztail} in
\Cref{sec:proof:lemma:tcritone,sec:proof:lemma:Aevent:tail,sec:proof:lemma:Ztail},
respectively.

%%%%%%%%%%%%%%%%%%%%%%%%%%%%%%%%%%%%%%%%%%%%%%%%%%%%%%%%%%%%%%%%%%%%%%%%%%%%%%%%%

\subsubsection{Proof of Lemma~\ref{lemma:tcritone}}
\label{sec:proof:lemma:tcritone}

By definition of $\tcrit$, we have $\SqDis \, \tfrac{\tcritsq}{4} =
\Exs[Z_\numobs(\tcrit)]$.  Consequently, we can prove the claim by
upper bounding the expectation.  Let $\{\varepsilon_i\}_{i=1}^\numobs$
be an i.i.d. sequence of Rademacher variables, independent of $\{ (\state_i,
\statetwo_i) \}_{i=1}^\numobs$. From a standard symmetrization argument, we
have
\begin{align*}
  \Exp[Z_\numobs(\tcrit)] & \leq 2 \; \Exp \Bigg[ \sup_{\substack{
        \munorm{f} \leq \tcrit \\ \hilnorm{f} \leq \newrad }} \Big|
    \frac{1}{\numobs} \sum_{i=1}^\numobs \varepsilon_i f(\state_i) \svar_i
    \Big| \Bigg].
\end{align*}

\paragraph{Proof for $\delcrit(\bou \newrad)$:}
We begin by proving the claim when $\SpecialConstant = \bou \newrad$.
Note that for any $(\state_i, \statetwo_i)$, we have
\begin{align*}
|\svar_i| & = \big| \reward(\state_i) - \thetastar(\state_i) + \discount \thetastar(\statetwo_i) \big| \; \leq \; \bou \hilnorm{\thetastar - \reward} +
\supnorm{\thetastar} \; \leq \; 2 \bou \newrad,
\end{align*}
using the definition of $\newrad$.  Consequently, by the
Ledoux-Talagrand contraction, we have
\begin{align*}
\SqDis \, \frac{\tcritsq}{4} \; = \; \Exp[Z_\numobs(\tcrit)] \; \leq \; 4 \bou
\newrad \; \Exp \Bigg[ \sup_{\substack{ \munorm{f} \leq \tcrit
      \\ \hilnorm{f} \leq \newrad }} \Big| \frac{1}{\numobs}
  \sum_{i=1}^\numobs \varepsilon_i f(\state_i) \Big| \Bigg] &
\stackrel{(i)}{=} 4 \bou \newrad^2 \; \Exp \Bigg[ \sup_{\substack{
      \munorm{g} \leq \tcrit/\newrad \\ \hilnorm{g} \leq 1 }} \Big|
  \frac{1}{\numobs} \sum_{i=1}^\numobs \varepsilon_i g(\state_i) \Big|
  \Bigg] \\
& \stackrel{(ii)}{\leq} 4 \bou \newrad^2 \, \frac{\newrad
    \SqDis}{\bou \newrad} \, \Big \{ \delcritsq + \frac{\tcrit
    \delcrit}{\newrad} \Big \} \\
& = 4 \newrad^2 \, \SqDis \, \Big \{ \delcritsq + \frac{\tcrit
    \delcrit}{\newrad} \Big \} 
\end{align*}
where equality (i) follows by reparameterizing the supremum in terms
of the rescaled functions $g = f/\newrad$, and inequality (ii) follows
from the definition of $\delcrit(\bou \newrad)$.  This implies that there is a universal constant $\plaincon_0$ such
that $\tcrit \leq \plaincon_0 \: \newrad \: \delcrit$, as
claimed.

\paragraph{Proof for $\delcrit(\unibou \stdfun(\thetastar))$:}
In this case, we begin by observing that
\begin{align*}
\Exp[Z_\numobs(\tcrit)] & \leq 2 \stdfun(\thetastar) \; \Exp \Bigg[
  \sup_{\substack{ \munorm{f} \leq \tcrit \\ \hilnorm{f} \leq \newrad
  }} \Big| \frac{1}{\numobs} \sum_{i=1}^\numobs f(\state_i) \xi_i \Big|
  \Bigg]
\end{align*}
where the variables $\xi_i = \tfrac{ \varepsilon_i \svar(\state_i,
  \statetwo_i)}{\stdfun(\thetastar)}$ have zero mean and unit
variance.

We now reparameterize the supremum in terms of the rescaled functions
$g = f/\newrad$, so that $\hilnorm{g} \leq 1$ and $\munorm{g} \leq
\tcrit/\newrad$.  In this way, we find that
\begin{align*}
\SqDis \, \frac{\tcritsq}{4} \; = \; \Exs[Z_\numobs(\tcrit)] & \; \leq \; \big( 2
\stdfun(\thetastar) \newrad \big) \; \Exp \Bigg[ \sup_{\substack{
      \munorm{g} \leq \tcrit/\newrad \\ \hilnorm{g} \leq 1 }} \Big|
  \frac{1}{\numobs} \sum_{i=1}^\numobs g(\state_i) \xi_i \Big| \Bigg].
\end{align*}
Recall that $g \in \RKHS$ can be written in the form $g =
\sum_{j=1}^\infty g_j \base_j$ for some coefficients
$\{g_j\}_{j=1}^\infty$ such that $\munorm{g}^2 = \sum_{j=1}^\infty
g_j^2$, and $\hilnorm{g}^2 = \sum_{j=1}^\infty \tfrac{g_j^2}{\eig_j}$.
Consequently, for any $g$ involved in the supremum, we have
\begin{align*}
 \Exp \Bigg[ \sup_{\substack{ \munorm{g} \leq \tcrit/\newrad
       \\ \hilnorm{g} \leq 1 }} \Big| \frac{1}{\numobs}
   \sum_{i=1}^\numobs g(\state_i) \xi_i \Big| \Bigg] & = \Exp \Bigg[
   \sup_{\substack{ \munorm{g} \leq \tcrit/\newrad \\ \hilnorm{g} \leq
       1 }} \bigg| \sum_{j=1}^\infty g_j \Big( \frac{1}{\numobs}
   \sum_{i=1}^\numobs \xi_i \base_j(\state_i) \Big) \bigg| \Bigg] \\
& \leq \Exp \Bigg[ \bigg\{ 2 \sum_{j=1}^\infty \min \big\{
       \tfrac{\tcritsq}{\newrad^2}, \eig_j \big\} \Big( \frac{1}{\numobs}
       \sum_{i=1}^\numobs \xi_i \base_j(\state_i) \Big)^2 \bigg\}^{1/2} \Bigg]
     \\
     & \leq \sqrt{ \frac{2}{\numobs} \sum_{j=1}^\infty \min \big\{
     \tfrac{\tcritsq}{\newrad^2}, \eig_j \big\} \Exs\big[\xi^2 \base^2_j(\State)\big] } \; .
\end{align*}
Since $\Exs[\xi^2] = 1$ and $\base^2_j(\State) \leq \unibou^2$ by assumption,
we have $\Exs\big[\xi^2 \base^2_j(\State)\big] \leq \unibou^2$.  Thus, we have
established that
\begin{align*}
\Exp \Bigg[ \sup_{\substack{ \munorm{g} \leq \tcrit/\newrad
      \\ \hilnorm{g} \leq 1 }} \Big| \frac{1}{\numobs}
  \sum_{i=1}^\numobs g(\state_i) \xi_i \Big| \Bigg] & \leq
\unibou \sqrt{\frac{2}{\numobs} \sum_{j=1}^\infty \min \big\{
	\tfrac{\tcritsq}{\newrad^2}, \eig_j \big\} } \; \leq \; \sqrt{2} \unibou
\, \frac{\newrad \SqDis}{\unibou \stdfun(\thetastar)} \, \Big \{
\delcrit^2 + \frac{\delcrit \tcrit}{\newrad} \Big \},
\end{align*}
where the final inequality follows from the definition of $\delcrit =
\delcrit(\unibou \stdfun(\thetastar))$.

Putting together all the pieces, we have
\begin{align*}
  \SqDis \, \frac{\tcritsq}{4} \leq \Exs[Z_\numobs(\tcrit)] & \leq \big(2
  \stdfun(\thetastar) \newrad \big) \, \frac{\sqrt{2} \, \newrad \,
    \SqDis}{\stdfun(\thetastar)} \, \Big \{ \delcrit^2 +
  \frac{\delcrit \tcrit}{\newrad} \Big \} \; = \; 2 \sqrt{2} \; \newrad^2 \, (1-
  \discount) \, \Big \{ \delcrit^2 + \frac{\delcrit \tcrit}{\newrad} \Big
  \}.
\end{align*}
This implies that there is a universal
constant $\plaincon_0$ such that $\tcrit \leq \plaincon_0 \: \newrad \: \delcrit$, as claimed.

%%%%%%%%%%%%%%%%%%%%%%%%%%%%%%%%%%%%%%%%%%%%%%%%%%%%%%%%%%%%%%%%%%%%%%%%%%%%%%%%%%%%%%%%%%%%%%%

\subsubsection{Proof of Lemma~\ref{lemma:Aevent:tail}}
\label{sec:proof:lemma:Aevent:tail}

First, we claim that if \ref{EqnDefnIf} holds for any function, then we can
find a function $g$ with $\hilnorm{g} \leq \newrad$ such that
\begin{align} \label{EqnAevent1} \Big| \frac{1}{\numobs} \sum_{i=1}^\numobs g(\state_i) \svar_i \Big| & \; \geq \; \SqDis \, \big\{ \ucrit^2 + \ucrit \, \munorm{g} \big\}. \end{align}
  Indeed, if $\hilnorm{f} \leq \newrad$, then we are done.
Otherwise, we define the rescaled function $g =
\tfrac{\newrad}{\hilnorm{f}} f$, and note that it also belongs to the
Hilbert space, and satisfies $\hilnorm{g} = \newrad$.  Moreover, since
$f$ satisfies \ref{EqnDefnIf}, we have
\begin{align*}
  \Big| \frac{1}{\numobs} \sum_{i=1}^\numobs g(\state_i) \svar_i \Big| \; =
  \; \frac{\newrad}{\hilnorm{f}} \, \Big| \frac{1}{\numobs}
  \sum_{i=1}^\numobs f(\state_i) \svar_i \Big| & \geq
  \frac{\newrad}{\hilnorm{f}} \, \Big\{ \SqDis \, \ucritsq \max \Big\{ 1, \tfrac{\hilnorm{f}}{\newrad} \Big\} + \SqDis \, \ucrit \,
  \munorm{f} \Big\} \; \\ & =
  \; \SqDis \, \big\{ \ucrit^2 + \ucrit \, \munorm{g} \big\}.
\end{align*}

Next, we claim that we can also find a function $h$ such that, in
addition, satisfies the bound $\munorm{h} \leq \ucrit$ and 
\begin{align} \label{EqnKentoJumping}
	\Big| \frac{1}{\numobs} \sum_{i=1}^\numobs h(\state_i) \svar_i \Big| & \; \geq \; \SqDis \, \ucritsq.
\end{align}
If the
function $g$ constructed above satisfies $\munorm{g} \leq \ucrit$,
then we are done.  Otherwise, we set $h = \tfrac{\ucrit}{\munorm{g}}
g$. Note that $h \in \RKHS$ satifies $\hilnorm{h} \leq
\hilnorm{g} = \newrad$ and $\munorm{h} = \ucrit$.  Moreover, since
$g$ satisfies inequality~\eqref{EqnAevent1}, we have
\begin{align*}
%\label{EqnKentoJumping}
\Big| \frac{1}{\numobs} \sum_{i=1}^\numobs h(\state_i) \svar_i \Big| \; =
\; \frac{\ucrit}{\munorm{g}} \, \Big| \frac{1}{\numobs}
\sum_{i=1}^\numobs g(\state_i) \svar_i \Big| & \; \geq \;
\frac{\ucrit}{\munorm{g}} \, \SqDis \, \max \big\{ \ucrit^2 + \ucrit \, \munorm{g} \big\} \; \geq \;
\SqDis \, \ucritsq.
\end{align*}
Consequently, we have shown that if the event $\Aevent$ holds, then we
can find a function $h$ such that $\hilnorm{h} \leq \newrad$ and
$\munorm{h} \leq \ucrit$, and such that the lower
bound~\eqref{EqnKentoJumping} holds.  The existence of this $h$
implies that $Z_\numobs(\ucrit) \geq \SqDis \, \ucritsq$, which shows
that $\Aevent \subset \{Z_\numobs(\ucrit) \geq \SqDis \, \ucritsq \}$,
as claimed.

\subsubsection{Proof of Lemma~\ref{lemma:Ztail}}
\label{sec:proof:lemma:Ztail}

By definition of $\ucrit$ from \Cref{lemma:tcritone}, we have
$\ucrit \geq \tcrit$, and hence
\begin{align*}
\Exs[Z_\numobs(\ucrit)] \; \leq \; \SqDis \; \ucrit \, \frac{\tcrit}{4} \; \leq \;
\SqDis \, \frac{\ucritsq}{4},
\end{align*}
using the definition of $\tcrit$.  Our next step is to prove that
there is a universal constant $\plaincon_1$ such that
\begin{align}
  \label{EqnZconc}
  \Prob \Big[Z_\numobs(\ucrit) \geq 2 \, \Exp[Z_\numobs(\ucrit)] +
   \SqDis \, \tfrac{\ucritsq}{2} \Big] \; \leq \; \exp \big( - \plaincon_1 \numobs \, \tfrac{
    \ucritsq \SqDis^2}{\bou^2 \newrad^2} \big) \, . %  \; = \; \exp \big(-\plaincon_1 \plaincon_0^2
%  \tfrac{\numobs \delcritsq \SqDis^2}{\bou^2} \big).
\end{align}
The statement given in the lemma follows by combining these two
claims.

It remains to prove the tail bound~\eqref{EqnZconc}.  By definition,
the random variable $Z_\numobs(\tcrit)$ corresponds to the supremum of
an empirical process in terms of functions of the form
\begin{align*}
  g(\state_i, \statetwo_i) & = f(\state_i) \underbrace{\Big \{ \big(\reward(\state_i) -
    \thetastar(\state_i) \big) + \discount \thetastar(\statetwo_i) \Big
    \}}_{\svar(\state_i, \statetwo_i)}
\end{align*}
where $f$ varies, while satisfying the constraints $\munorm{f} \leq
\ucrit$ and $\hilnorm{f} \leq \newrad$.  In order to establish
concentration for this supremum, we can apply Talagrand's theorem
(cf. Theorem~3.27 in the book~\cite{wainwright2019high}).  Doing so
requires us to bound $\|g\|_\infty$, as well as $\Exp[g^2]$,
uniformly over the relevant function class.

Recall that our definition of $\bou$ ensures that $\|h\|_\infty \leq
\bou \, \hilnorm{h}$ for any $h \in \RKHS$.  Consequently, we have
\begin{align*}
  \sup_{\state, \statetwo} |\svar(\state, \statetwo)| & \leq \bou \hilnorm{\thetastar -
    \reward} + \supnorm{\thetastar} \; = \bou \newrad, \quad
  \mbox{and} \quad \|g\|_\infty = \sup_{\state, \statetwo} \big| f(\state) \svar(\state, \statetwo)
  \big| \leq \|f\|_\infty \bou \newrad \leq \bou^2 \newrad^2,
\end{align*}
where we have used the fact that $\supnorm{f} \leq \bou \hilnorm{f}
\leq \bou \newrad$.  On the other hand, we have
\begin{align*}
\munorm{g}^2 & = \Exp \big[ f^2(\State) \svar^2(\State, \Statetwo) \big] \; \leq \;
\bou^2 \newrad^2 \Exp[ f^2(\State) \big] \; \leq \;  \bou^2 \newrad^2
\ucritsq.
\end{align*}
By Talagrand's theorem (cf. equation (3.86) in the
book~\cite{wainwright2019high}), there are universal constants $\plaincon_2$, $\plaincon_3$ such that
\begin{align*}
  \Prob \Big[Z_\numobs(\ucrit) \geq 2 \, \Exp[Z_\numobs(\ucrit)] + \plaincon_2
   \bou \newrad \ucrit \sqrt{s} + \plaincon_3 \bou^2 \newrad^2 s \Big] &
  \leq \exp(-\numobs s).
\end{align*}
Setting $s = \plaincon_1 \tfrac{\ucritsq \SqDis^2}{\bou^2 \newrad^2}$ for a
sufficiently small constant $\plaincon_1$ yields the claim in
equation~\eqref{EqnZconc}.

%%%%%%%%%%%%%%%%%%%%%%%%%%%%%%%%%%%%%%%%%%%%%%%%%%%%%%%%%%%%%%%

\subsection{Proof of Lemma~\ref{lemma:T3}}
\label{sec:proof:lemma:T3}

Recall our definitions of the operator $\GamOp = \CovOp - \discount
\CrOp$, as well as its empirical version $\GamOpHat = \CovOphat -
\discount \CrOphat$, as given in \Cref{lemma:decomp}.  Recall the
functional $\specfun^2(f) = \Exp [f^2(\State) - \discount f(\State) f(\Statetwo)]$, as
previously defined in equation~\eqref{EqnDefnSpecfun}.  For each $t > 0$, define
the random variable
\begin{align}
\label{EqnDefnZtil}
\Ztil_\numobs(t) & \defn \sup_{ \substack{ \specfun(f) \leq t
    \\ \hilnorm{f} \leq \newrad}} \big|\hilin{f}{(\GamOpHat -
  \GamOp)f} \big|,
\end{align}
and let $\tcrit > 0$ be the smallest positive solution to the
inequality
\begin{align}
  \label{EqnDefnTcritTwo}
\Exp[\Ztil_\numobs(t)] \; \leq \; \frac{t^2}{8}.
\end{align}
We begin by relating this critical radius $\tcrit$ to our original
radius $\delcrit$:
\begin{lemma}
\label{lemma:tcrittwo}  
There is a universal constant $\plaincon_0$ such that
\begin{align}
  \tcrit \leq \ucrit \defn \plaincon_0 \, \newrad \sqrt{1- \discount} \;
  \delcrit(\bou \newrad).
\end{align}
If, in addition, the sample size condition~\eqref{EqnNbound} holds,
then the same bound holds with $\delcrit(\unibou
\stdfun(\thetastar))$.
\end{lemma}
\noindent See \Cref{sec:proof:lemma:tcrittwo} for the proof of this
claim. \\

With this set-up, the remainder of proof has a structure similar to
that of \Cref{lemma:T1}. We say that a function $f \in \RKHS$
satisfies inequality \ref{EqnDefnJf} if
\begin{align}
\label{EqnDefnJf}
\big|\hilin{f}{(\GamOpHat - \GamOp) f}\big| \; \geq \; \ucritsq \; \max \Big \{ 1,
\tfrac{\hilnorm{f}^2}{\newrad^2} \Big \} + \frac{\specfun^2(f)}{2}. \tag{$J(f)$}
\end{align}
Now consider the event \mbox{$\Bevent \defn \big \{ \mbox{$\exists$ $f
    \in \RKHS$ s.t \ref{EqnDefnJf} holds} \big \}$.}  Note that conditioned on
$\Bevent^c$, we have the bound
\begin{align*}
\big|\hilin{\Deltahat}{(\GamOpHat - \GamOp) \Deltahat}\big| & \leq \ucritsq \,
\max \Big \{ 1, \tfrac{\hilnorm{\Deltahat}^2}{\newrad^2} \Big \} +
\frac{\specfun^2(\Deltahat)}{2} \; \leq \plaincon_0^2 \, \delcritsq \, \SqDis \,
\max \Big \{ \newrad^2, \hilnorm{\Deltahat}^2 \Big \} +
\frac{\specfun^2(\Deltahat)}{2}
\end{align*}
where the second inequality follows from the definition of $\ucrit$
given in equation~\eqref{EqnDefnTcritTwo}.  This is the bound claimed
in the statement of \Cref{lemma:T3}.  Consequently, it suffices to
bound the probability $\Prob[\Bevent]$.

We begin by upper bounding the probability of $\Prob[\Bevent]$ in
terms of the tail behavior of the random variable
$\Ztil_\numobs(\ucrit)$ as follows:
\begin{lemma}
\label{lemma:Bevent:tail}  
We have the upper bound
\begin{align}
  \Prob[\Bevent] & \leq \Prob \Big[ \Ztil_\numobs(\ucrit) \geq
    \tfrac{\ucritsq}{2} \Big].
\end{align}
\end{lemma}
\noindent See \Cref{sec:proof:lemma:Bevent:tail} for the proof. \\
    
\noindent Our second lemma provides control on the upper tail of
$\Ztil_\numobs(\ucrit)$.
\begin{lemma}
  \label{lemma:Ztiltail}
There is a universal constant $\plaincon_1$ such that
\begin{align}
  \Prob \Big[ \Ztil_\numobs(\ucrit) \geq \tfrac{\ucritsq}{2} \Big] &
  \leq \; \exp \big(-\plaincon_1 \numobs \tfrac{\ucritsq}{\bou^2 \newrad^2} \big)
  \; = \; \exp \big(-\plaincon_1 \plaincon_0^2 \, \tfrac{\numobs \delcrit^2
    \SqDis}{\bou^2} \big).
\end{align}
\end{lemma}
\noindent See \Cref{sec:proof:lemma:Ztiltail} for the proof of this
claim.

%%%%%%%%%%%%%%%%%%%%%%%%%%%%%%%%%%%%%%%%%%%%%%%%%%%%%%%%%%%%%%%%%

\subsubsection{Proof of Lemma~\ref{lemma:tcrittwo}}
\label{sec:proof:lemma:tcrittwo}

Define the random variables $\statenew_i = (\state_i, \statetwo_i)$ along with the
function \mbox{$g(\statenew_i) \defn f^2(\state_i) - \discount f(\state_i) f(\statetwo_i)$,}
and note that $\Ztil_\numobs(\tcrit)$ is a supremum of the empirical
process $\big\{ \tfrac{1}{\numobs} \sum_{i=1}^\numobs (g(\statenew_i) -
\Exs[g(\statenew_i)]) \big\}$ as $g$ varies as a function of $f$, and $f$ satisfies
the constraints $\hilnorm{f} \leq \newrad$ and $\specfun(f) \leq
\tcrit$.

By a standard symmetrization argument, we have
\begin{align*}
\Exp [ \Ztil_\numobs(\tcrit)] & = \Exp \Bigg[ \sup_{
    \substack{\specfun(f) \leq \tcrit \\ \hilnorm{f} \leq \newrad}} \Big|
  \frac{1}{\numobs} \sum_{i=1}^\numobs\big( g(\statenew_i) - \Exs[g(\Statenew)] \big)
  \Big| \Bigg] \; \leq \; 2 \Exp \Bigg[ \sup_{ \substack{\specfun(f)
      \leq \tcrit \\ \hilnorm{f} \leq \newrad}} \Big| \frac{1}{\numobs}
  \sum_{i=1}^\numobs \varepsilon_i g(\statenew_i) \Big| \Bigg],
\end{align*}
where $\{\varepsilon_i\}_{i=1}^\numobs$ is an i.i.d. sequence of
Rademacher variables.

Now by the lower bound~\eqref{EqnUsefulLower}, the constraint
$\specfun(f) \leq \tcrit$ implies that $\munorm{f} \leq
\tfrac{\tcrit}{\sqrt{1-\discount}}$.  Introducing the shorthand
$\Ellipse = \big \{ f \in \RKHS \mid \munorm{f} \leq
\frac{\tcrit}{\sqrt{1-\discount}}, \; \hilnorm{f} \leq \newrad \big
\}$, we have
\begin{align*}
\tfrac{1}{2} \Exp [ \Ztil_\numobs(\tcrit)] & \leq \Exp \Bigg[ \sup_{f
    \in \Ellipse} \Big| \frac{1}{\numobs} \sum_{i=1}^\numobs
  \varepsilon_i \big(f^2 (\state_i) - \discount f(\state_i) f(\statetwo_i) \big) \Big| \Bigg] \\
& \leq \Exp \Bigg[ \sup_{f \in \Ellipse} \Big| \frac{1}{\numobs}
  \sum_{i=1}^\numobs \varepsilon_i f^2(\state_i) \Big| \Bigg] + \Exp \Bigg[
  \sup_{f \in \Ellipse} \Big| \frac{1}{\numobs} \sum_{i=1}^\numobs
  \varepsilon_i f(\state_i) f(\statetwo_i) \Big| \Bigg].
\end{align*}

From this point, our proof diverges, depending on the
two choices of $\delcrit$.

\paragraph{Proof for $\delcrit(\bou \newrad)$:}
In this case, we use the fact that $\|f\|_\infty \leq \bou \newrad$.
Combined with the Ledoux-Talagrand contraction, we find that
\begin{align*}
  \tfrac{1}{2} \Exp [ \Ztil_\numobs(\tcrit)] & \leq 2 \bou \newrad
  \; \Exp \Bigg[ \sup_{f \in \Ellipse} \Big| \frac{1}{\numobs}
    \sum_{i=1}^\numobs \varepsilon_i f(\state_i) \Big| \Bigg] + 2 \bou
  \newrad \; \Exp \Bigg[ \sup_{f \in \Ellipse} \Big| \frac{1}{\numobs}
    \sum_{i=1}^\numobs \varepsilon_i f(\state_i) \Big| \Bigg] \\
& = 4 \bou \newrad \; \Exp \Bigg[ \sup_{f \in \Ellipse} \Big|
  \frac{1}{\numobs} \sum_{i=1}^\numobs \varepsilon_i f(\state_i) \Big|
  \Bigg].
\end{align*}
Define the rescaled ellipse $\Elltil \defn \tfrac{1}{\newrad} \Ellipse
\; = \; \big\{ h \in \RKHS \mid \munorm{h} \leq \frac{\tcrit}{\newrad
  \sqrt{1- \discount}}, \; \hilnorm{h} \leq 1 \big \}$.  By
construction, we have
\begin{align*}
\Exp \Bigg[ \sup_{f \in \Ellipse} \Big| \frac{1}{\numobs}
  \sum_{i=1}^\numobs \varepsilon_i f(\state_i) \Big| \Bigg] & = \newrad \; \Exp
\Bigg[ \sup_{h \in \Elltil} \Big| \frac{1}{\numobs} \sum_{i=1}^\numobs
  \varepsilon_i h(\state_i) \Big| \Bigg] .
\end{align*}
Finally, by definition of $\delcrit(\bou \newrad)$, we are guaranteed
that
\begin{align*}
\Exp \Bigg[ \sup_{h \in \Elltil} \Big| \frac{1}{\numobs}
  \sum_{i=1}^\numobs \varepsilon_i h(\state_i) \Big| \Bigg] & \leq
\frac{\newrad \SqDis}{\bou \newrad} \, \max \Big
\{ \delcritsq, \; \delcrit \tfrac{\tcrit}{\newrad \sqrt{1-\discount}}
\Big \}.
\end{align*}
Putting together the pieces, using the definition of $\tcrit$, we have
shown that
\begin{align*}
\frac{\tcrit^2}{8} \; = \; \Exp [ \Ztil_\numobs(\tcrit)] & \leq \big(4
\bou \newrad^2 \big) \; \frac{\newrad \SqDis}{\bou \newrad} \, \max \Big
\{ \delcritsq, \: \delcrit \tfrac{\tcrit}{\newrad \sqrt{1-\discount}}
\Big \} \; = 4 \newrad^2 \, \SqDis \, \max \Big \{ \delcritsq, \: \delcrit
\tfrac{\tcrit}{\newrad \sqrt{1-\discount}} \Big \}.
\end{align*}
This implies that there is a universal constant $\plaincon_0$ such that
$\tcritsq \leq \plaincon_0^2 \, \newrad^2 \, \SqDis \, \delcritsq$, as
claimed in the statement of the lemma.

\paragraph{Proof for $\delcrit(\unibou \stdfun(\thetastar))$:}
Recall the ellipse $\Ellipse = \big \{ f \in \RKHS \mid \munorm{f}
\leq \frac{\tcrit}{\sqrt{1-\discount}}, \; \hilnorm{f} \leq \newrad
\big \}$. We claim that it suffices to show that under the assumed
bound~\eqref{EqnNbound} on the sample size, we have
\begin{align}
\label{EqnSupNormBound}
\sup_{f \in \Ellipse} \supnorm{f} \leq \beta \defn \frac{\unibou
  \stdfun(\thetastar)}{128} \Big\{ 1 + \frac{\tcrit}{\delcrit \newrad
  \sqrt{1-\discount} } \Big \}.
\end{align}
Indeed, if this bound holds, then we can perform the Ledoux-Talagrand
contraction with the constraint $\supnorm{f} \leq \beta$, so as to conclude that
\begin{align*}
\tfrac{1}{2} \Exp [ \Ztil_\numobs(\tcrit)] & \leq 4 \beta \; \Exp \Bigg[
  \sup_{f \in \Ellipse} \Big| \frac{1}{\numobs} \sum_{i=1}^\numobs
  \varepsilon_i f(\state_i) \Big| \Bigg].
\end{align*}
Proceeding as before, we find that
\begin{align*}
\frac{\tcritsq}{16} \; = \; \tfrac{1}{2} \Exp [ \Ztil_\numobs(\tcrit)]
& \leq 4 \beta \; \frac{\newrad^2 \SqDis}{\unibou
  \stdfun(\thetastar)} \; \Big \{ \delcritsq + \frac{\delcrit
  \tcrit}{\newrad \sqrt{1-\discount}} \Big \} \\
& = \frac{\newrad^2 \SqDis}{32} \; \Big \{ 1 +
\frac{\tcrit}{\delcrit \newrad \sqrt{1-\discount}} \Big \} \; \Big \{
\delcritsq + \frac{\delcrit \tcrit}{\newrad \sqrt{1-\discount}} \Big
\} \\
& = \frac{\newrad^2 \SqDis}{32} \; \Big \{ \delcrit^2 + 2
\frac{\delcrit \tcrit}{\newrad \sqrt{1-\discount}} \Big \} +
\frac{\tcritsq}{32}.
\end{align*}
This bound implies that $\tcrit \leq \plaincon_0 \, \newrad \, \sqrt{1-\discount} \: \delcrit$,
as claimed.

Accordingly, let us prove the bound~\eqref{EqnSupNormBound}. Any $f
\in \Ellipse$ has the expansion $f = \sum_{j\geq 1} f_j \base_j$ for
some coefficients such that $\sum_{j=1}^\infty f_j^2 \leq
\tcritsq/\SqDis$ and $\sum_{j=1}^\infty f_j^2/\eig_j \leq
\newrad^2$.  Consequently, we have
\begin{align*}
  \supnorm{f} \; = \sup_{\state} \Big| \sum_{j=1}^\infty f_j \base_j(\state) \Big| & \leq
  \newrad \; \sup_{\state} \Big \{ 2 \sum_{j=1}^\infty \min \big\{
  \tfrac{\tcrit^2}{\newrad^2 \SqDis}, \eig_j \big\} \base_j^2(\state) \Big
  \}^{1/2} \\
 & \stackrel{(i)}{\leq} \newrad \unibou \; \Big \{ 2 \sum_{j=1}^\infty \min
  \big\{ \tfrac{\tcrit^2}{\newrad^2 \SqDis}, \eig_j \big\} \Big \}^{1/2} \\
& \stackrel{(ii)}{\leq} \newrad \unibou \; \sqrt{2 \numobs} \; \frac{\newrad
    \SqDis}{\unibou \stdfun(\thetastar)} \; \Big \{ \delcritsq +
  \frac{\delcrit \tcrit}{\newrad \sqrt{1-\discount}} \Big \} \\
& = \Big \{ \frac{\sqrt{2 \numobs} \newrad^2 \SqDis
      \delcritsq}{\unibou \stdfun^2(\thetastar)} \Big \} \; \unibou
    \stdfun(\thetastar) \Big \{ 1 + \frac{\tcrit}{\delcrit \newrad
      \sqrt{1-\discount} } \Big \} \\
& \stackrel{(iii)}{\leq} \frac{\unibou \stdfun(\thetastar)}{128} \Big
    \{ 1 + \frac{\tcrit}{\delcrit \newrad \sqrt{1-\discount} } \Big
    \},
\end{align*}
where step (i) uses the fact that $\|\base_j\|_\infty \leq \unibou$ by
assumption; step (ii) uses the definition of \mbox{$\delcrit =
  \delcrit(\unibou \stdfun(\thetastar))$;} and step (iii) follows from
the assumed bound~\eqref{EqnNbound}.

%%%%%%%%%%%%%%%%%%%%%%%%%%%%%%%%%%%%%%%%%%%%%%%%%%%%%%%%%%%%%%%%%%%%%%%%%%%%%%

\subsubsection{Proof of Lemma~\ref{lemma:Bevent:tail}}
\label{sec:proof:lemma:Bevent:tail}

We first claim that if there is some $f \in \RKHS$ such that \ref{EqnDefnJf}
holds, then we can construct a function $g \in \RKHS$ such that
$\hilnorm{g} \leq \newrad$, and
\begin{align} \label{EqnBevent1}
	\big| \hilin{f}{(\GamOpHat - \GamOp) f} \big| \geq \ucritsq + \frac{\specfun^2(g)}{2}.
\end{align}
Indeed, if the given function $f$ satisfies $\hilnorm{f} \leq
\newrad$, then we are done.  Otherwise, we define the rescaled
function $g \defn \tfrac{\newrad \: f}{\hilnorm{f}} \in \RKHS$, which
satisfies $\hilnorm{g} = \newrad$.  Now observe that
\begin{align*}
\big| \hilin{g}{(\GamOpHat - \GamOp)g} \big| \; = \;
\frac{\newrad^2}{\hilnorm{f}^2} \, \big| \hilin{f}{(\GamOpHat - \GamOp)f}
\big| & \geq \frac{\newrad^2}{\hilnorm{f}^2} \, \Big \{ \ucritsq \max
\big \{ 1, \tfrac{\hilnorm{f}^2}{\newrad^2} \big \} +
\frac{\specfun^2(f)}{2} \Big \} \nonumber \\
%
%\label{EqnLower}
& \geq \ucritsq +  \frac{\specfun^2(g)}{2},
\end{align*}
as claimed.

We now claim that that there must exist some function $h$ with
$\hilnorm{h} \leq \newrad$ and $\specfun(h) \leq \ucrit$ such that
$\big|\hilin{h}{(\GamOpHat - \GamOp) h}\big| \geq \frac{\ucritsq}{2}$.
Indeed, if the $g$ constructed above satisfies $\specfun(g) \leq
\ucrit$, then this function has the desired property.  Otherwise, we
may assume that $\specfun(g) > \ucrit$, and define $h =
\tfrac{\ucrit}{\specfun(g)} g \in \RKHS$. Observe that
$\hilnorm{h} \leq \hilnorm{g} = \newrad$, and $\specfun(h) = \ucrit$
by construction.  Moreover, since $g$ satisfies the lower
bound~\eqref{EqnBevent1}, we have
\begin{align*}
\big| \hilin{h}{(\GamOpHat - \GamOp)h} \big| \; = \;
\frac{\ucritsq}{\specfun^2(g)} \, \big| \hilin{g}{(\GamOpHat -
  \GamOp)g} \big| & \geq \frac{\ucritsq}{\specfun^2(g)} \, \Big \{
\ucritsq + \frac{\specfun^2(g)}{2} \Big \} \\
& \geq \frac{\ucritsq}{2}.
\end{align*}
Putting together the pieces, we have established that the event
$\Bevent$ is contained within the event $\big\{\Ztil_\numobs(\ucrit) \geq
\tfrac{\ucritsq}{2} \big\}$, as claimed.

%%%%%%%%%%%%%%%%%%%%%%%%%%%%%%%%%%%%%%%%%%%%%%%%%%%%%%%%%%%%%%%%%%%%%%%%%%%%%%%

\subsubsection{Proof of Lemma~\ref{lemma:Ztiltail}}
\label{sec:proof:lemma:Ztiltail}

As usual, we proceed by first bounding the mean
$\Exp[\Ztil_\numobs(\ucrit)]$, and then establishing concentration
around this mean.  Since $\ucrit \geq \tcrit$, by standard
properties of Rademacher complexities, we have  \begin{align}
    \label{EqnZtilMean}
    \Exp[\Ztil_\numobs(\ucrit)] & \leq \ucrit \, \frac{\tcrit}{8} \; \leq
    \frac{\ucritsq}{8}.
\end{align}
Consequently, in order to complete the proof, it suffices to show that
\begin{align}
\label{EqnZtilConc}    
\Prob\Big[\Ztil_\numobs(\ucrit) \geq 2 \, \Exp[\Ztil_\numobs(\ucrit)] +
  \tfrac{\ucritsq}{4}\Big] & \; \leq \; \exp \big(-\plaincon_1 \numobs \, \tfrac{\ucritsq}{\bou^2 \newrad^2} \big).
  \end{align}

Recall from the proof of \Cref{lemma:tcrittwo} that the random
variable $\Ztil_\numobs(\ucrit)$ is the supremum of an empirical
process defined by the random variables $\statenew = (\state, \statetwo)$ and functions of
the form \mbox{$g(\statenew) = f^2(\state) - \discount f(\state) f(\statetwo)$.}  In order to
apply Talagrand's concentration inequality, we need to bound
$\|g\|_\infty$ and $\Exp[g^2(\Statenew)]$ uniformly over the class.  We have
\begin{align*}
\|g\|_\infty \leq (1 + \discount) \, \|f\|_\infty^2 \; \leq 2 \bou^2
\newrad^2
\end{align*}
where the final inequality uses the facts that $\discount \leq 1$, and
$\|f\|_\infty \leq \bou \newrad$ for any function with $\hilnorm{f}
\leq \newrad$.  On the other hand, again using the fact that
$\|f\|_\infty \leq \bou \newrad$, we have
\begin{align*}
\Exp[g^2(\Statenew)] & = \Exp \Big[ f^2(\State) \big(f(\State) - \discount f(\Statetwo) \big)^2
  \Big] \\
  & \leq \bou^2 \newrad^2 \; \Exp \Big[f^2(\State) + \discount^2 f^2(\Statetwo) - 2
  \discount f(\State) f(\Statetwo) \Big] \\
 & \stackrel{(i)}{\leq} 2 \bou^2 \newrad^2 \; \underbrace{\Exp
  \Big[f^2(\State) - \discount f(\State) f(\Statetwo) \Big]}_{\specfun^2(f)} \;
\stackrel{(ii)}{\leq} \; 2 \bou^2 \newrad^2 \ucritsq,
\end{align*}
where inequality (i) uses the fact that $\Exp[f^2(\Statetwo)] = \Exp[f^2(\State)]$
and $\discount \leq 1$; and inequality (ii) uses the fact that
$\specfun^2(f) \leq \ucritsq$ for all functions $f$ in the relevant
class.

Consequently, by applying Talagrand's theorem (cf. equation (3.86) in
the book~\cite{wainwright2019high}), there are universal constants
$\plaincon_2, \plaincon_3$ such that
\begin{align*}
  \Prob \Big[\Ztil_\numobs(\ucrit) \geq 2 \, \Exp[\Ztil_\numobs(\ucrit)]
    + \plaincon_2 \bou \newrad \ucrit \sqrt{s} + \plaincon_3 \bou^2 \newrad^2 s
    \Big] & \leq \exp(-\numobs s).
\end{align*}
Setting $s = \plaincon_1 \tfrac{\ucritsq}{\bou^2 \newrad^2}$ for a
sufficiently small constant $\plaincon_1$ yields the claim.

%%%%%%%%%%%%%%%%%%%%%%%%%%%%%%%%%%%%%%%%%%%%%%%%%%%%%%%%%%%%%%%%%%%%%%%%%%%%%%%%%%

\section{Auxiliary results for Theorem~\ref{thm:lb}}
\label{sec:proof_tablb}

This appendix is devoted to various auxiliary results associated with
\Cref{thm:lb}.  In \Cref{AppLower}, we discuss the sample size
requirements of the theorem, along with the conditions on the kernel
eigenvalues.  In the remaining subsections, we provide various
technical results used in the proof.  \Cref{append:lb} is devoted to
the analysis of MRP class $\MRPclassA$; whereas
\Cref{append:lb_2} concerns family $\MRPclassB$.

%%%%%%%%%%%%%%%%%%%%%%%%%%%%%%%%%%%%%%%%%%%%%%%%%%%%%%%%%%%%%%%%%

\subsection{Conditions of \Cref{thm:lb}}
\label{AppLower}

In this appendix, we discuss the requirements needed for our lower
bounds to be valid.  First, we claim that the sample size conditions
required for the lower bounds in \Cref{thm:lb}, parts (a) and (b), are
are weaker than the requirement~\eqref{EqnNbound} for the upper bounds
in \Cref{thm:ub}.  It is clear that the constraint in part (a) and the
first inequality in constraint~\eqref{eq:ncond_2} are looser than
bound~\eqref{EqnNbound}. Additionally, the second inequality in
condition~\eqref{eq:ncond_2} is easy to satisfy if the eigengap
$(1-\frac{\eig_2}{\eig_1})$ has constant order and the uniform bound
$\bou$ and the radius $\newradbar$ are not too large. For these
reasons, the lower bounds require even milder conditions on the sample
size $\numobs$. \\

The other condition in \Cref{thm:lb} is the eigengap condition
\mbox{$\min_{3 \leq j \leq \statdim} \big\{ \sqrt{\eig_{j-1}} -
  \sqrt{\eig_j} \big\} \geq \frac{\delcrit}{2\statdim}$}.  We claim
that this is condition is rather mild.  For instance, if we consider a
kernel whose eigenvalues exhibit $\alpha$-polynomial
decay~\eqref{EqnPolyDecay}---that is, say $\eig_j = c j^{-2 \alpha}$
for some constant $c > 0$ and exponent $\alpha > \tfrac{1}{2}$.  In
this case, we have
\begin{align*}
\sqrt{\eig_{j-1}} - \sqrt{\eig_j} = \sqrt{c} \, \big\{ (j-1)^{-\alpha} -
j^{-\alpha}\big\} = \sqrt{c} \, j^{-\alpha}\big\{ (1-\frac{1}{j})^{-\alpha} -
1 \big\} \geq \sqrt{c} \, j^{-\alpha} \frac{\alpha}{j} = \sqrt{\eig_j} \;
\frac{\alpha}{j} \geq \frac{\sqrt{\eig_j}}{2j}.
\end{align*}
If $j \leq \statdim$, then $\sqrt{\eig_j} \geq \delcrit$. Therefore,
$\sqrt{\eig_{j-1}} - \sqrt{\eig_j} \geq \frac{\delcrit}{2\statdim}$
for any $j \leq \statdim$ and the assumption is satisfied. \\

Finally, we comment on the critical inequality~\eqref{eq:CI_lb}, and show that in Regime B, by replacing $\{\eig_j\}_{j=1}^{\infty}$ with $\{ \eig_j(\TransOp) \}_{j=1}^{\infty}$, we would get a smaller critical radius.
We recall from definition~\eqref{eq:def_MRPclassB} of $\MRPclassB$ that any $\MRP(\TransOp, \rewardB, \discount) \in \MRPclassB$ satisfies $\eig_j(\TransOp) \leq \eig_j$ for any $j \geq 2$. Since the statistical dimension $\statdim = \max\big\{ j \mid \eig_j \geq \delcritsq \big\}$, we have $\min\big\{ \eig_j(\TransOp), \delcritsq \big\} \leq \min \big\{ \eig_j, \delcritsq \big\}$ for any $j \in \Int_+$ as long as $\statdim \geq 2$.
Recall that $\delcrit$ is the smallest positive solution to inequality~\eqref{eq:CI_lb}, therefore,
\begin{align*}
\sqrt{\sum_{j=1}^{\infty} \min\big\{ \frac{\eig_j(\TransOp)}{\delcritsq}, 1 \big\}} \leq \sqrt{\sum_{j=1}^{\infty} \min\big\{ \frac{\eig_j}{\delcritsq}, 1 \big\}} \leq \sqrt{\numobs} \; \frac{\newradbar \, \SqDis}{2 \stdbar} \, \delcrit \;.
\end{align*}
In other words, $\delcrit$ satisfies the critical inequality defined by $\{ \eig_j(\TransOp) \}_{j=1}^{\infty}$. Hence, $\delcrit \geq \delcrit(\TransOp)$, where $\delcrit(\TransOp)$ is the critical radius induced by $\{ \eig_j(\TransOp) \}_{j=1}^{\infty}$. In this way, \ref{EqnLB} further implies another lower bound
$\norm{\thetahat - \thetastar}_{\distr(\TransOp)}^2 \geq \plaincon_1 \, \newradbar^2 \delcritsq(\TransOp)$, as claimed in \Cref{sec:def_MRPclass}.

%%%%%%%%%%%%%%%%%%%%%%%%%%%%%%%%%%%%%%%%%%%%%%%%%%%%%%%%%%%%%%%%

\subsection{Proofs of auxiliary results in Regime A}
\label{append:lb}

In this part, \Cref{append:lb_welldefn} presents the proof of
\Cref{lemma:welldefn}, which shows the well-definedness of our MRP
instances and verifies that they belong to the model family
$\MRPclassA$. \Cref{append:lb_KL} is devoted to the proof of
\Cref{lemma:lb_KL}, which provides an upper bound on the pairwise KL
distances in our construction.  On the other hand,
\Cref{append:lb_gap} provides the proof of \Cref{lemma:lb_gap}, which
lower bounds the pairwise distances between the value functions in our
model family.

%%%%%%%%%%%%%%%%%%%%%%%%%%%%%%%%%%%%%%%%%%%%%%%%%%%%%%%%%%%%%%%%%%%%%%%

\subsubsection{Proof of Lemma~\ref{lemma:welldefn}}
\label{append:lb_welldefn}

We verify the conditions in the definition~\eqref{eq:def_MRPclassA} of family
$\MRPclassA$. In order that our constructed MRP instances
$\MRP_m \in \MRPclassA$ for any $m \in [\PackNum]$, we check the constraints in equation~\eqref{eq:def_MRPclassA} one by one. We first note that $\thetastar_m \in \RKHSA$ by our construction. As for
condition (ii) in definition~\eqref{eq:def_MRPclassA}, we recall that all models $\{ \MRP_m
\}_{m=1}^{\PackNum} \subset \MRPclassA$ have Lebesgue measure $\distr$ as the common
stationary distribution, thus the covariance
operator $\CovOp$ has eigenpairs $\{ (\eig_j, \baseA{j})
\}_{j=1}^{\infty}$, where $\{ \eig_j \}_{j=1}^{\infty}$ are the pre-specified parameters and $\{\baseA{j}\}_{j=1}^{\infty}$ are the bases of $\RKHSA$ defined in equation~\eqref{eq:phi_r}. Since \mbox{$\sup_{j \in \Int_+}
  \supnorm{\baseA{j}} = 1 \leq \unibou$}, condition (ii) is satisfied.
In the sequel, we
only need to verify inequalities~\eqref{EqnCondB}.  Specifically, we will prove that for
each $\MRP_m$, the following properties hold:
\begin{itemize}
\item The Bellman residual variance satisfies $\stdfun^2(\thetastar_m)
  \leq \stdbar^2$, and;
\item The norms satisfy $\max\big\{ \norm{\thetastar_m -
  \rewardA}_{\RKHSA}, \tfrac{2\supnorm{\thetastar_m}}{\bou}
  \big\} \leq \newradbar$.
\end{itemize}~

Before proving the two claims above, we first develop upper bounds on
$\supnorm{f_m}$ and $\big|\Deltap_m^{(k)}\big|$, which are crucial in
our estimations below. We claim that
\begin{align}
\label{EqnClaimOne}
\supnorm{f_m} \leq \frac{\parap}{9} \qquad \text{and} \qquad
\big|\Deltap_m^{(k)}\big| \leq \frac{\parap}{8} \qquad \text{for any
  $k \in [K]$ and $m \in [\PackNum]$. }
\end{align}
In fact, by using the definition of $f_m$ in
equation~\eqref{eq:def_thetam} and the fact that
$\supnorm{\baseA{j}} \leq \unibou$, we find that
\begin{align*}
  \supnorm{f_m} \leq \sqrt{\frac{\parap \, (1-\parap)}{120 \,
      \numobs}} \; \sum_{j=2}^{\statdim} \supnorm{\baseA{j}}
  \leq \unibou \statdim \sqrt{\frac{\parap}{120 \, \numobs}} \; .
\end{align*}
The critical inequality~\eqref{eq:CI_lb} ensures $\statdim \leq
\numobs \, \big\{ \frac{\newradbarreward \delcrit
  \SqDis}{\unibou\stdbarreward} \big\}^2$, and therefore
\begin{align*}
  \supnorm{f_m} \leq \unibou \numobs \, \Big\{ \frac{\newradbarreward
    \delcrit \SqDis}{\unibou\stdbarreward} \Big\}^2
  \sqrt{\frac{\parap}{120 \, \numobs}} \overset{(i)}{\leq}
  \sqrt{\tfrac{1}{30} \; \parap \, \SqDis} \leq \frac{\parap}{9},
\end{align*}
where we have used condition~\eqref{eq:ncond} in the step (i).  We
plug the inequality~$\supnorm{f_m} \leq \frac{\parap}{9}$ into the
definition of $\Deltap_m^{(k)}$ in equation~\eqref{eq:def_tau}. It
follows that
\begin{align*}
  \big|\Deltap_m^{(k)}\big| & = \frac{1-\discount + 2
    \discount\parap}{1-\discount+2\discount\parap-2\discount f_m(\state_k)}
  \; |f_m(\state_k)| \leq \frac{1-\discount + 2
    \discount\parap}{1-\discount+2\discount\parap- 2\discount
    (\parap/9)} \, (\parap/9) \leq \frac{\parap}{8}.
\end{align*}

%%%%%%%%%%%%%%%%%%%%%%%%%%%%%%%%%%%%%%%%%%%%%%%%%%%%%%%%%%%%%%%%%%%%%%%%%%%%%%%%%%%%%%%%%%%%%%%%%%%%

\paragraph{Upper bound on $\stdfun^2(\thetastar_m)$:}

We consider the condition $\stdfun(\thetastar_m) \leq \stdbar$. Recall that
the MRP $\MRP_m$ consists of $\numint$ local models, each is
determined by the transition matrix $\bP_m^{(k)} = \bPr\big(
\parap, \Deltap_m^{(k)} \big)$ and reward vector $\br = [1,
  -1]^{\top}$. The Bellman residual variance $\stdfun^2(\thetastar_m)$ of the
full-scale MRP $\MRP_m$ is the average of those of the small local
MRPs. Let $\btheta_m^{(k)}$ be the value function associated with the
$k$-th local MRP. We use some algebra and find that
\begin{align*}
\stdfun^2\big(\btheta_m^{(k)}\big) = \frac{4 \discount^2 \big(\parap +
  \Deltap_m^{(k)}\big) \big(1 - \parap - \Deltap_m^{(k)}
  \big)}{\big(1-\discount+2\discount\parap + 2\discount
  \Deltap_m^{(k)}\big)^2}\,.
\end{align*}
Since $\big|\Deltap_m^{(k)}\big| \leq \parap/8$ and $\parap = \tfrac{3
  \SqDis}{\discount}$, we have
\begin{align*}
  \stdfun^2\big(\btheta_m^{(k)}\big) \leq \frac{(4\discount^2)
    \; (\parap+\parap/8) \, (1-\parap -
    \parap/8)}{(1-\discount+2\discount\parap - 2\discount
    (\parap/8))^2} \leq \frac{1+\discount}{5\SqDis} \leq
  \stdbar^2.
\end{align*}
By taking the average of $\stdfun^2\big(\btheta_m^{(k)}\big)$ over
indices $k \in [\numint]$, we conclude that
\begin{align*}
  \stdfun^2(\thetastar_m) = \frac{1}{\numint} \sum_{k=1}^\numint
  \stdfun^2\big(\btheta_m^{(k)}\big) \leq \stdbar^2.
\end{align*}

%%%%%%%%%%%%%%%%%%%%%%%%%%%%%%%%%%%%%%%%%%%%%%%%%%%%%%%%%%%%%%%%%%%%%%%%%%%%%%

\paragraph{Upper bounds on $\norm{\thetastar_m - \rewardA}_{\RKHSA}$
  and $\norm{\thetastar_m}_{\infty}$:}

We first consider the RKHS norm $\norm{\thetastar_m -
  \rewardA}_{\RKHSA}$.  Recall from
equations~\eqref{eq:def_reward}~and~\eqref{eq:def_thetam} that the
reward and value functions $\rewardA$ and $\thetastar_m$ are
\begin{align*}
\rewardA = \baseA{1} \qquad \text{and} \qquad \thetastar_m =
\frac{1}{1 - \discount + 2\discount \parap} \; \baseA{1} -
\frac{2\discount}{(1-\discount+2\discount\parap)^2} \sqrt{\frac{\parap
    \, (1-\parap)}{120 \, \numobs}} \; \sum_{j=2}^{\statdim}
\alpha_m^{(j-1)} \, \baseA{j}.
\end{align*}
We take a shorthand $\eta \defn \frac{2\discount}{(1-\discount +
  2\discount\parap)^2} \sqrt{\frac{\parap \, (1-\parap)}{120}} $.
Note that $\{ \sqrt{\eig_j} \; \baseA{j} \}_{j=1}^{\infty}$ is
an orthonormal basis in RKHS $\RKHSA$ and $\delcrit^2 \leq
\eig_j$ for any $j \in [\statdim]$; as a consequence, we have
\begin{align}
  \label{eq:RKHSnorm}
  \norm{\thetastar_m - \rewardA}_{\RKHSA}^2 & = \Big\{
  \frac{\discount \, (1-2\parap)}{1-\discount+2\discount\parap}
  \Big\}^2 \, \frac{1}{\eig_1} + \frac{\eta^2}{\numobs} \;
  \sum_{j=2}^{\statdim} \frac{(\alpha_m^{(j-1)})^2}{\eig_j} \leq
  \Big\{ \frac{\discount \, (1-2\parap)}{1-\discount+2\discount\parap}
  \Big\}^2 \, \frac{1}{\eig_1} + \frac{\eta^2 \,\statdim}{\numobs
    \delcritsq}.
\end{align}
Since $\parap = \frac{3\SqDis}{\discount}$, the first term in the
upper bound above satisfies
\begin{subequations}
  \label{eq:RKHS'}
\begin{align}
 \Big \{ \frac{\discount \, (1-2\parap)}{1-\discount+2\discount\parap}
 \Big \}^2 \, \frac{1}{\eig_1} = \Big \{ \frac{7\discount -
   6}{7\SqDis} \Big \}^2 \, \frac{1}{\eig_1} \leq \Big \{
 \frac{\discount}{7\SqDis \sqrt{\eig_1}} \Big \}^2 \leq \Big \{
 \frac{6 \newradbar}{7} \Big \}^2,
\end{align}
where we have used the relation $\newradbar \geq
\tfrac{\discount}{6\SqDis\sqrt{\eig_1}}$.  As for the second term in
the right hand side of inequality~\eqref{eq:RKHSnorm}, we recall that
the critical inequality~\eqref{eq:CI_lb} ensures
$\frac{\statdim}{\numobs\delcritsq} \leq
\big\{\frac{\newradbar\SqDis}{\stdbar}\big\}^2$, therefore,
\begin{align*}
\frac{\eta^2 \,\statdim}{\numobs \delcritsq} \leq \newradbar^2 \;
\big\{ \eta \SqDis/\stdbar \big\}^2.
\end{align*}
%We now control the term $\eta \SqDis/\stdbar $. 
Combining the
definition of $\eta$, the equality $\parap =
\frac{3\SqDis}{\discount}$ and the relation $ \stdbarreward^2 \geq
\tfrac{1+\discount}{5\SqDis}$, we find that $\eta \SqDis/\stdbar \leq
\frac{1}{98}$. It follows that
\begin{align*}
\frac{\eta^2 \,\statdim}{\numobs \delcritsq} \leq \big\{
\frac{\newradbar}{98} \big\}^2.
\end{align*}
\end{subequations}
Plugging inequalities~\eqref{eq:RKHS'}~into~\eqref{eq:RKHSnorm} yields
$\norm{\thetastar_m - \rewardA}_{\RKHSA} \leq \newradbar$.

We now estimate the sup-norm $\supnorm{\thetastar_m}$. We
use the inequality $\supnorm{f_m} \leq \frac{\parap}{9}$ and find that
\begin{align*}
  \supnorm{\thetastar_m} & \leq
  \frac{1}{1-\discount+2\discount\parap} \; \supnorm{\baseA{1}}
  + \frac{2\discount}{(1-\discount+2\discount\parap)^2} \;
  \supnorm{f_m} \\
& \leq \frac{1}{1-\discount+2\discount\parap} +
  \frac{2\discount}{(1-\discount+2\discount\parap)^2} \; (\parap/9)
  \leq \frac{1}{6\SqDis}.
\end{align*}
Since $\newradbar \geq \tfrac{2}{3\bou\SqDis}$, we have
$\tfrac{2\supnorm{\thetastar_m}}{\bou} \leq \newradbar$.

Integrating the two parts, we conclude that $\max\big\{
\norm{\thetastar_m - \rewardA}_{\RKHSA},
\tfrac{2\supnorm{\thetastar_m}}{\bou} \big\} \leq \newradbar$.

%%%%%%%%%%%%%%%%%%%%%%%%%%%%%%%%%%%%%%%%%%%%%%%%%%%%%%%%%%%%%%%%%%%%%%%

\subsubsection{Proof of Lemma~\ref{lemma:lb_KL}}
\label{append:lb_KL}

Since the $\numobs$ samples $\{ (\state_i, \statetwo_i) \}_{i=1}^n$ are i.i.d., we
have
\begin{align*}
\kull[\big]{\TransOp_{m'}^{1:\numobs}}{\TransOp_m^{1:\numobs}} =
\numobs \; \kull{\TransOp_{m'}}{\TransOp_m}.
\end{align*}
Thus, the remainder of our proof focuses on bounding
$\kull{\TransOp_{\mprime}}{\TransOp_m}$, for an arbitrary pair $m,
\mprime \in [\PackNum]$.

From Jensen's inequality and the concavity of the logarithm, the KL
divergence can be upper bounded by the $\chi^2$-divergence---that is
\begin{align*}
\kull{\TransOp_{m'}}{\TransOp_m} & \leq \chi^2
\kulldiv{\TransOp_{m'}}{\TransOp_m} = \int_{\StateSp^2} \distr(\state) \;
\frac{\big(\TransOp_{m'}(\statetwo \mid \state)-\TransOp_m(\statetwo \mid \state)\big)^2}{\TransOp_m(\statetwo \mid \state)} \; \dx \, \dx '.
\end{align*}
Recall our shorthand notation $\numint = 2^{\lceil \log_2 \statdim \rceil}$,
where the kernel dimension $\statdim$ was previously defined as $\statdim = \max\big\{ j \mid \eig_j \geq \delcritsq \big\}$.  Since our construction of transition model $\TransOp_m$ is an ensemble of $\numint$ blocks $\{ \bP_m^{(k)} \}_{k=1}^{\numint}$, each involving
two states, the $\chi^2$-divergence can be written as the sum
\begin{align}
\label{eq:fullchisq}
\chi^2 \kulldiv{\TransOp_{m'}}{\TransOp_m} = \frac{1}{\numint}
\sum_{k=1}^{\numint} \chi^2 \kulldiv[\big]{\bP_{m'}^{(k)}}{\bP_m^{(k)}}.
\end{align}
The local $\chi^2$-divergence is defined as
\begin{align*}
\chi^2 \kulldiv[\big]{\bP_{m'}^{(k)}}{\bP_m^{(k)}} \defn \sum_{\state,\statetwo
  \in \{ \xplus, \xneg \}} \bmu(\state) \; \frac{\big( \bP_{m'}^{(k)}(\statetwo \mid \state) - \bP_m^{(k)}(\statetwo \mid \state) \big)^2}{\bP_m^{(k)}(\statetwo \mid \state)}
\end{align*}
where $\bmu \defn \big[\tfrac{1}{2}, \tfrac{1}{2}\big]$ is the
stationary distribution.  We recall from equation~\eqref{eq:def_2stateMRP} the expression of local model $\bP_m^{(k)} =
\bPr\big( \parap, \Deltap_m^{(k)} \big)$ and derive that
\begin{align}
\label{EqnChiDecompose}
\chi^2 \kulldiv[\big]{\bP_{m'}^{(k)}}{\bP_m^{(k)}} \; = \; \frac{\big(
  \Deltap_{m'}^{(k)} - \Deltap_m^{(k)} \big)^2}{\big( \parap +
  \Deltap_m^{(k)} \big)\big( 1 - \parap - \Deltap_m^{(k)} \big)} \; .
\end{align}
We develop upper and lower bounds on the numerator and denominator
separately.

We first consider the numerator $\big( \Deltap_{m'}^{(k)} -
\Deltap_m^{(k)} \big)^2$.  Following some algebra, we find that
\begin{align*}
\Deltap_{m'}^{(k)} - \Deltap_m^{(k)} & = \frac{f_m(\state_k) -
  f_{m'}(\state_k)}{\big(1-\frac{2\discount \;
    f_{m'}(\state_k)}{1-\discount+2\discount\parap}\big)\big(1-\frac{2\discount
    \; f_m(\state_k)}{1-\discount+2\discount\parap}\big)} \; ,
\end{align*}
where $\state_k$ is any point in interval $\Delta_+^{(k)}$.  It was shown
in the bound~\eqref{EqnClaimOne} that $\supnorm{f_m} \leq \parap/9$,
therefore, $\min\big\{ 1 - \frac{2\discount \;
  f_{m'}(\state_k)}{1-\discount+2\discount\parap}, 1 - \frac{2\discount \;
  f_{m}(\state_k)}{1-\discount+2\discount\parap} \big\} \geq
\frac{19}{21}$. It follows that
\begin{subequations}
\label{eq:localchisq'}
\begin{align}
\big(\Deltap_{m'}^{(k)} - \Deltap_m^{(k)}\big)^2 \leq 2 \; \big(
f_m(\state_k) - f_{m'}(\state_k) \big)^2.
\end{align}
As for the numerator $\big( \parap + \Deltap_m^{(k)} \big)\big( 1 - \parap - \Delta_m^{(k)}
\big)$ in the right hand side of
equality~\eqref{EqnChiDecompose}, we have proved in the bound~\eqref{EqnClaimOne} that $\big|\Deltap_m^{(k)}\big| \leq \parap/8$, so that
\begin{align}
\big( \parap + \Deltap_m^{(k)} \big)\big( 1 - \parap - \Delta_m^{(k)}
\big) \geq \tfrac{7}{8} \; \parap (1-\parap).
\end{align}
\end{subequations}
Combining inequalities~\eqref{eq:localchisq'} with
equation~\eqref{EqnChiDecompose} yields
\begin{align}
  \label{eq:localchisq}
  \chi^2 \kulldiv[\big]{\bP_{m'}^{(k)}}{\bP_m^{(k)}} \leq \frac{3 \,
    \big( f_{m'}(\state_k) - f_m(\state_k) \big)^2}{\parap ( 1 - \parap )} \; .
\end{align}

We plug the bound~\eqref{eq:localchisq} into
equation~\eqref{eq:fullchisq} and find that
\begin{align*}
  \chi^2 \kulldiv{\TransOp_{m'}}{\TransOp_m} & \leq \frac{3}{\parap (
    1 - \parap )} \, \bigg\{ \frac{1}{K} \sum_{k=1}^{K} \big(
  f_{m'}(\state_k) - f_m(\state_k) \big)^2 \bigg\} \\
  & \overset{(i)}{=} \frac{3}{\parap(1-\parap)} \int_{\StateSp} \big(
  f_{m'}(\state) - f_m(\state) \big)^2 \dx = \frac{3}{\parap(1-\parap)} \;
  \munorm{f_{m'} - f_m}^2.
\end{align*}
Here step (i) is due to the property that $f_m(\state) = f_m(\state_k)$ for any
$\state \in \interval_k^+$ and $f_m(\state) = -f_m(\state_k)$ for any $\state \in
\interval_k^-$.  Regarding the $\Lmu$-distance $\munorm{f_m -
  f_{m'}}$, we leverage the orthonormality of basis functions $\{
\baseA{j} \}_{j=1}^{\statdim}$ in $\Lmu$ and find that
\begin{align*}
  \munorm{f_{m'} - f_m}^2 = \frac{\parap(1-\parap)}{120\numobs}
  \sum_{j=2}^{\statdim} \big( \alpha_m^{(j-1)} - \alpha_{m'}^{(j-1)}
  \big)^2 \leq \parap(1-\parap) \, \frac{\statdim}{120\numobs}.
\end{align*}
Therefore, we have
\begin{align*}
\chi^2 \kulldiv{\TransOp_{m'}}{\TransOp_m} \leq \frac{\statdim}{40
  \numobs}.
\end{align*}

Putting together the pieces yields
\begin{align*}
  \kull[\big]{\TransOp_{m'}^{1:\numobs}}{\TransOp_m^{1:\numobs}} \leq
  \numobs \; \chi^2 \kulldiv{\TransOp_{m'}}{\TransOp_m} \leq
  \frac{\statdim}{40},
\end{align*}
as claimed in the lemma statement.

%%%%%%%%%%%%%%%%%%%%%%%%%%%%%%%%%%%%%%%%%%%%%%%%%%%%%%%%%%%%%%%%%%%%%%%%%

\subsubsection{Proof of Lemma~\ref{lemma:lb_gap}}
\label{append:lb_gap}

We now lower bound the $\Lmu$-norm between the value functions of
different models in our family.  Recall the expression of value
function $\thetastar_m$ in equation~\eqref{eq:def_thetam}.  We find
that
\begin{align*}
  \thetastar_m = \thetastar_0 + \frac{\eta}{\sqrt{\numobs}}
  \sum_{j=2}^{\statdim} \alpha_m^{(j-1)} \, \baseA{j}
\end{align*}
where $\eta = \frac{2\discount}{(1-\discount + 2\discount\parap)^2}
\sqrt{\frac{\parap \, (1-\parap)}{120}} $. Since $\{ \baseA{j}
\}_{j=1}^{\infty}$ is an orthonormal basis in $\Lmu$, we can write
\begin{align*}
  \munorm{\thetastar_{m'} - \thetastar_m}^2 & =
  \frac{\eta^2}{\numobs} \sum_{j=2}^{\statdim} \big(
  \alpha_{m'}^{(j-1)} - \alpha_{m}^{(j-1)} \big)^2.
\end{align*}
By our construction, $\big\{ \alphabold_m \big\}_{m=1}^{\PackNum}$ is
a $\tfrac{1}{4}$-packing of the Boolean hypercube
$\{0,1\}^{\statdim-1}$ with respect to the rescaled Hamming distance,
therefore,
\begin{align*}
\sum_{j=2}^{\statdim} \big( \alpha_{m'}^{(j-1)} - \alpha_{m}^{(j-1)}
\big)^2 \geq \frac{\statdim-1}{4}.
\end{align*}
We use the conditions $\parap = \frac{3\SqDis}{\discount}$, $\discount
\in [0.9,1)$ and $\stdbar^2 \leq \frac{1 + \discount}{1-\discount}$,
  and find by some algebra that
\begin{align*}
\eta \geq \frac{\plaincon_2'}{1-\discount}
\sqrt{\frac{1+\discount}{1-\discount}} \geq \frac{\plaincon_2' \,
  \stdbar}{1-\discount}
\end{align*}
where $\plaincon_2' > 0$ is a universal constant. Combining the
inequalities, we obtain
\begin{align*}
\munorm{\thetastar_{m'} - \thetastar_m} \geq \frac{\plaincon_1' \, \stdbar}{1-\discount} \sqrt{\frac{\statdim}{\numobs}}
\end{align*}
for another univeral constant $\plaincon_1' > 0$.  By further using
the regularity condition~\eqref{eq:kernel_reg}, we can derive
inequality~\eqref{eq:gap>=} in the lemma statement.

%%%%%%%%%%%%%%%%%%%%%%%%%%%%%%%%%%%%%%%%%%%%%%%%%%%%%%%%%%%%%%%%%%%%%%%%%%%

\subsection{Proofs of auxiliary results in Regime B}
\label{append:lb_2}

This section contains proofs of auxiliary results that underlie the
minimax lower bound over model family
$\MRPclassB$. Specifically, \Cref{append:density_2} proves the
density ratio condition~\eqref{eq:density}, that is, $\frac{\diff
  \distrm{m}}{\diff \distr}(\state) \geq
\frac{1}{2}$. \Cref{append:lb_welldefn_2} is devoted to the proof of
\Cref{lemma:welldefn}, which shows that our constructed models $\{ \MRP_m \}_{m=1}^{\PackNum}$ belong
to the family $\MRPclassB$. \Cref{append:lb_KL_2} proves
\Cref{lemma:lb_KL}, which upper bounds the pairwise
KL-divergence. \Cref{append:lb_gap_2} presents the proof of
\Cref{lemma:lb_gap}, which estimates the pairwise distance in value
functions.

%%%%%%%%%%%%%%%%%%%%%%%%%%%%%%%%%%%%%%%%%%%%%%%%%%%%%%%%%%%%%%%%%%%%%%%%%%%

\subsubsection{Proof of density ratio condition}
\label{append:density_2}

We prove the density ratio condition 
\eqref{eq:density}, i.e. $\frac{\diff
	\distrm{m}}{\diff \distr}(\state) \geq
\frac{1}{2}$.
Recall our definition of $f_m$ in equation~\eqref{eq:def_fm_2}. Since
$\sup_{j \geq 1} \supnorm{\baseB{j}} \leq \unibou$, we have
\begin{align}
\label{EqnClaimOne_2}
\supnorm{f_m} \leq \frac{\unibou \; \parap \; \statdim}{25
  \sqrt{\numobs}} \overset{(i)}{\leq} \numobs \; \Big\{
\frac{\newradbar\delcrit\SqDis}{\unibou\stdbar} \Big\}^2 \,
\frac{\unibou \; \parap}{25 \sqrt{\numobs}} \overset{(ii)}{\leq}
\frac{\parap \, \SqDis}{2}.
\end{align}
Here step (i) is due to the critical inequality~\eqref{eq:CI_lb}
and in step (ii) we have used the inequality $\newradbarreward^2 \delcrit^2 \leq \frac{12
  \; \unibou \stdbarreward^2}{\SqDis \sqrt{\numobs}}$ in
condition~\eqref{eq:ncond_2}. We plug inequality~\eqref{EqnClaimOne_2} into the
expression of stationary distribution $\distrm{m}$ in
equation~\eqref{eq:def_mum}. It follows that
$\frac{\diff \distrm{m}}{\diff \distr}(\state) \geq \frac{1}{2}$, as
claimed.

%%%%%%%%%%%%%%%%%%%%%%%%%%%%%%%%%%%%%%%%%%%%%%%%%%%%%%%%%%%%%%%%%%%%%%%%%%%

\subsubsection{Proof of Lemma~\ref{lemma:welldefn}}
\label{append:lb_welldefn_2}

By our construction, condition $\thetastar_m \in \RKHSB$
in equation~\eqref{eq:def_MRPclassB} naturally holds. In the sequel, we verify
the remaining constraints in the definition of $\MRPclassB$, including
\begin{itemize}
  \itemsep = -.2em
\item the regularity condition $\discount
  \norm{\thetastar_m}_{\distrm{m}} \leq 1$ and the Bellman residual variance bound
  $\stdfun^2(\thetastar_m) \leq \stdbar^2$,
\item the norm condition $\max \big\{ \norm{\thetastar_m -
  \reward}_{\RKHSB}, \tfrac{2 \supnorm{\thetastar_m}}{\bou}
  \big\} \leq \newradbarreward$,
\item the property that the covariance operator $\CovOp(\TransOp_m)$ has eigenpairs $\big\{ \big(\eig_j(\TransOp_m),
  \base_j(\TransOp_m)\big) \big\}_{j=1}^{\infty}$ with $\eig_j(\TransOp_m) \leq \eig_j$
  for $j \geq 2$ and $\sup_j \supnorm[\big]{\base_j(\TransOp_m)} \leq 2 =
  \unibou$.
\end{itemize}

\paragraph{Upper bounds on $\discount \norm{\thetastar_m}_{\distrm{m}}$ and $\stdfun^2(\thetastar_m)$:}

The MRP $\MRP_m$ consists of $\numint$ blocks, each is a small local MRP
determined by the transition matrix $\bP_m^{(k)} = \bPt\big( \parap,
\Deltap_m^{(k)} \big)$ and a reward vector $\br = \big\{ \parap +
\tfrac{1-\discount}{2\discount} \big\} \, [1, -1]^{\top}$. The
stationary distribution of $\bP_m^{(k)}$ takes the form
\begin{align}
\label{eq:bmu}
\bmu_m^{(k)} = \big[ \tfrac{1}{2} +
  \tfrac{\Deltap_m^{(k)}}{2\parap}, \tfrac{1}{2} -
  \tfrac{\Deltap_m^{(k)}}{2\parap} \big].
\end{align}
The $\bmu_m^{(k)}$-weighted norm of value
function $\btheta_m^{(k)}$ and the variance term
$\stdfun^2(\btheta_m^{(k)})$ satisfy
\begin{align*}
  \discount^2
  \norm[\big]{\btheta_m^{(k)}}_{\bmu_m^{(k)}}^2 \; = \;
  \frac{1}{4} + \frac{\discount \;
    (1-\discount+\discount\parap)}{\parap \; \SqDis^2} \, \big(
  \Deltap_m^{(k)} \big)^2 ~~ \text{and} ~~ \stdfun^2(\btheta_m^{(k)})
  = \frac{1-\parap}{\parap} \, \big\{ \parap^2 - \big( \Deltap_m^{(k)}
  \big)^2 \big\} \leq \parap \, (1-\parap).
\end{align*}

The squared $L^2(\distrm{m})$-norm of the full-scale value function
$\thetastar_m$ is the average of
$\norm[\big]{\btheta_m^{(k)}}_{\bmu_m^{(k)}}^2$ over indices $k
\in [\numint]$. We use the relation $\Deltap_m^{(k)} = f_m(\state_k)$ and find
that
\begin{subequations}
\begin{align}
  \label{eq:reg2_1}
  \discount^2 \mumnorm{m}{\thetastar_m}^2 = \frac{1}{4} + \frac{1}{\numint}
  \sum_{k=1}^{\numint} \; \frac{\discount \; ( 1 - \discount + \discount
    \parap ) }{\parap \; \SqDis^2} \; f_m^2(\state_k) = \frac{1}{4} +
  \frac{\discount \; ( 1 - \discount + \discount \parap ) }{\parap \;
    \SqDis^2} \; \munorm{f_m}^2.
\end{align}
Due to the orthonormality of bases $\{ \baseB{j} \}_{j=2}^{\statdim}$
in $\Lmu$, we have $\munorm{f_m} = \frac{\parap}{25\sqrt{\numobs}} \,
\norm{\alphabold_m}_2 \leq \frac{\parap}{25}
\sqrt{\frac{\statdim}{\numobs}}$.  According to the critical
inequality~\eqref{eq:CI_lb}, it holds that
\begin{align}
  \label{eq:reg2_2}
\munorm{f_m} \leq \frac{\parap}{25} \sqrt{\frac{\statdim}{\numobs}}
\leq {\frac{\parap}{25}} \; \Big\{
\frac{\newradbar\delcrit\SqDis}{\unibou\stdbar} \Big\}
\overset{(i)}{\leq} \frac{2}{5} \; \parap \, \SqDis.
\end{align}
In step (i), we have used the inequality $\newradbar \delcrit \leq 10
\; \unibou\stdbar$, which is implied by condition~\eqref{eq:ncond_2}.
\end{subequations}
We plug inequality~\eqref{eq:reg2_2} into equation~\eqref{eq:reg2_1}
and conclude that $\discount \mumnorm{m}{\thetastar_m} \leq 1$. Therefore, the regularity condition in Regime B is satisfied.

Similarly, we calculate the variance term $\stdfun^2(\thetastar_m)$ by
taking the average of $\big\{ \stdfun^2(\btheta_m^{(k)})
\big\}_{k=1}^\numint$. It follows that $\stdfun^2(\thetastar_m) \leq \parap
\, (1-\parap)$. Since $\stdbar^2 \geq \frac{1}{8} = \parap$, we have
$\stdfun(\thetastar_m) \leq \stdbar$, as required by
equation~\eqref{EqnCondB}.

%%%%%%%%%%%%%%%%%%%%%%%%%%%%%%%%%%%%%%%%%%%%%%%%%%%%%%%%%%%%%%%%%%%%%%%%%%%%%%%%%%%%%%%%%%%%%%

\paragraph{Upper bounds on $\norm{\thetastar_m - \rewardB}_{\RKHSB}$ and $\norm{\thetastar_m}_{\infty}$:}

We first consider the RKHS norm $\norm{\thetastar_m -
  \rewardB}_{\RKHSB}$. Recall that the reward and value
functions $\rewardB$ and $\thetastar_m$ take the form
\begin{align*}
  \rewardB = \Big\{ \parap + \frac{1-\discount}{2\discount} \Big\} \;
  \baseB{1}, \qquad \thetastar_m = \frac{1}{2\discount} \;
  \baseB{1} + \frac{1}{1-\discount} \; f_m =
  \frac{1}{2\discount} \; \baseB{1} +
  \frac{\parap}{25\SqDis\sqrt{\numobs}} \sum_{j=2}^{\statdim}
  \alpha_m^{(j-1)} \, \baseB{j}.
\end{align*}
Since $\{ \sqrt{\eig_j} \: \baseB{j} \}_{j=1}^{\infty}$ is an
orthonormal basis of $\RKHSB$, we use the property that
$\eig_j \geq \delcritsq$ for any $j \leq \statdim$ and find that
\begin{align*}
  \norm{\thetastar_m - \rewardB}_{\RKHSB}^2 =
  \frac{(\tfrac{1}{2} - \parap)^2}{\eig_1} + \frac{\parap^2}{25^2 \;
    \SqDis^2 \; \numobs} \sum_{j=2}^{\statdim} \frac{( \alpha_m^{(j-1)}
    )^2}{\eig_j} \leq \frac{(\tfrac{1}{2} - \parap)^2}{\eig_1} +
  \frac{\parap^2 \; \statdim}{25^2 \; \SqDis^2 \; \numobs\delcritsq}.
\end{align*}
The critical inequality~\eqref{eq:CI_lb} ensures
$\frac{\statdim}{\numobs \delcritsq} \leq \big\{ \frac{\newradbar
  \SqDis}{\unibou \stdbar} \big\}^2 $ and implies
\begin{subequations} \label{eq:newradbar}
\begin{align} \label{eq:newradbar_a}
  \norm{\thetastar_m - \rewardB}_{\RKHSB}^2 \leq
  \frac{(\tfrac{1}{2} - \parap)^2}{\eig_1} +
  \frac{\parap^2}{25^2\unibou^2\stdbar^2} \newradbar^2 \leq
  \frac{9}{64\eig_1} + \frac{\newradbar^2}{25^2} \leq \newradbar^2,
\end{align}
where we have used the properties $\frac{1}{8} = \parap \leq \stdbar^2
\leq 1$, $\unibou \geq 1$ and $\frac{1}{\sqrt{\eig_1}} \leq
2\newradbar$.

As for the upper bound on sup-norm $\supnorm{\thetastar_m}$, we apply the
estimation of $\supnorm{f_m}$ in inequality~\eqref{EqnClaimOne_2} and
find that
\begin{align} \label{eq:newradbar_b}
  \supnorm{\thetastar_m} \leq \tfrac{1}{2\discount} \,
  \supnorm{\baseB{1}} + \tfrac{1}{1-\discount} \, \supnorm{f_m}
  \leq \tfrac{1}{2\discount} + \tfrac{\parap}{2} \leq
  \tfrac{1}{\discount}.
\end{align}
\end{subequations}
Therefore, it holds that $\tfrac{2 \supnorm{\thetastar_m}}{\bou} \leq
\newradbar$.

Combining inequalities~\eqref{eq:newradbar_a}~and~\eqref{eq:newradbar_b}, we conclude that $\max
\big\{ \norm{\thetastar_m - \rewardB}_{\RKHSB}, \tfrac{2
  \supnorm{\thetastar_m}}{\bou} \big\} \leq \newradbarreward$.

%%%%%%%%%%%%%%%%%%%%%%%%%%%%%%%%%%%%%%%%%%%%%%%%%%%%%%%%%%%%%%%%%%%%%%%%%%%

\paragraph{Analysis of eigenpairs $\big\{ \big(\eig_j(\TransOp_m), \base_j(\TransOp_m) \big) \big\}_{j=1}^{\infty}$:}

Recall that $\CovOp(\TransOp_m)$ is the covariance operator of kernel
$\KerB$ associated with distribution $\distrm{m} = \distr(\TransOp_m)$,
$\big\{ \eig_j(\TransOp_m) \big\}_{j=1}^{\infty}$ are the eigenvalues of
$\CovOp(\TransOp_m)$ arranged in non-increasing order, and $\base_j(\TransOp_m)$
is the eigenfunction corresponding to $\eig_j(\TransOp_m)$. In the following
\Cref{lemma:arrowhead}, we develop upper bounds on the eigenvalues and
the sup-norms of the eigenfunctions.
\begin{lemma}
	\label{lemma:arrowhead}
	Under our construction of kernel $\KerB$ and MRP instances $\{ \MRP_m \}_{m=1}^{\PackNum} \subset \MRPclassB$ in \mbox{Regime B}, for any $m \in [\PackNum]$, the eigenpairs $\big\{ \big(\eig_j(\TransOp_m),
	\base_j(\TransOp_m)\big) \big\}_{j=1}^{\infty}$ satisfy the claims below:
	\begin{itemize}
	\item[(a)] It holds that $\eig_j(\TransOp_m) \leq \eig_j$ for any $j \geq 2$.
	\item[(b)] Suppose \mbox{$\min_{3 \leq j \leq \statdim}
          \big\{ \sqrt{\eig_{j-1}} - \sqrt{\eig_j} \big\} \geq
          \frac{\delcrit}{2\statdim}$} and the sample size
          $\numobs$ is sufficiently large such that
          condition~\eqref{eq:ncond_2} holds. Then the
          eigenfunctions $\big\{ \base_j(\TransOp_m) \big\}_{j=1}^{\infty}$
          satisfy $\sup_{j \in \Int_+} \supnorm[\big]{\base_j(\TransOp_m)}
          \leq 2$.
	\end{itemize}
\end{lemma}
\noindent We establish the proof of \Cref{lemma:arrowhead} by first
connecting the eigenpairs $\big\{ \big(\eig_j(\TransOp_m), \base_j(\TransOp_m) \big)
\big\}_{j=1}^{\infty}$ to the spectrum of an arrowhead matrix, and then
developing the desired bounds based on properties of the matrix. See
\Cref{append:arrowhead} for the details.

%%%%%%%%%%%%%%%%%%%%%%%%%%%%%%%%%%%%%%%%%%%%%%%%%%%%%%%%%%%%%%%%%%%%%%%%%%%

\subsubsection{Proof of Lemma~\ref{lemma:lb_KL}}
\label{append:lb_KL_2}

Similar to the proof in \Cref{append:lb_KL}, we also upper bound the
KL-divergence
$\kull[\big]{\TransOp_{m'}^{1:\numobs}}{\TransOp_m^{1:\numobs}}$ by
the average of $\chi^2$-divergences between local models
$\bP_{m'}^{(k)}$ and $\bP_m^{(k)}$.  The calculation of local
$\chi^2$-divergence in the MRPs $\{ \MRP_m \}_{m=1}^{\PackNum} \subset
\MRPclassB$ is different from that in \Cref{append:lb_KL}, since the
stationary distributions $\bmu_m^{(k)}$ and $\bmu_{m'}^{(k)}$ (given
in equation~\eqref{eq:bmu}) are unequal.  In particular, the local
$\chi^2$-divergence takes the form
\begin{align}
\label{eq:chisq_2} \chi^2
\kulldiv[\big]{\bF_{m'}^{(k)}}{\bF_m^{(k)}} = \sum_{\state,\statetwo
  \in \{ \xplus, \xneg \}} \frac{\big( \bF_{m'}^{(k)}(\statetwo \mid
  \state) - \bF_m^{(k)}(\statetwo \mid \state)
  \big)^2}{\bF_m^{(k)}(\statetwo \mid \state)}
\end{align}
where the matrix $\bF_{\iota}^{(k)} \defn \big[
  \diag\big(\bmu_{\iota}^{(k)}\big) \big] \, \bP_{\iota}^{(k)} \in
\Real^{2 \times 2}$ for $\iota = m$ or $m'$.

We learn from inequality~\eqref{EqnClaimOne_2} that
\mbox{$\supnorm{f_{m}} \leq \frac{\parap}{2}$}, therefore,
$\big|\Deltap_m^{(k)}\big| \leq \frac{\parap}{2}$ in local Markov
chain $\bP_m^{(k)} = \bPt\big( \parap, \Deltap_m^{(k)} \big)$. It
follows that $\bmu_m^{(k)}(\state) \geq \frac{1}{4}$ and
$\bP_m^{(k)}(\statetwo \mid \state) \geq \frac{1}{2} \,
\bPbase(\statetwo \mid \state)$ for any $\state,\statetwo \in \{
\xplus, \xneg \}$. Here, $\bPbase$ is the base Markov chain defined in
equation~\eqref{eq:def_2stateMRP_0}. These lower bounds imply
that
\begin{align*}
\bF_m^{(k)}(\statetwo \mid \state) \geq \frac{1}{8} \,
\bPbase(\statetwo \mid \state) \qquad \text{for $\state,\statetwo \in
  \{ \xplus, \xneg \}$}.
\end{align*}
Substituting the above inequality into equation~\eqref{eq:chisq_2}
yields
\begin{align*}
  \chi^2 \kulldiv[\big]{\bF_{m'}^{(k)}}{\bF_m^{(k)}} \leq 8
  \sum_{\state,\statetwo \in \{ \xplus, \xneg \}} \frac{\big(
    \bF_{m'}^{(k)}(\statetwo \mid \state) -
    \bF_m^{(k)}(\statetwo \mid \state)
    \big)^2}{\bPbase(\statetwo \mid \state)}.
\end{align*}
We use some algebra and derive that
\begin{align*}
  \chi^2 \kulldiv[\big]{\bF_{m'}^{(k)}}{\bF_m^{(k)}} \leq \frac{8 \;
    \big( \Deltap_{m'}^{(k)} - \Deltap_m^{(k)}
    \big)^2}{\parap^3(1-\parap)} \: \big\{ \parap + \big(
  \Deltap_{m'}^{(k)} + \Deltap_m^{(k)} \big)^2 \big\}
  \overset{(i)}{\leq} \frac{16 \; \big( \Deltap_{m'}^{(k)} -
    \Deltap_m^{(k)} \big)^2}{\parap^2(1-\parap)}.
\end{align*}
In step (i) above, we have used the relation $\max\big\{
|\Deltap_m^{(k)}|, |\Deltap_{m'}^{(k)}| \big\} \leq \frac{\parap}{2}$
once again.

Collecting all the local $\chi^2$-divergences yields
\begin{align*}
\kull[\big]{\TransOp_{m'}^{1:\numobs}}{\TransOp_m^{1:\numobs}} \leq
\frac{\numobs}{\numint} \sum_{k=1}^\numint \chi^2
\kulldiv[\big]{\bF_{m'}^{(k)}}{\bF_m^{(k)}} \leq \frac{16 \; \numobs
}{\parap^2(1-\parap)} \; \bigg\{ \frac{1}{K} \sum_{k=1}^K \big(
\Deltap_{m'}^{(k)} - \Deltap_m^{(k)} \big)^2 \bigg\}.
\end{align*}
Recall that by our construction, $f_{\iota}(\state) =
\Deltap_{\iota}^{(k)}$ for any $\state \in \interval_+^{(k)} \cup
\interval_-^{(k)}$ and $\iota = m$ or $m'$, therefore, it holds that
\begin{align*}
\frac{1}{K} \sum_{k=1}^K \big( \Deltap_{m'}^{(k)} - \Deltap_m^{(k)}
\big)^2 = \int_{\StateSp} \big( f_{m'}(\state) - f_m(\state) \big)^2
\; \dx = \munorm{f_{m'} - f_m}^2 \; .
\end{align*}
Due to the orthogonality of basis $\{ \baseB{j} \}_{j=1}^{\infty}$ in
$\Lmu$, the definitions of $f_{m'}$ and $f_m$ in
equation~\eqref{eq:def_fm_2} imply
\begin{align}
 \label{eq:fdiff}
 \munorm{f_{m'} - f_m}^2 = \frac{\parap^2}{625 \; \numobs}
 \sum_{j = 2}^{\statdim} \big( \alpha_{m'}^{(j-1)} -
 \alpha_m^{(j-1)} \big)^2 \leq \frac{\parap^2 \; \statdim}{625
   \; \numobs}.
\end{align}
Putting together the pieces, we prove that
$\kull[\big]{\TransOp_{m'}^{1:\numobs}}{\TransOp_m^{1:\numobs}} \leq
\frac{\statdim}{40}$, as claimed.

%%%%%%%%%%%%%%%%%%%%%%%%%%%%%%%%%%%%%%%%%%%%%%%%%%%%%%%%%%%%%%%%%%%%%%%%%%%

\subsubsection{Proof of Lemma~\ref{lemma:lb_gap}}
\label{append:lb_gap_2}

Due to the definitions of $\thetastar_m$ and $\thetastar_{m'}$ in
equation~\eqref{eq:def_fm_2}, we find that
\begin{align*}
	\munorm{\thetastar_{m'} - \thetastar_m} = \frac{1}{1-\discount} \; \munorm{f_{m'} - f_m}.
\end{align*}
We recall from equation~\eqref{eq:fdiff} that the $\Lmu$-difference $\munorm{f_{m'} - f_m}$ can be expressed by vectors $\alphabold_m$ and $\alphabold_{m'}$. Using the property that $\big\{ \alphabold_m \big\}_{m=1}^{\PackNum}$ is a $\tfrac{1}{4}$-packing of the Boolean hypercube $\{0,1\}^{\statdim-1}$, we find that
\begin{align*}
	\munorm{f_{m'} - f_m}^2 = \frac{\parap^2}{625 \; \numobs} \sum_{j = 2}^{\statdim} \big( \alpha_{m'}^{(j-1)} - \alpha_m^{(j-1)} \big)^2 \geq  \frac{\parap^2}{25^2 \; \numobs} \frac{\statdim-1}{4} \, ,
\end{align*}
Plugging the lower bound on $\munorm{f_{m'} - f_m}$ into the expression of $\munorm{\thetastar_{m'} - \thetastar_m}$, we have
\begin{align*}
	\munorm{\thetastar_{m'} - \thetastar_m} \geq \frac{\parap}{50 \; \SqDis} \sqrt{\frac{\statdim-1}{\numobs}} \ .
\end{align*}
It follows from the conditions~$\stdbar \leq 1$, $\parap = \frac{1}{8}$ and $\unibou = 2$ that
\begin{align*} \munorm{\thetastar_{m'} - \thetastar_m} \geq \frac{\plaincon_1' \unibou \stdbar}{ 1-\discount} \sqrt{\frac{\statdim}{\numobs}} \end{align*}
for some universal constant $\plaincon_1' > 0$. Under the regularity condition~\eqref{eq:kernel_reg}, the above lower bound further implies
inequality~\eqref{eq:gap>=} in the lemma statement.

%%%%%%%%%%%%%%%%%%%%%%%%%%%%%%%%%%%%%%%%%%%%%%%%%%%%%%%%%%%%%%%%%%%%%%%%%%%
\section{Proof of technical lemmas}

In this appendix, we collect together various technical lemmas.

\subsection{A kernel-based computation}
\label{sec:matrix}

Here we provide an explicit expression for the kernel LSTD estimate in
terms of kernel matrices. Define the kernel covariance matrix $\CovMt
\in \Real^{\numobs \times \numobs}$ and cross-covariance matrix $\CrMt
\in \Real^{\numobs \times \numobs}$ with entries
\begin{align}
\CovMt(i,j) = \Ker(\state_i, \state_j)/\numobs, \qquad \text{and}
\qquad \CrMt(i,j) = \Ker(\state_i,\statetwo_j)/\numobs \quad \text{for
  $i,j = 1,\ldots,\numobs$}.
\end{align}
The following lemma yields an explicit linear-algebraic expression
for the solution:
\begin{lemma}[Kernel-based computation]
\label{lemma:matrix}
The LSTD estimator $\thetahat$ takes the form
\begin{align}
	\label{eq:linear_expression}
  \thetahat = \reward + \frac{\discount}{\sqrt{\numobs}}
  \sum_{i=1}^n \widehat{\alpha}_i \; \Ker(\cdot, \state_i),
\end{align}
where the coefficient vector $\widehat{\alphabold} \in
\Real^{\numobs}$ is the solution to the linear system
\begin{align}
  \label{eq:lineq}
  \big( \CovMt + \ridge \IdMt - \discount \CrMt^{\top} \big) \;
  \widehat{\alphabold} \; = \; \yvec.
\end{align}
Here $\yvec \in \Real^{\numobs}$ has entries $\statenew_i =
\reward(\statetwo_i)/\sqrt{n}$.
\end{lemma}
\begin{proof}
We first show that function $(\thetahat - \reward)$ can be linearly expressed by the representers of evaluation $\{ \Rep{\state_i} \}_{i=1}^{\numobs}$ as in equation~\eqref{eq:linear_expression}. Take a linear subspace $\RKHShat$ of $\RKHS$ that is spanned by
representer functions $\{ \Rep{\state_i} \}_{i=1}^{\numobs}$. By denoting $\CovOptilde \defn \CovOphat + \ridge \IdOp$, we recast equation~\eqref{eq:def_thetahat} into \begin{align} \label{eq:projBellman_new} \CovOptilde (\thetahat - \reward) = \CrOphat \thetahat. \end{align} The right hand side satisfies $\CrOphat \thetahat \in \RKHShat$ by definition. As long as we can show that \begin{align} \label{eq:inverseCov} \CovOptilde^{-1} \RKHShat \subset \RKHShat, \end{align} it  follows from equation~\eqref{eq:projBellman_new} that $\thetahat - \reward = \CovOptilde^{-1} \big( \CrOphat \thetahat \big) \in \RKHShat$, which then implies the existence of a coefficient vector $\widehat{\alphabold} \in
\Real^{\numobs}$ such that
\begin{align*}
  \thetahat = \reward + \frac{\discount}{\sqrt{\numobs}} \sum_{i=1}^n
  \widehat{\alpha}_i \Rep{\state_i} = \reward +
  \frac{\discount}{\sqrt{\numobs}} \sum_{i=1}^n \widehat{\alpha}_i
  \Ker(\cdot, \state_i).
\end{align*}
We now prove the relation~\eqref{eq:inverseCov} by contradiction.
In fact, if there exists a function $f \in
\RKHShat$ such that $g = \CovOptilde^{-1} f \notin \RKHShat$, then
$\CovOptilde g = \frac{1}{\numobs} \sum_{i=1}^{\numobs} \Rep{\state_i}
g(\state_i) + \ridge g \notin \RKHShat$, which contradicts the condition $f \in \RKHShat$.

Below we derive the explicit form of vector $\widehat{\alphabold}$.
We take a shorthand $\widehat{f} \defn \discount^{-1} (\thetahat -
\reward) = \frac{1}{\sqrt{\numobs}} \sum_{i=1}^{\numobs}
\widehat{\alpha}_i \Ker(\cdot, \state_i)$.  It follows from equation~\eqref{eq:projBellman_new} that
\begin{align}
  \label{eq:Bellman_sol}
\big( \CovOphat - \discount \CrOphat \big) \widehat{f} + \ridge \widehat{f} =
\CrOphat \reward.
\end{align}
Plugging the definitions of $\CovOphat$ and
$\CrOphat$ into equation~\eqref{eq:Bellman_sol}, we find that the left hand side equals
\begin{align*}
  & \frac{1}{\numobs\sqrt{\numobs}} \sum_{i=1}^{\numobs} \Rep{\state_i}
  \sum_{j=1}^{\numobs} \widehat{\alpha}_j \big( \Ker(\state_i, \state_j) -
  \discount \Ker(\statetwo_i,\state_j) \big) + \frac{\ridge}{\sqrt{\numobs}}
  \sum_{i=1}^{\numobs} \widehat{\alpha}_i \Rep{\state_i} \\ = &
  \frac{1}{\sqrt{\numobs}} \big[ \Rep{x_1}, \Rep{x_2}, \ldots,
    \Rep{x_{\numobs}} \big] \big( \CovMt - \discount \CrMt^{\top} +
  \ridge \IdMt \big) \widehat{\alphabold}.
\end{align*}
The right hand side of equation~\eqref{eq:Bellman_sol} takes the form
\begin{align*}
\frac{1}{\numobs}\sum_{i=1}^{\numobs} \Rep{\state_i} \reward(\statetwo_i) =
\frac{1}{\sqrt{\numobs}} \big[ \Rep{x_1}, \Rep{x_2}, \ldots,
  \Rep{x_{\numobs}} \big] \yvec.
\end{align*}
Comparing both sides, we have shown that the coefficient vector $\widehat{\alphabold}$ satisfies the linear
system~\eqref{eq:lineq}, thereby completing the proof.
\end{proof}

%%%%%%%%%%%%%%%%%%%%%%%%%%%%%%%%%%%%%%%%%%%%%%%%%%%%%%%%%%%%%%%%%%%%%%%

\subsection{Proof of Lemma~\ref{lemma:arrowhead}}
\label{append:arrowhead}

We observe that the distribution $\distrm{m}$ is relatively close to
the uniform measure $\distr$ over $[0,1)$. Therefore, we expect that
  the eigenspectra of $\CovOp(\TransOp_{m})$ and $\CovOp$ should be
  similar, where $\CovOp(\TransOp_m)$ and $\CovOp$ are the covariance
  operators associated with distributions $\distrm{m}$ and $\distr$
  respectively. Recall that by our construction of the kernel $\KerB$
  in equation~\eqref{eq:def_Ker}, $\CovOp$ has eigenpairs $\{ (\eig_j,
  \baseB{j}) \}_{j=1}^{\infty}$. In the following, we expand
  $\CovOp(\TransOp_m)$ using the basis functions $\{ \baseB{j}
  \}_{j=1}^{\infty}$, which yields an arrowhead matrix $\bSigma$. We
  take shorthands $\eigtil_j \equiv \eig_j(\TransOp_m)$ and
  $\basetil_j \equiv \base_j(\TransOp_m)$, and connect the eigenpairs
  $\{ (\eigtil_j, \basetil_j) \}_{j=1}^{\infty}$ of
  $\CovOp(\TransOp_m)$ with the spectrum of $\bSigma$ in
  \Cref{sec:def_arrowhead}. The bounds on eigenvalues and the norms of
  eigenfunctions are developed in
  \Cref{sec:arrowhead_eigval,sec:arrowhead_eigfun} respectively.

\subsubsection{Explicit forms of the eigenvalues and eigenfunctions}
\label{sec:def_arrowhead}

We calculate the pairwise inner products of functions $\{ \baseB{j}
\}_{j=1}^{\infty}$ under the distribution $\distrm{m}$. By definition
of $\{ \baseB{j} \}_{j=1}^{\infty}$ in equation~\eqref{eq:phi_theta},
we have $\baseB{j}^2(\state) = 1$, therefore, $\int_{\StateSp}
\baseB{j}^2(\state) \; \distrm{m}(\dx) = 1$ for any $j = 1,2,\ldots$.
We then consider $\int_{\StateSp} \baseB{i}(\state) \baseB{j}(\state)
\; \distrm{m}(\dx)$ with $i \neq j$.  Suppose that $i,j \geq
2$. Recall that by our construction, $\baseB{j}(\state) =
\baseB{j}(\state+\tfrac{1}{2})$ for any $\state \in
\big[0,\tfrac{1}{2}\big)$ and $j \geq 2$. Moreover, we have
  $\distrm{m}(\state) + \distrm{m}(\state + \tfrac{1}{2}) = 2$ by
  equation~\eqref{eq:def_mum}.  Based on these observations, we derive
  that
\begin{align*}
\int_{\StateSp} \baseB{i}(\state) \baseB{j}(\state) \distrm{m}(\dx) & = \int_0^{\frac{1}{2}} \baseB{i}(\state) \baseB{j}(\state) \big\{ \distrm{m}(\dx) + \distrm{m}\big(\diff (\state+\tfrac{1}{2})\big) \big\} \\ & = 2 \int_0^{\frac{1}{2}} \baseB{i}(\state) \baseB{j}(\state) \dx = \int_{\StateSp} \baseB{i}(\state) \baseB{j}(\state) \dx = 0,
\end{align*}
where the last equality is because $\baseB{i}$ and
$\baseB{j}$ are orthogonal in $\Lmu$ for any $i \neq j$.  As
for the cases where $i=1$ and $j \geq 2$, we find that
\begin{align*}
\int_{\StateSp} \baseB{i}(\state) \baseB{j}(\state) \;
\distrm{m}(\dx) & \overset{(i)}{=} \int_0^{\frac{1}{2}} \base_{\theta,
  j}(\state) \; \big\{ \distrm{m}(\dx) - \distrm{m}\big(\diff(\state+\tfrac{1}{2})\big) \big\} \\
& \overset{(ii)}{=} \frac{2}{\parap} \int_0^{\frac{1}{2}}
\baseB{j}(\state) f_m(\state) \; \dx \overset{(iii)}{=} \frac{1}{\parap}
\int_{\StateSp} \baseB{j}(\state) f_m(\state) \dx
\overset{(iv)}{=} \begin{cases} \frac{\alpha_m^{(j)}}{25
    \sqrt{\numobs}} \!\!\! & \text{if $j \leq \statdim$}, \\ 0 &
  \text{otherwise}. \end{cases}
\end{align*}
Here step (i) follows from the fact that $\baseB{1}(\state) =
\mathds{1}\big\{\state \in [0,\tfrac{1}{2})\big\} - \mathds{1}\big\{\state \in
  [\tfrac{1}{2},1)\big\}$; step (ii) follows from the equality
    $\distrm{m}(\state) - \distrm{m}(\state+\tfrac{1}{2}) = (2/\parap) \,
    f_m(\state)$ by equation~\eqref{eq:def_mum}; step (iii) is because
    $f_m(\state) = f_m(\state+\frac{1}{2})$ for any $\state \in [0, \frac{1}{2})$;
      and step (iv) results from our choice of $f_m$ in
      equation~\eqref{eq:def_fm_2}.

Based on the calculations above, we are now ready to explicitly express the
eigenvalues and eigenfunctions of operator $\CovOp(\TransOp_m)$. Define
a $\statdim$-by-$\statdim$ matrix
\begin{align}
\label{eq:def_arrowhead}
\bSigma \defn \begin{pmatrix} \eig_1 & \bx^{\top} \\ \bx & \bD
\end{pmatrix}
\end{align}
where $\bD$ is a diagonal matrix given by $\bD \defn \diag \, \{
\eig_2, \eig_3, \ldots, \eig_{\statdim} \} \in \Real^{(\statdim-1)
  \times (\statdim-1)}$ and the vector $\bx$ satisfies $\bx \defn
\frac{1}{25\sqrt{\numobs}} \, \sqrt{\eig_1\bD} \: \alphabold_m \in
\Real^{\statdim-1}$. Recall that the binary vector $\alphabold_m$ is a
component in the packing of Boolean hypercube $\{0, 1\}^{\statdim-1}$.

Let $\{ \eigtil_j \}_{j=1}^{\statdim}$ be the eigenvalues of
matrix ${\bf \Sigma}$ in non-increasing order and define $\eigtil_j
\defn \eig_j$ for $j \geq \statdim + 1$.  Then $\{ \eigtil_j
\}_{j=1}^{\infty}$ are the eigenvalues of covariance operator
$\CovOp(\TransOp_m)$.  For any index $j \geq \statdim + 1$, the basis
function $\baseB{j}$ is the eigenfunction associated with
eigenvalue $\eigtil_j = \eig_j$, i.e. $\basetil_j =
\baseB{j}$. When $j \in [\statdim]$, let $\bv_j \in
\Real^{\statdim}$ be the $j$-th eigenvector of the arrowhead matrix
$\bSigma$ defined in equation~\eqref{eq:def_arrowhead}. The function
$\basetil_j \defn (\baseB{1}, \baseB{2}, \ldots,
\baseB{\statdim}) \, \bv_j$ is the eigenfunction associated
with eigenvalue $\eigtil_j$.

In the sequel, we leverage the properties of the arrowhead matrix $\bSigma$ to analyze the eigenpairs $\{ (\eigtil_j, \basetil_j) \}_{j=1}^{\infty}$.

%%%%%%%%%%%%%%%%%%%%%%%%%%%%%%%%%%%%%%%%%%%%%%%%%%%%%%%%%%%%%%%%%%%%%%%%%%%

\subsubsection{Bounds on eigenvalues}
\label{sec:arrowhead_eigval}

We learn from Cauchy interlacing theorem that
\begin{align}
  \label{eq:interlace}
\eigtil_1 \geq \eig_2 \geq \eigtil_2 \geq \ldots \geq \eig_{\statdim}
\geq \eigtil_{\statdim}.
\end{align}
Therefore, $\eigtil_j \leq \eig_j$ for $j \geq 2$.

%%%%%%%%%%%%%%%%%%%%%%%%%%%%%%%%%%%%%%%%%%%%%%%%%%%%%%%%%%%%%%%%%%%%%%%%%%%

\subsubsection{Bonds on the norms of eigenfunctions} \label{sec:arrowhead_eigfun}

By our construction, we have $\supnorm{\baseB{j}} = 1$ for all $j = 1,2,\ldots$. 
It follows that $\supnorm{\basetil_j} = \supnorm{\baseB{j}} \leq \unibou$ for \mbox{$j \geq \statdim + 1$}.
As for indices $j \in [\statdim]$, it holds that $\supnorm{\basetil_j} \leq \norm{\bv_j}_1 \; \sup_{i \in [\statdim]} \supnorm{\baseB{i}} = \norm{\bv_j}_1$. In what follows, we verify that $\norm{\bv_j}_1 \leq \unibou$ for any $j \in [\statdim]$.

Using the properties of arrowhead matrix $\bSigma$ \cite{o1990computing}, we find that $\bv_j$ can be explicitly written as
\begin{align} \label{eq:def_eigvec}
\bv_j = \frac{\bu_j}{\norm{\bu_j}_2} \qquad \text{with } \bu_j = \begin{pmatrix}
1 \\ (\eigtil_i {\bf I} - {\bf D})^{-1} \bx
\end{pmatrix}
= \begin{pmatrix}
1 \\ \frac{\sqrt{\eig_1}}{25\sqrt{\numobs}} \; (\eigtil_j \bI - \bD)^{-1} \sqrt{\bD} \: \alphabold_m
\end{pmatrix}
\end{align}
for any $j \in [\statdim]$. The eigenvalues $\{ \eigtil_j \}_{j=1}^{\statdim}$ are zeros to the characteristic function
\begin{align} \label{eq:def_charfun} \cha(\eig) \; \defn \; \eig_1 - \eig + \bx^{\top}(\eig \bI - \bD)^{-1} \bx. \tag{$\cha(\eig)$} \end{align}

\paragraph{Estimation of $\norm{\bv_1}_1$:} We first consider $\norm{\bv_1}_1$, the $\ell_1$-norm of the first eigenvector. Since \mbox{$\norm{\bu_j}_2 \geq 1$}, we use the expression of $\bv_1$ in equation~\eqref{eq:def_eigvec} and find that
\begin{align}
\label{eq:v1_l1}
\norm{\bv_1}_1 & \leq \norm{\bu_1}_1 = 1 + \frac{\sqrt{\eig_1}}{25\sqrt{\numobs}} \sum_{i=2}^{\statdim} \frac{\sqrt{\eig_i}}{\eigtil_1 - \eig_i} \; .
\end{align}
According to the characteristic equation $\cha(\eigtil_1) = 0$, it holds $\eigtil_1 - \eig_1 = \bx^{\top} (\eigtil_1 \bI - \bD)^{-1} \bx$.
Inequality~\eqref{eq:interlace} ensures that $\eigtil_1 \geq \eig_2 \geq \ldots \geq \eig_{\statdim}$, therefore, $\bx^{\top} (\eigtil_1 \bI - \bD)^{-1} \bx \geq 0$. It further implies $\eigtil_1 \geq \eig_1$. We plug it into inequality~\eqref{eq:v1_l1} and obtain that
\begin{align*}
\norm{\bv_1}_1 \leq 1 + \frac{\sqrt{\eig_1}}{25\sqrt{\numobs}} \sum_{i=2}^{\statdim} \frac{\sqrt{\eig_i}}{\eig_1 - \eig_2} \overset{(i)}{\leq} 1 + \frac{\sqrt{\eig_1} \sqrt{\statdim \sum_{i=1}^{\infty} \eig_i}}{25\sqrt{\numobs}(\eig_1 - \eig_2)} \overset{(ii)}{\leq} 1 + \frac{\bou\sqrt{\eig_1}}{25(\eig_1 - \eig_2)} \, \frac{1-\discount}{\unibou \stdbar} \, \newradbar \delcrit \overset{(iii)}{\leq} 2 = \unibou.
\end{align*}
Here, step (i) is due to the Cauchy-Schwarz inequality; step (ii) is by inequality
$\sum_{j=1}^\infty \eig_j \leq \tfrac{\bou^2}{4}$ in condition~\eqref{EqnCondA} and the critical inequality~\eqref{eq:CI_lb}; and step (iii) is due to inequality $\newradbar \delcrit \leq 10 \unibou \stdbar \big( 1 - \tfrac{\eig_2}{\eig_1} \big) \frac{\sqrt{\eig_1}}{\bou}$ in condition~\eqref{eq:ncond_2}. We then conclue that $\supnorm{\base_1} \leq \norm{\bv_1}_1 \leq 2 = \unibou$, as claimed in the lemma statement.

\paragraph{Estimation of $\norm{\bv_j}_1$ for $j = 2,\ldots,\statdim$:}

We next consider the $\ell_1$-norms of eigenvectors
$\bv_2,\ldots,\bv_{\statdim}$.  Intuitively, when the sample size
$\numobs$ is sufficiently large, vector $\bx$ in matrix $\bSigma$ is
small and $\bSigma$ is approximately diagonal. In this case,
we expect that the eigenvector $\bv_j$ is close to the $j$-th
canonical basis $\be_j$ so that $\supnorm{\basetil_j} = \norm{\bv_j}_1
\approx \norm{\be_j}_1 = 1$.

In order to prove this claim, we will show that the $j$-th entry of
vector $\bu_j$ (denoted by $\bu_j(j)$) in
equation~\eqref{eq:def_eigvec} is noticeably larger than the other
entries in $\bu_j$.  This is because the eigenvalue difference
$|\eigtil_j - \eig_j|$ is rather small compared with eigengaps
$|\eigtil_j - \eig_i|$ with $i \neq j$. Indeed, we will prove that it
roughly holds $\eig_j - \eigtil_j \lesssim \frac{\eig_j}{\numobs}$,
thus $\bu_j(j)$ has order $\Omega(\sqrt{\numobs})$. Under our eigengap
condition $\min_{3 \leq j \leq \statdim} \big\{ \sqrt{\eig_{j-1}} -
\sqrt{\eig_j} \big\} \geq \frac{\delcrit}{2\statdim}$, the gaps
$|\eigtil_j - \eig_i|$ with $i \neq j$ are relatively large so that
the sum of entries $\big\{ |\bu_j(i)| \mid i \neq j \big\}$ is at most
$\widetilde{\mathcal{O}}(\sqrt{\statdim})$ \footnote{$\widetilde{\mathcal{O}}$ stands for
the big $O$ notation, omitting logarithmic factors.}. Here, $\bu_j(i)$
denotes the $i$-th entry of vector $\bu_j$. To this end, rescaling
$\bu_j$ yields a vector $\bv_j$ that approximates $\be_j$. \\

Let us now prove the arguments that were sketched above.  For
notational simplicity, we only consider $\bv_j$ with $2 \leq j \leq
\statdim - 1$. The analysis of $\bv_{\statdim}$ is very similar.  We
first partition the entries of $\bu_j$ into three groups and decompose
the norm $\norm{\bv_j}_1$ accordingly. Specifically, we have
$\norm{\bv_j}_1 = \TermA_1 + \TermA_2 + \TermA_3$ where
\begin{align*}
& \TermA_1 \defn \frac{1}{\norm{\bu_j}_2} \big\{ |\bu_j(1)| +
  |\bu_j(j)| + |\bu_j(j+1)| \big\}, \\ & \TermA_2 \defn
  \frac{1}{\norm{\bu_j}_2} \sum_{i = 2}^{j-1} |\bu_j(i)| \qquad
  \text{and} \qquad \TermA_3 \defn \frac{1}{\norm{\bu_j}_2} \sum_{i =
    j+2}^{\statdim} |\bu_j(i)|.
\end{align*}
By the Cauchy-Schwarz inequality, the term $\TermA_1$ satisfies
\begin{align}
\label{eq:TermA1}
\TermA_1 \leq \frac{|\bu_j(1)| + |\bu_j(j)| + |\bu_j(j+1)|}{\sqrt{(\bu_j(1))^2 + (\bu_j(j))^2 + (\bu_j(j+1))^2}} \leq \sqrt{3} \;.
\end{align}
We take shorthands $\util_{j,i} \defn
\frac{25\sqrt{\numobs}}{\sqrt{\eig_1}} \, \bu_j(i) =
\frac{\sqrt{\eig_i}}{\eigtil_j - \eig_i} \, \alphabold_m(i)$ for $i =
2,3,\ldots,\statdim$.  Since $\bu_j(j)$ dominates the other entries in
$\bu_j$, we approximate $\TermA_2$ and $\TermA_3$ by
\begin{subequations}
	\label{eq:TermA23}
\begin{align}
& \TermA_2 \leq \frac{1}{|\bu_j(j)|} \sum_{i = 2}^{j-1} |\bu_j(i)| =
  \frac{1}{|\util_{j,j}|} \sum_{i = 2}^{j-1} |\util_{j,i}| \nfed
  \TermAtil_2, \\ & \TermA_3 \leq \frac{1}{|\bu_j(j)|} \sum_{i =
    j+2}^{\statdim} |\bu_j(i)| = \frac{1}{|\util_{j,j}|} \sum_{i =
    j+2}^{\statdim} |\util_{j,i}| \nfed \TermAtil_3.
\end{align}
\end{subequations}
In the following, we estimate upper bounds $\TermAtil_2$ and
$\TermAtil_3$ in inequalities~\eqref{eq:TermA23}. \\

Under the eigengap condition \mbox{$\min_{3 \leq i \leq \statdim} \big\{ \sqrt{\eig_{i-1}} - \sqrt{\eig_i} \big\} \geq \frac{\delcrit}{2\statdim}$}, we can show that
\begin{align} \label{eq:claim}
& \sum_{i=2}^{j-1} \frac{\sqrt{\eig_i}}{\eig_i - \eig_j} \leq \frac{2\statdim}{\delcrit}(1 + \log \numobs),
& \sum_{i=j+2}^{\statdim} \frac{\sqrt{\eig_i}}{\eig_{j+1} - \eig_i} \leq \frac{2\statdim}{\delcrit}(1 + \log \numobs).
\end{align}
We assume the claim~\eqref{eq:claim} to hold at this point and prove that both $\TermAtil_2$ and $\TermAtil_3$ are constant order.

In terms of the numerators of terms $\TermAtil_2$ and $\TermAtil_3$, the interlacing inequality~\eqref{eq:interlace} and the claim~\eqref{eq:claim} imply that
\begin{subequations}
	\label{eq:numerator}
	\begin{align}
	& \sum_{i=2}^{j-1} |\util_{j,i}| = \sum_{i=2}^{j-1} \frac{\sqrt{\eig_i}}{\eig_i - \eigtil_j} \leq \sum_{i=2}^{j-1} \frac{\sqrt{\eig_i}}{\eig_i - \eig_j} \leq \frac{2\statdim}{\delcrit}(1 + \log \numobs), \label{eq:term1} \\
	& \sum_{i=j+2}^{\statdim} |\util_{j,i}| = \sum_{i=j+2}^{\statdim} \frac{\sqrt{\eig_i}}{\eigtil_j - \eig_i} \leq \sum_{i=2}^{j-1} \frac{\sqrt{\eig_i}}{\eig_{j+1} - \eig_i} \leq \frac{2\statdim}{\delcrit}(1 + \log \numobs). \label{eq:term2}
	\end{align}
\end{subequations}
Consider the common denominator $|\util_{j,j}| = \frac{\sqrt{\eig_j}}{\eig_j - \eigtil_j}$ of $\TermAtil_2$ and $\TermAtil_3$. A key step in our analysis is to estimate the perturbation term $\eig_j - \eigtil_j$.
Recall that $\eigtil_j$ satisfies the characteristic equation $\cha(\eigtil_j) = 0$, which translates into
\begin{align*}
%\label{eq:chaeqn}
\frac{\eig_1\eig_j}{25^2 \numobs(\eig_j - \eigtil_j)} = \eig_1 - \eigtil_j + \frac{\eig_1}{25^2 \numobs} \sum_{\begin{subarray}{c} 2 \leq i \leq \statdim, \\ i \neq j \end{subarray}} \frac{\eig_i}{\eigtil_j - \eig_i}.
\end{align*}
We use the interlacing inequality~\eqref{eq:interlace} and obtain that
\begin{align*}
\frac{\eig_1\eig_j}{25^2 \numobs(\eig_j - \eigtil_j)} \geq \eig_1 - \eig_j - \frac{\eig_1}{25^2 \numobs} \sum_{i=2}^{j-1} \frac{\eig_i}{\eig_i - \eig_j} \geq \eig_1 - \eig_j - \frac{\eig_1^{\frac{3}{2}}}{25^2 \numobs} \sum_{i=2}^{j-1} \frac{\sqrt{\eig_i}}{\eig_i - \eig_j}.
\end{align*}
When the bounds~\eqref{eq:claim} hold, we have
\begin{align} \label{eq:small}
\frac{\eig_1^{\frac{3}{2}}}{25^2 \numobs} \sum_{i=2}^{j-1} \frac{\sqrt{\eig_i}}{\eig_i - \eig_j} \overset{}{\leq} \frac{2 \eig_1^{\frac{3}{2}}\statdim}{25^2 \numobs\delcrit} (1 + \log \numobs) \overset{(i)}{\leq} \frac{2\eig_1^{\frac{3}{2}}}{25^2} \; \Big\{ \frac{\newradbar\SqDis}{\unibou \stdbar} \Big\}^2 \delcrit \; (1 + \log \numobs) \overset{(ii)}{\leq} \frac{8}{125} (\eig_1 - \eig_j),
\end{align}
where step (i) is due to inequality~\eqref{eq:CI_lb}; and in step (ii)
we use inequality \mbox{$\newradbar \delcrit \!\leq\! 10
  \unibou\stdbar \big( 1 \!-\! \tfrac{\eig_2}{\eig_1} \big)
  \frac{\unibou \stdbar/(\sqrt{\eig_1}\newradbar)}{\SqDis^2\log
    \numobs }$} in condition~\eqref{eq:ncond_2}.  We integrate the
pieces and derive that
\begin{align*}
\frac{\eig_1\eig_j}{25^2 \numobs(\eig_j - \eigtil_j)} \geq \frac{1}{2}
\, (\eig_1 - \eig_j).
\end{align*}
It further implies
\begin{align} \label{eq:ujj}
\frac{1}{|\util_{j,j}|} = \frac{\eig_j - \eigtil_j}{\sqrt{\eig_j}}
\leq \frac{2 \; \eig_1\sqrt{\eig_j}}{25^2 \, \numobs(\eig_1 -
  \eig_j)}.
\end{align}

Combining inequalities~\eqref{eq:numerator}~and~\eqref{eq:ujj}, we
find that the terms $\TermAtil_2$ and $\TermAtil_3$ in bounds
\eqref{eq:TermA23} satisfy
\begin{align*}
\max\{ \TermAtil_2, \TermAtil_3 \} & \; \leq \;
\frac{2\eig_1\sqrt{\eig_j}}{25^2 \numobs(\eig_1 - \eig_j)} \, \Big\{
\frac{2\statdim}{\delcrit}(1 + \log \numobs) \Big\} \leq
\frac{2}{\eig_1 - \eig_j} \Big\{
\frac{2\eig_1^{\frac{3}{2}}\statdim}{25^2 \numobs \delcrit}(1 + \log
\numobs) \Big\} \overset{(i)}{\leq} \frac{16}{125},
\end{align*}
where step (i) follows from inequality~\eqref{eq:small}. We plug
inequalities~\eqref{eq:TermA1}~and~\eqref{eq:TermA23} into the
decomposition $\norm{\bv_j}_1 = \TermA_1 + \TermA_2 + \TermA_3$ and
derive that $\supnorm{\basetil_j} \leq \norm{\bv_j}_1 \leq 2 =
\unibou$, as claimed in the lemma statement. \\

It only remains to prove the claim~\eqref{eq:claim}. We use some
algebra and obtain that
\begin{align*}
  \frac{\sqrt{\eig_i}}{\eig_i - \eig_j} \leq
  \frac{1}{\sqrt{\eig_i} - \sqrt{\eig_j}} \quad \text{for $i
    \leq j-1$} \qquad \text{and} \qquad
  \frac{\sqrt{\eig_i}}{\eig_{j+1} - \eig_i} \leq
  \frac{1}{\sqrt{\eig_{j+1}} - \sqrt{\eig_i}} \quad \text{for $i
    \geq j+2$}.
\end{align*}
Under the eigengap condition \mbox{$\min_{3 \leq i \leq \statdim}
  \big\{ \sqrt{\eig_{i-1}} - \sqrt{\eig_i} \big\} \geq
  \frac{\delcrit}{2\statdim}$}, we have $\sqrt{\eig_{i_1}} -
\sqrt{\eig_{i_2}} \geq (i_2 - i_1) \, \frac{\delcrit}{2\statdim}$ for
any $2 \leq i_1 < i_2 \leq \statdim$. It then follows that
\begin{align*}
\sum_{i=2}^{j-1} \frac{\sqrt{\eig_i}}{\eig_i - \eig_j} & \leq
\sum_{i=2}^{j-1} \frac{1}{\sqrt{\eig_i} - \sqrt{\eig_j}} \leq
\frac{2\statdim}{\delcrit} \sum_{i=2}^{j-1} \frac{1}{j-i} \leq
\frac{2\statdim}{\delcrit} \, \big\{ 1 + \log (j-2)\big\} \leq
\frac{2\statdim}{\delcrit} ( 1 + \log \numobs ) \, .
\end{align*}
The second bound in equation~\eqref{eq:claim} can be proved in a
similar way.

%%%%%%%%%%%%%%%%%%%%%%%%%%%%%%%%%%%%%%%%%%%%%%%%%%%%%%%%%%%%%%%%%%%%%%%%%%%%%%%
	
\bibliographystyle{abbrv} \bibliography{ref}

\begin{thebibliography}{10}

\bibitem{BagSch03}
J.~A. Bagnell and J.~Schneider.
\newblock Policy search in kernel {H}ilbert space.
\newblock Technical report, Carnegie {M}ellon {U}niversity, 2003.

\bibitem{BarPrePin16}
A.~M.~S. Bareeto, D.~Precup, and J.~Pineau.
\newblock Practical kernel-based reinforcement learning.
\newblock {\em Journal of {M}achine {L}earning {R}esearch}, 17:1--70, 2016.

\bibitem{BerTho04}
A.~Berlinet and C.~Thomas-{A}gnan.
\newblock {\em Reproducing kernel {H}ilbert spaces in probability and
  statistics}.
\newblock Kluwer {A}cademic, Norwell, {MA}, 2004.

\bibitem{bertsekas2011dynamic}
D.~P. Bertsekas.
\newblock Dynamic programming and optimal control 3rd edition, volume ii.
\newblock {\em Belmont, MA: Athena Scientific}, 2011.

\bibitem{bertsekas1995dynamic}
D.~P. Bertsekas, D.~P. Bertsekas, D.~P. Bertsekas, and D.~P. Bertsekas.
\newblock {\em Dynamic programming and optimal control}, volume~1.
\newblock Athena scientific Belmont, MA, 1995.

\bibitem{BerTsi96}
D.~P. Bertsekas and J.~N. Tsitsiklis.
\newblock {\em Neuro-Dynamic Programming}.
\newblock Athena Scientific, 1st edition, 1996.

\bibitem{BouDij17}
R.~J. Boucherie and N.~M. van Dijk.
\newblock {\em Markov decision processes in practice}.
\newblock Springer, New York, 2017.

\bibitem{bradtke1996linear}
S.~J. Bradtke and A.~G. Barto.
\newblock Linear least-squares algorithms for temporal difference learning.
\newblock {\em Machine learning}, 22(1-3):33--57, 1996.

\bibitem{dai2017learning}
B.~Dai, N.~He, Y.~Pan, B.~Boots, and L.~Song.
\newblock Learning from conditional distributions via dual embeddings.
\newblock In {\em Artificial Intelligence and Statistics}, pages 1458--1467.
  PMLR, 2017.

\bibitem{fan2020theoretical}
J.~Fan, Z.~Wang, Y.~Xie, and Z.~Yang.
\newblock A theoretical analysis of deep {Q}-learning.
\newblock In {\em Learning for Dynamics and Control}, pages 486--489. PMLR,
  2020.

\bibitem{farahmand2016regularized}
A.-m. Farahmand, M.~Ghavamzadeh, C.~Szepesv{\'a}ri, and S.~Mannor.
\newblock Regularized policy iteration with nonparametric function spaces.
\newblock {\em The Journal of Machine Learning Research}, 17(1):4809--4874,
  2016.

\bibitem{feng2019kernel}
Y.~Feng, L.~Li, and Q.~Liu.
\newblock A kernel loss for solving the {Bellman} equation.
\newblock In {\em Advances in Neural Information Processing Systems}, pages
  15456--15467, 2019.

\bibitem{feng2020accountable}
Y.~Feng, T.~Ren, Z.~Tang, and Q.~Liu.
\newblock Accountable off-policy evaluation with kernel {B}ellman statistics.
\newblock In {\em International Conference on Machine Learning}, pages
  3102--3111. PMLR, 2020.

\bibitem{grunewalder2012modelling}
S.~Grunewalder, G.~Lever, L.~Baldassarre, M.~Pontil, and A.~Gretton.
\newblock Modelling transition dynamics in {MDP}s with {RKHS} embeddings.
\newblock Technical report, UCL, 2012.

\bibitem{Gu02}
C.~Gu.
\newblock {\em Smoothing spline {ANOVA} models}.
\newblock Springer {S}eries in {S}tatistics. Springer, New York, NY, 2002.

\bibitem{jiang2016doubly}
N.~Jiang and L.~Li.
\newblock Doubly robust off-policy value evaluation for reinforcement learning.
\newblock In {\em International Conference on Machine Learning}, pages
  652--661. PMLR, 2016.

\bibitem{kallus2019efficiently}
N.~Kallus and M.~Uehara.
\newblock Efficiently breaking the curse of horizon: Double reinforcement
  learning in infinite-horizon processes.
\newblock {\em stat}, 1050:12, 2019.

\bibitem{kallus2020double}
N.~Kallus and M.~Uehara.
\newblock Double reinforcement learning for efficient off-policy evaluation in
  markov decision processes.
\newblock {\em Journal of Machine Learning Research}, 21(167):1--63, 2020.

\bibitem{khamaru2020temporal}
K.~Khamaru, A.~Pananjady, F.~Ruan, M.~J. Wainwright, and M.~I. Jordan.
\newblock Is temporal difference learning optimal? {A}n instance-dependent
  analysis.
\newblock {\em {S}{I}{A}{M} {J}. {M}ath. {D}ata Science}, page To appear, 2021.

\bibitem{KimWah71}
G.~Kimeldorf and G.~Wahba.
\newblock Some results on {T}chebycheffian spline functions.
\newblock {\em Jour. Math. Anal. Appl.}, 33:82--95, 1971.

\bibitem{koppel2020policy}
A.~Koppel, G.~Warnell, E.~Stump, P.~Stone, and A.~Ribeiro.
\newblock Policy evaluation in continuous {MDP}s with efficient kernelized
  gradient temporal difference.
\newblock {\em IEEE Transactions on Automatic Control}, 2020.

\bibitem{long20212}
J.~Long, J.~Han, and W.~E.
\newblock An ${L}^2$ analysis of reinforcement learning in high dimensions with
  kernel and neural network approximation.
\newblock {\em arXiv preprint arXiv:2104.07794}, 2021.

\bibitem{mendelson2002geometric}
S.~Mendelson.
\newblock Geometric parameters of kernel machines.
\newblock In {\em Computational Learning Theory}, pages 29--43. Springer Berlin
  Heidelberg, 2002.

\bibitem{MouLiWaiBarJor20}
W.~Mou, C.~J. Li, M.~J. Wainwright, P.~L. Bartlett, and M.~I. Jordan.
\newblock On linear stochastic approximation: {F}ine-grained {P}olyak-{R}uppert
  and non-asymptotic concentration.
\newblock In {\em Conference on Learning Theory (COLT)}, volume 125, pages
  2947--2997, 2020.

\bibitem{mou2020optimal}
W.~Mou, A.~Pananjady, and M.~J. Wainwright.
\newblock Optimal oracle inequalities for solving projected fixed-point
  equations.
\newblock {\em arXiv preprint arXiv:2012.05299}, 2020.

\bibitem{MunSze08}
R.~Munos and C.~Szepesvari.
\newblock Finite-time bounds for fitted value iteration.
\newblock {\em Journal of {M}achine {L}earning {R}esearch}, 1:815--857, 2008.

\bibitem{NewPow03}
W.~K. Newey and J.~L. Powell.
\newblock Instrumental variable estimation of non-parametric models.
\newblock {\em Econometrica}, 71(5):1565--1578, 2003.

\bibitem{nguyen2021sample}
T.~Nguyen-Tang, S.~Gupta, H.~Tran-The, S.~Venkatesh, et~al.
\newblock Sample complexity of offline reinforcement learning with deep relu
  networks.
\newblock {\em arXiv preprint arXiv:2103.06671}, 2021.

\bibitem{o1990computing}
D.~O'leary and G.~Stewart.
\newblock Computing the eigenvalues and eigenvectors of symmetric arrowhead
  matrices.
\newblock {\em Journal of Computational Physics}, 90(2):497--505, 1990.

\bibitem{ormoneit2002kernel}
D.~Ormoneit and {\'S}.~Sen.
\newblock Kernel-based reinforcement learning.
\newblock {\em Machine learning}, 49(2-3):161--178, 2002.

\bibitem{pananjady2019value}
A.~Pananjady and M.~J. Wainwright.
\newblock Instance-dependent $\ell_{\infty}$-bounds for policy evaluation in
  tabular reinforcement learning.
\newblock {\em IEEE Transactions on Information Theory}, 67(1):566--585, 2020.

\bibitem{Puterman05}
M.~L. Puterman.
\newblock {\em Markov decision processes: Discrete stochastic dynamic
  programming}.
\newblock Wiley, 2005.

\bibitem{RasWaiYu12}
G.~Raskutti, M.~J. Wainwright, and B.~Yu.
\newblock Minimax-optimal rates for sparse additive models over kernel classes
  via convex programming.
\newblock {\em Journal of {M}achine {L}earning {R}esearch}, 12:389--427, March
  2012.

\bibitem{shawe2004kernel}
J.~Shawe-Taylor, N.~Cristianini, et~al.
\newblock {\em Kernel methods for pattern analysis}.
\newblock Cambridge university press, 2004.

\bibitem{sobel1982variance}
M.~J. Sobel.
\newblock The variance of discounted markov decision processes.
\newblock {\em Journal of Applied Probability}, 19(4):794--802, 1982.

\bibitem{Steinwart05}
I.~Steinwart.
\newblock Consistency of support vector machines and other regularized kernel
  machines.
\newblock {\em IEEE Trans. Info. Theory}, 51:128--142, 2005.

\bibitem{Stone82}
C.~J. Stone.
\newblock Optimal global rates of convergence for non-parametric regression.
\newblock {\em Annals of {S}tatistics}, 10(4):1040--1053, 1982.

\bibitem{Sut88}
R.~S. Sutton.
\newblock Learning to predict via the methods of temporal differences.
\newblock {\em Machine Learning}, 3:9--44, 1988.

\bibitem{sutton2018reinforcement}
R.~S. Sutton and A.~G. Barto.
\newblock {\em Reinforcement learning: An introduction}.
\newblock MIT press, 2018.

\bibitem{taylor2009kernelized}
G.~Taylor and R.~Parr.
\newblock Kernelized value function approximation for reinforcement learning.
\newblock In {\em Proceedings of the 26th annual international conference on
  machine learning}, pages 1017--1024, 2009.

\bibitem{tsitsiklis1997analysis}
J.~N. Tsitsiklis and B.~Van~Roy.
\newblock Analysis of temporal-diffference learning with function
  approximation.
\newblock In {\em Advances in neural information processing systems}, pages
  1075--1081, 1997.

\bibitem{vandeGeer}
S.~van~de Geer.
\newblock {\em Empirical Processes in $M$-Estimation}.
\newblock Cambridge University Press, 2000.

\bibitem{wainwright2019high}
M.~J. Wainwright.
\newblock {\em High-dimensional statistics: A non-asymptotic viewpoint},
  volume~48.
\newblock Cambridge University Press, 2019.

\bibitem{WeiYanWai19}
Y.~Wei, F.~Yang, and M.~J. Wainwright.
\newblock Early stopping for kernel boosting algorithms: {A} general analysis
  with localized complexities.
\newblock {\em IEEE Trans. Info. Theory}, 65(10):6685--6703, October 2019.

\bibitem{Whi82}
H.~White.
\newblock Instrumental variables regression with independent observations.
\newblock {\em Econometrica}, 50(2):483--499, 1982.

\bibitem{Woo10}
J.~M. Wooldridge.
\newblock {\em Econometric Analysis of Cross Section and Panel Data}.
\newblock {M}{I}{T} {P}ress, Cambridge, {M}{A}, 2010.

\bibitem{xie2019towards}
T.~Xie, Y.~Ma, and Y.-X. Wang.
\newblock Towards optimal off-policy evaluation for reinforcement learning with
  marginalized importance sampling.
\newblock In {\em Advances in Neural Information Processing Systems}, pages
  9668--9678, 2019.

\bibitem{yang2017randomized}
Y.~Yang, M.~Pilanci, M.~J. Wainwright, et~al.
\newblock Randomized sketches for kernels: Fast and optimal nonparametric
  regression.
\newblock {\em The Annals of Statistics}, 45(3):991--1023, 2017.

\bibitem{yin2020asymptotically}
M.~Yin and Y.-X. Wang.
\newblock Asymptotically efficient off-policy evaluation for tabular
  reinforcement learning.
\newblock {\em arXiv preprint arXiv:2001.10742}, 2020.

\bibitem{yu2010error}
H.~Yu and D.~P. Bertsekas.
\newblock Error bounds for approximations from projected linear equations.
\newblock {\em Mathematics of Operations Research}, 35(2):306--329, 2010.

\bibitem{Zhang2005b}
T.~Zhang.
\newblock Learning bounds for kernel regression using effective data
  dimensionality.
\newblock {\em Neural Computation}, 17(9):2077--2098, 2005.

\end{thebibliography}

%%%%%%%%%%%%%%%%%%%%%%%%%%%%%%%%%%%%%%%%%%%%%%%%%%%%%%%%%%%%%%%%%%%%%%%%%%%%%%%%
	
\end{document}